\documentclass{article}
\usepackage{arxiv}
\usepackage{cite}

\usepackage{graphicx}
\usepackage{multirow}
\usepackage{amsmath,amssymb,amsfonts,amsthm}
\usepackage{bm}
\usepackage{mathrsfs}
\usepackage{mathtools}
\usepackage{bbm}
\usepackage[title]{appendix}
\usepackage[dvipsnames]{xcolor}
\usepackage{textcomp}
\usepackage{manyfoot}
\usepackage{booktabs}
\usepackage{url}
\usepackage[algo2e,ruled,vlined,linesnumbered]{algorithm2e}
\usepackage{listings}
\usepackage{adjustbox}
\usepackage{cases}
\usepackage{bbding}
\usepackage{pifont}
\usepackage{circuitikz}
\usepackage{etoolbox} %
\usepackage{booktabs,tabularx} 

\makeatletter
\renewcommand{\algocf@makecaption@ruled}[2]{%
\global\sbox\algocf@capbox{\hskip\AlCapHSkip %
\addtolength{\algocf@lcaptionbox}{-2\AlCapHSkip}%
\parbox[t]{\algocf@lcaptionbox}{\algocf@captiontext{#1}{#2}}}%
}%
\makeatother

\DeclareMathOperator{\E}{E}
\DeclareMathOperator{\Var}{Var}

\DeclareMathOperator{\Eest}{\widehat{E}}
\DeclareMathOperator{\Varest}{\widehat{Var}}

\DeclareMathOperator*{\argmin}{arg\,min}

\DeclareMathAlphabet{\ts}{OT1}{cmss}{bx}{n}

\newcommand{\N}{\mathbb{N}}
\newcommand{\R}{\mathbb{R}}

\newcommand{\set}[1]{\left\{#1\right\}}
\newcommand{\br}[1]{\left(#1\right)}
\newcommand{\cl}[1]{\left[#1\right]}
\renewcommand{\epsilon}{\varepsilon}

\renewcommand{\mid}{\,\middle\vert\,}

\newcommand{\Sdiv}{\ensuremath{\mathrm{D}}}

\newcommand{\F}{\ensuremath{\mathrm{F}}\xspace}
\newcommand{\Fest}{\ensuremath{\widehat{\mathrm{F}}}\xspace}
\newcommand{\dL}{\ensuremath{d_{\mathbb{L}}}}
\newcommand{\lap}{\ensuremath{\textsc{LAP}}}

\let\oldnl\nl
\newcommand{\nlnonumber}{\renewcommand{\nl}{\let\nl\oldnl}}

\newcommand{\BF}[1]{
	\relax
	\ifmmode
	\ifcat\noexpand#1\relax %
		\boldsymbol{#1}     %
	\else
		\mathbf{#1}
	\fi
	\else
		\textbf{#1}
	\fi
}
\renewcommand{\rm}[1]{\ensuremath{\mathrm{#1}}}
\newcommand{\uar}[1]{\ensuremath{\textsc{u.a.r.}\!\br{#1}}}

\theoremstyle{thmstyleone}
\newtheorem{theorem}{Theorem}
\newtheorem{proposition}[theorem]{Proposition}
\newtheorem{assumption}{Assumption}
\newtheorem{requirements}{Requirements}
\newtheorem{formulation}{Problem Formulation}
\newtheorem*{formulation*}{Problem Formulation}

\theoremstyle{thmstyletwo}

\theoremstyle{thmstylethree}
\newtheorem{definition}{Definition}

\newcommand{\upd}[1]{#1}
\newcommand{\utwo}[1]{{#1}}
\newcommand{\uthree}[1]{{#1}}

\begin{document}


\title{Selection of Filters for Photonic Crystal Spectrometer Using Domain-Aware Evolutionary Algorithms}
\author{\begin{tabular}{c c c}
    Kirill Antonov$^{1}$ & Marijn Seimons$^{2}$ & Niki van Stein$^{1}$ \\
    Thomas H.~W. Bäck$^{1}$ & Ralf Kohlhaas$^{2,3}$ & Anna V. Kononova$^{1}$
\end{tabular}\\
$^{1}$Leiden University, The Netherlands\\
$^{2}$SRON Netherlands Institute for Space Research, The Netherlands\\
$^{3}$Delft University of Technology, The Netherlands
}

\maketitle


%



\begin{abstract}
This work addresses the critical challenge of optimal filter selection for a novel trace gas measurement device. This device uses photonic crystal filters to retrieve trace gas concentrations prone to photon and read noise. The filter selection directly influences the accuracy and precision of the gas retrieval and therefore is a crucial performance driver.
We formulate the problem as a stochastic combinatorial optimization problem and develop a simulator \utwo{modeling} gas retrieval with noise. 
Metaheuristics representing various families of optimizers are used to minimize the retrieval error objective function.
\utwo{We improve the found top-performing algorithms using our novel distance-driven extensions, which employ metrics on the space of filter selections}.
This leads to a novel adaptation of the Univariate Marginal Distribution Algorithm (UMDA), which we call the Univariate Marginal Distribution Algorithm Unified by Probabilistic Logic Sampling driven by Distance (UMDA-U-PLS-Dist), equipped with one of the proposed distance metrics as the most efficient and robust solver among the considered ones. 
\utwo{We apply this algorithm to obtain a diverse set of high-performing solutions and analyze them to make general conclusions of the better combinations of transmission profiles.}
\utwo{Analysis reveals} that filters with large local differences in transmission improve the device performance.
Moreover, \utwo{the obtained top-performing solutions show} significant improvement compared to the baseline. 
\end{abstract}

\keywords{Photonic crystal \and Trace Gas Measurements \and Spectrometer \and Metaheuristic \and Evolutionary Computation \and Domain-Aware Optimization \and Stochastic Combinatorial Optimization \and Simulation Optimization}




\section{Introduction}
\label{sec:introduction}

As global climate change severely impacts our world, there is an increasing demand to monitor trace gases with a high spatial resolution and accuracy.  
At the same time, these instruments and satellites need to be compact in order to have constellations for short revisit times~\cite{Pastena2020}. %
In this study we explore a new spectrometer instrument concept to detect trace gases where photonic crystal filters directly placed on a detector replace traditional diffraction-based optical elements~\cite{siemons2023compressive}. 


\utwo{Trace gasses in the atmosphere slightly change the spectrum of the transmitted light due to absorption at specific wavelengths.}
The instrument is able to perform trace gas retrieval by measuring the total amount of reflected light each filter transmits for a certain ground pixel. 
By analyzing the measured intensities, it is possible determine the concentration of trace gases.
The performance of the instrument, measured as accuracy and precision in gas retrieval, depends on the signal-to-noise ratio and significantly on the selected filters.
Therefore, a key challenge is to select the filters that will most accurately measure the trace gas concentration.
To achieve this, we simulated how 2D photonic crystals respond to light, \utwo{creating a library of about 5000 different filters in various shapes and sizes.} 
Given that the considered instrument can have at maximum 640 filters, we reduce the problem to the selection of an optimal combination of filters from the created library. 
In this study, we are specifically focusing on detecting methane as a test case.
Therefore, we can measure the instrument's performance by comparing the detected methane concentration with the actual concentration.
While we concentrate in this work on the optimization of a trace gas imager for earth observation, the optimization algorithm introduced in this paper could be transferred to other instruments with nanostructured filter arrays, such as in biological imaging~\cite{xiong2022dynamic}, mid-infrared gas detection~\cite{meng2024smart} and flow cytometry~\cite{hong2024metasurface}.

\upd{
This paper addresses the described problem of selecting from a large library of approximately $L \approx 5 \cdot 10^3$ filters a combination that optimizes the instrument's performance in methane retrieval.
Unlike a set, where each element is unique, filters in this selection can be chosen multiple times, i.e., a solution candidate is a multiset with $N$ elements chosen from $L$.
For the remainder of this paper, we refer to this core problem as Optimal Filter Selection (OFS).
The problem is addressed by an optimization algorithm that tackles both challenges: \utwo{the prohibitively large number of possible solutions and the presence of noise introduced by a realistic atmospheric model.}
The latter is a special challenge, as many evaluations of different filter sets can result in outliers, eroding the convergence of the optimization algorithm.
Therefore, it is most natural to formulate OFS as a stochastic combinatorial optimization problem. %
}

\subsection{Combinatorial Optimization}
\label{sec:comb-optim}

The field of combinatorial optimization~\cite{korte2011combinatorial, du2022introduction} focuses on designing algorithms to find one or more objects from a finite set that minimize a specified cost criterion. 

\utwo{
According to a high-level classification of combinatorial optimization algorithms~\cite{du2022introduction}, these methods can be broadly divided into the following three categories:}

\begin{enumerate}
    \item \utwo{\textbf{Exact algorithms}, which aim to compute the optimal solution with mathematical guarantees. These include classical methods such as branch-and-bound, branch-and-cut, and dynamic programming, and are extensively studied in~\cite{korte2011combinatorial, du2022introduction}. While they provide optimality, their computational cost often limits their applicability to small or moderately sized instances.}
    
    \item \utwo{\textbf{Approximation algorithms}, which provide provably near-optimal solutions within a guaranteed error bound. These algorithms are especially useful for NP-hard problems where exact computation is infeasible, and are covered in depth in~\cite{du2012design}.}

    \item \utwo{\textbf{Heuristic and metaheuristic algorithms}, including evolutionary algorithms, simulated annealing, and ant colony optimization. These are state-of-the-art numerical solvers that are widely applied in practice due to their ability to find high-quality solutions for large-scale and complex problems, albeit without formal optimality guarantees. Comprehensive reviews can be found in~\cite{blum2003metaheuristics, peres2021combinatorial}.}
\end{enumerate}

Algorithms in the former two categories are conventional combinatorial optimization methods (CCOMs).
They search via mathematical or dynamic programming and employ heuristic ideas to disregard intuitively sub-optimal solutions.
Particular instances of CCOMs are implemented to solve a specific known COP and therefore heavily depend on the known structure of the addressed problem.

The third category of algorithms includes methods that apply nature-inspired meta-heuristics~\cite{rahman2021nature} and machine learning~\cite{zhang2023survey} approaches to find solution candidates with sufficiently small costs.
While these optimizers do not guarantee to locate the optimal solution, they are less dependent on the constraint nature of the problem and so can be applied to different classes of COPs without major modifications.
Successful applications of these methods are demonstrated in works such as~\cite{slowik2020evolutionary, herrero2008non, zhou2021novel}.
Those and other relevant research directions are systematically studied in this recent survey~\cite{weinand2022research}.

\subsection{Stochastic Combinatorial Optimization}
\label{sec:stochastic-comb-optim}

Stochastic combinatorial optimization problems (SCOPs) introduce an additional layer of complexity: the cost associated with a specific solution candidate exhibits randomness, meaning it may vary upon reevaluation.
This property is especially relevant for engineering applications as the cost is estimated by a simulation of a real-world process, for example, see Part IV in~\cite{andradottir1998simulation}.
Application of an algorithm to a SCOP that relies on a simulation is called simulation optimization (SO) in the literature~\cite{andradottir1998simulation, fu2015handbook, amaran2016simulation}, since a closed-form mathematical expression is not available to describe the optimized function.
SO can be classified as a particular case of a more general Black-Box Optimization task~\cite{alarie2021two}.
CCOMs can be applied as simulation optimizers, for example, see Chapter 2 of~\cite{fu2015handbook}. 
However, rapidly developing areas of relevant meta-heuristic and machine learning methods~\cite{bianchi2009survey} contain alternative promising algorithms.
As shown in Table 1 of~\cite{bianchi2009survey}, \utwo{those algorithms were} advantageous in a number of practical use cases.
See the survey~\cite{do2022metamodel} for more examples and a fresh systematic review of the related publications.

\subsection{Addressing OFS problem}
\label{sec:addressing-ofs}

In this paper, we address the problem of filter selection as an instance of SCOP. %
Due to the highly complex relation between the selected filters and measured trace gas concentration, we do not study this relation directly.
Instead, we approach our goal using SO and in this regard, we implemented a simulator of the Trace Gas Measurement Device (TGMD).
Our simulator constitutes an approximate model of gas tracing.
More details on TGMD and the simulator are provided in Section~\ref{sec:description-trace-gas-measurement-device}.

While CCOMs for OFS might exist, the complexity of deriving them may not be practical. 
Therefore, we leverage the efficiency of numerical methods, such as meta-heuristics and machine learning, to obtain a \utwo{high-quality} approximate solution to OFS.
\uthree{Meta-heuristics require neither explicit analytical models nor gradients, and have a long track record of delivering high-quality solutions on complex, noisy, and nonconvex problems across many application domains as discussed in Sec.~\ref{sec:comb-optim},~\ref{sec:stochastic-comb-optim}.
They are also flexible and easily extensible because they are robust to changes in the simulator or constraint set.
Therefore, addressing the OFS problem with meta-heuristics is the best available approach, which we employ in this work.}

Given the vast number of existing general-purpose methods, we use the latest (known to us) comprehensive survey~\cite{amaran2016simulation}.
This survey provides a clear description of various methods applicable to SO, helping us in selecting optimizers for OFS task.
Independent of our choice, it is guaranteed by the No Free Lunch (NFL) theorem that none of the selected optimizers is the best on average for a big enough class of problems, meaning, a 
class of problems closed under the permutations of the search space~\cite{schumacher2001no}. %
Therefore, knowing the best optimizer for OFS in advance is intractable.
Hence, we follow the approach suggested in Section 2.7 of~\cite{peres2021combinatorial}.
We conduct experiments where we repeatedly approximate the solution of the OFS problem using every chosen algorithm.
This allows us to statistically rank the performance of each algorithm and identify the most suitable ones for our specific problem.
We call the top-performing solvers the \emph{leading} algorithms.

Since the optimal solution of OFS is
unknown, developing algorithms that outperform the best existing solvers is relevant for deeper exploration of OFS.
Moreover, in the future, we plan to improve our simulator of TGMD to be even closer to the real-world, 
which will result in increased computational complexity of the simulations.
Therefore, developing an optimizer capable of finding solutions to OFS within a small computational budget is relevant for finding low-cost approximate solutions to the enhanced OFS in a feasible time frame.
For those two reasons, we aim to improve leading solvers and explore a possible approach for doing so.
We propose extensions of the leading algorithms to take advantage of domain-specific information. 
Such information will be formulated as metrics on the set of filters.

The rest of the paper is organized as follows:
In Section~\ref{sec:description-trace-gas-measurement-device} we provide an overview of the developed Trace Gas Measurement Device and its simulation.
Then, in Section~\ref{sec:statement} we formalize the OFS problem and define the cost criterion.
An overview of existing approaches in numerical optimization which we selected to solve OFS is given in Section~\ref{sec:optimization}.
The leading algorithms are outlined in Section~\ref{sec:optimization:comparison}.
We define the information specific to OFS, propose methods to embed this information in the leading algorithms and analyse the obtained solvers in Section~\ref{sec:dd}.
Finally, we study the obtained high-performing OFS solutions in Section~\ref{sec:solutions}.
        \textbf{Our contribution:}
        \begin{enumerate}
            \item \utwo{Compare  meta-heuristics that represent various classes of SCOP solvers, and select the solvers with the best performance on OFS};
            \item We express specific knowledge regarding the space of filters as a metric  (in the mathematical sense), propose extensions of algorithms selected in step 1, and apply domain-aware meta-heuristics to OFS;
            \item \utwo{To implement a key step required for obtaining an effective mutation operator in this distance-driven evolutionary framework, we design a heuristic method\utwo{, called \emph{DDA-EA},} to approximately solve the inverse of the Linear Assignment Problem;}
            \item \utwo{We introduce a novel method, \emph{UMDA-U-PLS-Dist}, which incorporates a user-defined distance metric to the Estimation of Distribution Algorithm to guide the search process more effectively, enhancing the exploitation of problem-specific structure.}
            \item We \utwo{obtain a diverse set of high-performing solutions to the OFS problem and analyze them to make general conclusions of the better combinations of transmission profiles.} 
        \end{enumerate}

\section{Description of the Trace Gas Measurement Device}
\label{sec:description-trace-gas-measurement-device}

The Trace Gas Measurement Device (TGMD), \utwo{shown in  Figure~\ref{fig:instrumentconcept}} consists of an imaging telescope, where photonic crystal filters\upd{, acting as transmission filters,} are placed directly on the detector.
\upd{
As the instrument flies over the earth, the spectrum of light of a single ground pixel is projected on to pixels with different optical filters. Each pixel then measures the total transmitted intensity for this filter. 
Effectively, the spectrum of light is mapped to the (compressed) bases of the transmission profiles of the filters. 
The trace gas concentration can then be retrieved directly from the different measured integrated intensities. 
In this study we focus on methane retrieval as an initial test case. 
The detector in question has 512$\times$640 pixels and is orientated such that the 512 pixels are perpendicular to the flight direction and the 640 pixels are along the flight direction. 
A single ground resolution element is then observed by 640 pixels as the instrument flies over the earth. 
Therefore, there can be at maximum 640 filters placed on the detector, one for each pixel. 
However, current practical limitations such as manufacturing constraints on the distance between the filters and the detector and electromagnetic interactions between neighboring filters limits the number to 16 different filters. 
This number of filters still contains redundancy in case of in-flight variations of the filter response and unexpected biases due to mismatches in the atmosphere model and imperfect supporting (weather) data, such as wind speeds.
}

\begin{figure}
\centering
\includegraphics[width=0.5\textwidth]{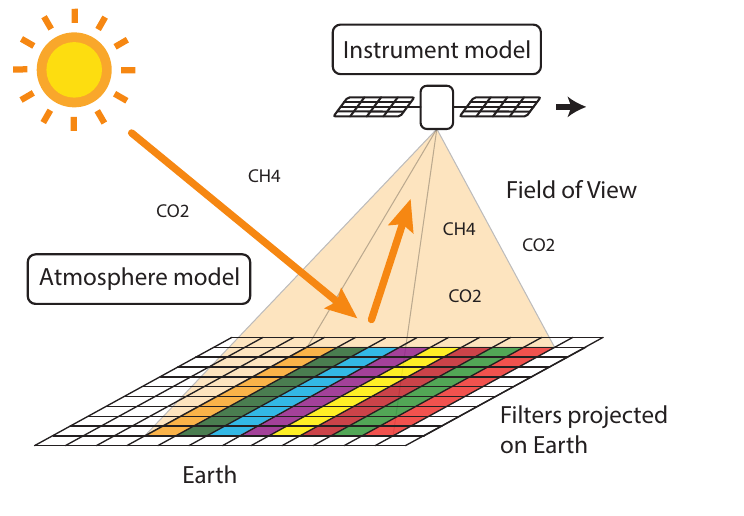}
\caption{Illustration of the trace gas measurement device concept. Sunlight travels through the Earth's atmosphere, is reflected and is measured by the trace gas measurement device by a camera with different photonic crystal filters. As the instrument flies over the earth, each ground pixel is measured with each filter.}
\label{fig:instrumentconcept}
\end{figure}

\begin{figure}
\centering
\includegraphics[width=0.5\textwidth]{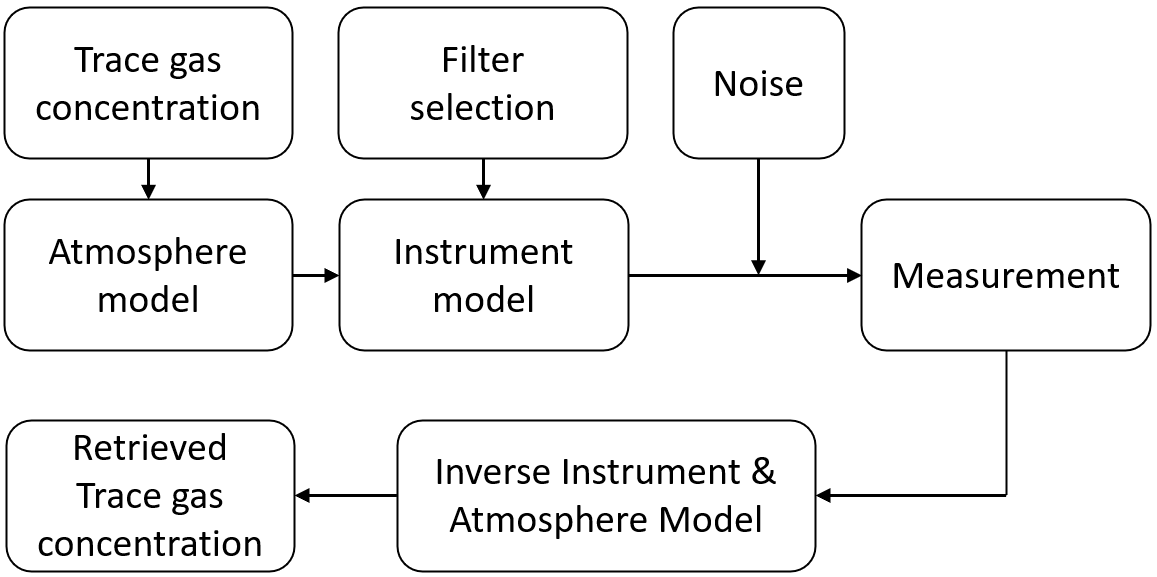}
\caption{Diagram of the trace gas measurement device description.}
\label{fig:flowdiagram}
\end{figure}

\upd{As the light is incident on the filters}, different integrated intensities for each filter are measured for a single ground resolution element. From this detector data, the gas concentrations can be retrieved. The retrieval model is described in previous work~\cite{Kohlhaas2024}. %
In short, the trace concentration is used as input for the atmosphere model, see Figures~\ref{fig:instrumentconcept} and~\ref{fig:flowdiagram}. The atmosphere model produces the spectrum reflected by the Earth, which is used as input for the instrument model, taking into account the specific filters, which compute the measured values to which the photon and read noise are added. 
From these noisy measurement values, an inverse model of the instrument and atmosphere are used to retrieve the trace gas concentration using a 
fitting routine. By performing this fit for different noisy realizations, the fit precision of the trace gas concentration is assessed. This precision is the performance indicator of the instrument concept and should be as low as possible.

In general, the achieved gas concentration retrieval precision depends on many factors such as optical aperture and ground resolution. The specifications and parameters defining the final performance of the instrument are chosen to be reasonable and achievable. Therefore, the achieved precision after filter selection optimization is a representative and likely achievable retrieval precision. Currently, the expected retrieval error of TANGO (Twin Anthropogenic Greenhouse gas Observers),   %
an ESA Scout mission to measure methane, which consists of a traditional state-of-the-art methane retrieval instrument, is below 1.0\%~\cite{TANGO}. However, the optical design of the proposed instrument using the photonic crystals is expected to be \textit{5-10 times smaller}. 

Prior to the development of optimization algorithms, we constructed a reasonable filter selection, which we will refer to as \emph{SRON baseline}. As this initial solution, we choose filters based on the sharpness of the spectral transmission features. We used the second moment of the Fourier transform of the transmission profile to rank all filters. Filters with sharp spectral features would result in large amplitudes for the high-frequency components. From this ranked list, 160 filters are chosen with the largest second moment, and this filter set is used as a reference solution. 

\section{Formulation of OFS problem}
\label{sec:statement}

The goal of OFS is to find filters from the library that together result in the smallest value of the given cost criteria.
In this section, we make this general formulation concrete.
At first, a specific version of the \emph{search space} of all possible solutions of OFS, 
is proposed in  Section~\ref{sec:searchSpace}.
Then the concrete \emph{objective function}, which defines the cost of every solution, is proposed in Section~\ref{sec:objf}.
We interchangeably refer to the value returned by the objective function as to the cost or quality.
Notably, we always consider the minimization of this function, therefore we refer to the solutions with small values of the objective function as high-performing or high-quality.

\subsection{Search Space}
\label{sec:searchSpace}


The search space is defined as sequences of size $N = 640$ of the filters from the library $\mathbb{L}, L \coloneqq |\mathbb{L}| = 4374.$
One filter can be \utwo{included} multiple times in the sequence.
Let us denote the space of all such sequences as $\mathbb{L}^N$.
We assume that the cost criterion is given as a function $F_N : \mathbb{L}^N \to \R.$
Our goal is to minimize this criterion, meaning find such a sequence $\bm{x}^* \in \mathbb{L}^N$ that $F_N\!\br{\bm{x}^*} = \min\!\set{F_N(\bm{x}) \mid \bm{x} \in \mathbb{L}^N}.$

\utwo{Although heuristic algorithms are specifically designed to navigate large and complex search spaces without exhaustive enumeration, their effectiveness is still fundamentally limited by the size and structure of the search space. In extremely large combinatorial spaces such as $\mathbb{L}^N$, even heuristics face difficulties due to several factors: the high likelihood of getting trapped in local optima, the sparsity of high-quality solutions relative to the total number of candidates, and the limited number of evaluations typically allowed due to computational constraints. As a result, the probability of encountering a solution of reasonably good quality becomes increasingly small as $N$ grows, especially in the absence of strong guidance mechanisms. This motivates our restriction to a smaller subspace $\mathbb{L}^M$, with $M \ll N$, where heuristic search becomes more tractable while still retaining sufficient diversity for effective optimization.}
Limiting the number of filter types provides a practical benefit: reduced manufacturing costs due to a less complex production process for a smaller variety of different filters, as discussed in Section~\ref{sec:description-trace-gas-measurement-device}.

We encode filters in the set $\mathbb{L}$ by integers $1, 2, \ldots, L.$
Given a vector of $M$ filters, we create a vector of $N$ filters by repeating each filter multiple times.
More precisely, if $k \coloneqq \lfloor N / M \rfloor$ and $r \coloneqq N - k \cdot M, $ then the first $r$ filters are repeated $k + 1$ time and the rest $M - r$ filters are repeated $k$ times.
We can see that the size of the final selection is exactly $N$, indeed $r \cdot (k + 1) + (M - r) \cdot k = r + k \cdot M = N - k \cdot M + k \cdot M = N.$
For example, if $N = 5, M = 3$, then $k = 1, r = 2.$
If the selection of size $M$ is $(2, 4, 7)$, then the restored selection of size $N$ is $(2, 2, 4, 4, 7).$
Note that $r = 2$ first filters, meaning filters $2$ and $4$, are repeated $k + 1 = 2$ times and the last filter $7$ is repeated $k = 1$ time.

Let us denote the mapping that does this transformation as $\Phi : \mathbb{L}^M \to \mathbb{L}^N.$
Given a vector of filters $\bm{x} \in \mathbb{L}^M$ we can assess its cost by $\F : \mathbb{L}^M \to \R, \F \coloneqq \F_N \circ \Phi.$ 
Notation $f \circ g(x)$ stands for $f(g(x))$.
Now, we are ready to define the problem in the form, which is addressed in this paper.

\begin{formulation}[OFS problem in the form addressed in this work]
For the given constant $M$, let $\mathscr{L} = \set{\bm{x} = (x_1, x_2, \dots, x_M) \mid x_i \in \mathbb{L}}$ be the space of candidate solutions.
For function $\F : \mathscr{L} \to \R$ such that $\F(\bm{x}) = \F_N \circ \Phi (\bm{x})$, OFS problem is to find $\bm{x}^* \in \mathscr{L}$ which satisfies the Eq.~\eqref{eq:ofs-optimization-main}:

\begin{subequations}
    \begin{numcases}{}
        \F(\bm{x}^*) = \min_{\bm{x} \in \mathscr{L}}{\F(\bm{x})} \label{eq:ofs-optimization:goal} \\
        M = |\set{x_1, x_2, \dots, x_M}| \label{eq:ofs-optimization:different}
    \end{numcases}
    \label{eq:ofs-optimization-main}
\end{subequations}
\label{formulation:ofs-reduced}
\end{formulation}

The cardinality of the space $\mathscr{L}$ equals $L^M$, which is large in our particular instance of OFS: $L \approx 5000, M = 16, L^M \approx 10^{59}$. 
However, all permutations of components in a candidate solution $\bm{x} \in \mathscr{L}$ result in the same vector, so they all have the same associated cost.
We show how to exploit this property for numerical optimization in Section~\ref{sec:dd}.

The OFS in the form defined in Formulation~\ref{formulation:ofs-reduced} is convenient for \utwo{heuristic optimization}, since every component in a candidate solution can be considered independently from others.
\utwo{From a practical perspective, the constraint in Eq.~\eqref{eq:ofs-optimization:different} is soft, meaning that minor violations are acceptable. This is because solutions with an arbitrary number of filters can still be manufactured in practice. The restriction to exactly $M$ filters is introduced primarily to reduce the combinatorial complexity of the search space, rather than to enforce a strict physical limitation.}

\subsection{Design of the Objective Function}
\label{sec:objf}

This section \utwo{proposes} a criterion to estimate the quality of the candidate solution.
Let us denote the set of concentrations of different gases as $\mathscr{C}$ and the set of all other possible configurations of TGMD as $\mathscr{R}$.
\upd{Example of configurations in $\mathscr{R}$ are parameters of the exact gas concentration retrieval algorithm}.
\utwo{Atmospheric noise is modeled by a random variable $\xi$ drawn from a noise space $\Omega$. 
A new instance of $\xi$ is sampled from $\Omega$ each time a gas measurement simulation is performed, reflecting the stochastic nature of atmospheric variability.}
In our approximation \upd{of gas receiving by TGMD}, we consider the environment \upd{where TGMD is located} to be homogeneous, \upd{meaning that noise does not depend on the time of the measurement.}
Therefore, the probability measure $P\left(\set{\xi}\right)$ is chosen in advance and does not depend on time for all $\xi \in \Omega$.
When all the parameters of the gas measurement device are fixed, it is modeled as a function $$\mathrm{S}: \mathscr{L} \times \mathscr{R} \times \mathscr{C} \times \Omega \to \mathscr{C}.$$
In the context of our application, we consider only methane as a gas to detect, so the set $\mathscr{C}$ contains only different concentrations of methane.
Therefore, in our case $\mathscr{C} \subset \R$.
Since object $\xi \in \Omega$ is taken by chance, the function $\mathrm{S}$ defines a stochastic process. %
%
%
For the convenience of the reader, we gather all notation used here in Table~\ref{tbl:notation}.

\begin{table*}
    \caption{Numenclature used in this paper.}
    \label{tbl:notation}
\scalebox{0.9}{
    \begin{tabular}{p{0.3\linewidth}p{0.75\linewidth}}
        Notation & Meaning \\
        \hline
        $\mathbb{L}$ & set of all available filters; \\
        $L$ & $|\mathbb{L}| = 4374$, number of available filters; \\
        $N$ & $640$, the length of the sequence of filters; \\
        $M$ & $16$, the dimensionality of candidate solutions to OFS\utwo{, which is the number of selected filter types}; \\
        $K$ & $10^3$, number of independent samples of a random variable; \\ 
        $\mathscr{L}$ & $\mathbb{L}^M$, search space of OFS; \\
        $\dL$ & metric on the space $\mathbb{L}$; \\
        $\mathrm{S}(\bm{x}, \bm{r}^*, c^*, \xi)$ & function which denotes the result of the recovery of a certain gas; \\
        $\Sdiv(\bm{x}, \xi)$ & stochastic process defined in Eq.~\eqref{eq:Sdiv} that assess\utwo{es} the quality of solution $\bm{x}$ to OFS problem; \\
        $\bm{x}, \bm{y}, \dots$ & lowercase bold letters denote vectors; \\
        $\bm{X}, \bm{Y}, \dots$ & capital bold letters denote matrices; \\
        $x_i, y_i, \dots$ & letter of this font with subscript denotes the $i$-th element of vectors $\bm{x}, \bm{y}, \dots$; \\
        $\bm{X}_{i, :}, \bm{Y}_{i, :}, \dots$ & this subscript denotes the $i$-th row of matrices $\bm{X}, \bm{Y}, \dots$; \\
        $X_{i, j}, Y_{i, j}, \dots$ & letter of this font with subscript denotes the number in $i$-th row and $j$-th column of matrices $\bm{X}, \bm{Y}, \dots$; \\
        $\rm{x}, \rm{y}, \rm{X}, \rm{Y}, \dots$ & letters of this font denote random variables; \\
        $\BF{x}, \BF{y}, \dots$ & letters of this font denote random vectors; \\
        ${\BF{X}}, {\BF{Y}}, \dots$ & letters of this font denote random matrices; \\
        $\xi^{(i)}, \bm{x}^{(i)}, \BF{x}^{(i)},$ $\BF{X}^{(i)}, \dots$ & superscript denotes the number of the object. Objects with different numbers but the same letter belong to the same domain;  \\
        $\F(\bm{x}), \Fest_K(\bm{x})$ & objective function of OFS task \F and its noisy estimate \Fest given $K$ following samples: $\set{\Sdiv^2\!\br{\bm{x}, \xi^{(i)}}}_{i=1}^K$; \\ 
        $\E(\mathrm{x}), \Eest_K(\mathrm{x})$ & mean and sample mean of a random variable $\mathrm{x}$ given its $K$ samples; \\
        $\Var(\mathrm{x}), \Varest_K(\mathrm{x})$ & variance and sample variance of a random variable $\mathrm{x}$ given its $K$ samples; \\
        $\overline{a,b}$ & set of integers $\set{a, a+1, \dots, b}$; \\
        $\uar{\mathbb{X}}$ & select an element uniformly at random from set $\mathbb{X}$; \\
        $\uar{a, b}$ & select a real value uniformly at random from the segment $[a,b] \cap \R$.
    \end{tabular}
}
    
\end{table*}

As described in Section~\ref{sec:description-trace-gas-measurement-device}, we aim to locate the best filters $\bm{x} \in \mathscr{L}$ for the single gas concentration $c^* \in \mathscr{C}$ and single configuration of the system $\bm{r}^* \in \mathscr{R}$.
Therefore, we assume $c^*$ and $\bm{r}^*$ fixed and so simplify the notation to define a function $F$ that qualifies the performance of the chosen $\bm{x} \in \mathscr{L}$.
This function is defined on the space $\mathscr{L}$ and returns a real value from $\R$, which denotes the error made in gas retrieval.
\begin{requirements}Informally, filter set $\bm{x}$ should be qualified better than $\bm{y}$ by the function $F$ if $\bm{x}$ leads to the same or better precision in the retrieved concentration of gas as $\bm{y}$, but with a smaller ``\upd{uncertainty}''.
\label{req:objf-informal}
\end{requirements}
This requirement is motivated by a wish to obtain a solution that allows for precise detection of the gas with strong confidence.
Now we will give a more formal interpretation of our intuitive requirement.

For the fixed in advance $\bm{r}^*$ and $c^*$, let us introduce in Eq.~\eqref{eq:Sdiv} a stochastic process parameterized by $\bm{x} \in \mathscr{L}$: 
\begin{equation} \Sdiv(\bm{x}, \xi) \coloneqq 1 - \dfrac{\mathrm{S}(\bm{x}, \bm{r}^*, c^*, \xi)}{c^*}. \label{eq:Sdiv} \end{equation}
Here we omit $c^*$ and $\bm{r}^*$ in arguments of $\Sdiv$ because we assume them to be fixed in advance, and $\xi$ is decided by chance.
Given the probabilistic measure $P$, we express the precision in the gas retrieval when filters $\bm{x} \in \mathscr{L}$ are used, as the mathematical expectation of the random variable $\Sdiv(\bm{x}, \xi)$, written as $\E\cl{\Sdiv\br{\bm{x}, \xi}}$.
The \upd{uncertainty in the gas retrieval} is then expressed as the variance of the random variable $\Sdiv(\bm{x}, \xi)$, written as $\Var\cl{\Sdiv\br{\bm{x}, \xi}}$.                                                         
Following the conventional definition of Stochastic Integer Programming (see Definition 3 in~\cite{bianchi2009survey}), we define the objective function in Eq.~\eqref{eq:objfunction} to assess the quality of the given selection of filters $\bm{x} \in \mathscr{L}$:
\begin{equation}
    \F(\bm{x}) \coloneqq \E\left[\Sdiv^2(\bm{x}, \xi)\right].
    \label{eq:objfunction}
\end{equation}

The proposed OFS objective Eq.~\eqref{eq:objfunction} both meets our intuitive requirements and admits an analytic characterization of the mean-variance trade‑off. 
We propose an analysis of its properties in the~\ref{sec:objf:properties}, and here we keep only the conclusions. 
The function compromises the variance and mean of the retrieved gas concentration.
In the conflicting case of these objectives, we propose the necessary and sufficient conditions for qualifying a solution as better than others. 
Together, these results guarantee how much increase in mean is required to compensate any rise in variance, providing justification for the behavior of $F$ in selecting certain filter combinations.

\subsubsection{Computation of the Objective Function on a Given Solution}
\label{sec:objf:noise}
It is not possible to compute moments of random variable $\mathrm{D}^2(\bm{x})$ precisely due to \utwo{absence of the analytical form of} the simulation $\mathrm{S}$. 
In this regard, we estimate the value of $\F(x)$ using the sample mean.
Specifically, we sample $K$ times the value of random variable $\Sdiv^2(\bm{x}, \xi)$ and compute an approximation $\Eest_K\left[\Sdiv^2(\bm{x}, \xi)\right]$ of $\E\left[\Sdiv^2(\bm{x}, \xi)\right]$ by finding the arithmetical average of the sampled values.
Then, the approximation of $\F(\bm{x})$ is \begin{equation} \Fest_K(\bm{x}) \coloneqq \sum\limits_{i=1}^K \dfrac{\Sdiv^2\!\br{\bm{x}, \xi^{(i)}}}{K} . \label{eq:approx-objf} \end{equation}

Following our definition, $ \Fest(\bm{x}, K)$ is a noisy function, where the order of magnitude of noise depends on the value of $K$ used in the approximation. 
In order to study this random variable, we impose the following Assumption~\ref{assumption:varience}.

\begin{assumption}
    For every $\bm{x} \in \mathscr{L}$ the value of $\Var\! \left[ \Sdiv^2(\bm{x}, \xi) \right]$ is finite.
    \label{assumption:varience}
\end{assumption}

This is a realistic assumption because the distributions of random variables $\Sdiv(\bm{x}, \xi)$ are close to normal for several $\bm{x}$ that we considered, for example, \utwo{it becomes clear later in} Figure~\ref{fig:noise-comparison}. 
Therefore, we can reasonably assume that the distribution of random variable $\Sdiv^2(\bm{x}, \xi)$ is not heavy-tailed for every $\bm{x} \in \mathscr{L}$.

Application of Central Limit Theorem (CLT)~\cite{parthasarathy2005introduction, casella2024statistical} gives us that the distribution of noise converges to the Gaussian distribution: 
\begin{equation}
\Fest_K(\bm{x}) \xrightarrow{d} 
\mathcal{N}\!\left(\E\!\left[\Sdiv^2(\bm{x}, \xi)\right], \dfrac{\Var\!\left[\Sdiv^2(\bm{x}, \xi)\right]}{K}\right) . 
\label{eq:sample-mean-normality}
\end{equation}
It means that the distribution of $\Fest_K(\bm{x})$ can be approximated with Gaussian distribution when a big enough value of $K$ is used.
In this case, when $T$ samples $\left(f_i\right)_{i=1,2,\dots,T}$ of random variable $\Fest_K(\bm{x})$ are observed, the maximum likelihood estimator of mean in Gaussian Distribution is known to be: $$\sum\limits_{i=1}^T \dfrac{f_i}{T} = \Fest_{(K\cdot T)}(\bm{x}).$$
In this regard, we consider values of $\Fest_K(\bm{x})$ with a big value of $K$, specifically $K = 10^3$. 

When, for some $\bm{x}, \bm{y} \in \mathscr{L}$, the difference between $\Fest_{K_1}(\bm{x})$ and $\Fest_{K_2}(\bm{y})$ becomes small, we use Welch's unequal variances \textit{t}-test~\cite{welch1947generalization} to reject or accept the null hypothesis that the estimated means are different.
This test imposes the constraint that the sample means of random variables are normally distributed, which is satisfied in our case due to the chosen large value of $K$ in Eq.~\eqref{eq:sample-mean-normality}.
This test was chosen because 1) we do not assume that the variances of the compared random variables are equal, and 2) the test allows us to compare populations with different numbers of samples, meaning $K_1 \ne K_2$, which we find convenient for implementation.
\upd{We apply this test in Algorithm~\ref{alg:n} described in Section~\ref{sec:metrics-cmp}.}

\section{Optimization of OFS problem}
\label{sec:optimization}

A representative selection of metaheuristics and Machine Learning approaches is considered to solve the OFS formulated in Section~\ref{sec:searchSpace} in the Problem Formulation~\ref{formulation:ofs-reduced} with the objective function \F defined in Section~\ref{sec:objf}.
The choice of these numerical methods is justified for the reasons discussed in Section~\ref{sec:introduction}.
We discuss the particular selected algorithms and then compare the results of their application to OFS in Section~\ref{sec:optimization:comparison}. Sections~\ref{sec:optimization:selection:GAs},~\ref{sec:optimization:selection:RSM},~\ref{sec:optimization:selection:MBs} correspond to the classes of algorithms defined in~\cite{amaran2016simulation}, which provides a comprehensive review of simulation optimization algorithms and their applications, ensuring that our selection of methods is grounded in established best practices within the field.

We acknowledge the possibility that there might be other optimization algorithms, not considered in this study, that could potentially identify even better solutions to OFS within the given computational limits.
However, the chosen algorithms demonstrate the effectiveness of applying optimization techniques to our OFS problem.
Notably, several algorithms outperform the baseline approach described in Section~\ref{sec:description-trace-gas-measurement-device}.
This success establishes the feasibility of efficient OFS optimization within the defined time constraints and highlights the potential for obtaining high-quality solutions.
Future work can investigate both unaddressed optimization algorithms and more advanced versions of the ones already considered.
This focus on further improvement becomes even more relevant as we plan to enhance the realism of our simulator, leading to more computationally expensive simulations.

\subsection{Evolutionary Algorithms}
\label{sec:optimization:selection:GAs}

Randomized search heuristics (RSH) are off-the-shelf algorithms to solve complex optimization problems~\cite{auger2011theory, rozenberg2012handbook}.
Such algorithms define the distribution over candidate solutions, sample populations of candidates from the distribution and usually adjust the distribution based on the qualities of the sampled candidates.
Evolutionary Algorithms~\cite{rozenberg2012handbook, jansen2013analyzing} are nature-inspired RSH that use metaphors of evolution.
These algorithms produce descendant candidate solutions by a sequence of reproductions, selections, recombinations, and mutations.
The quality of those solutions is evaluated by the objective function.

The pseudo-boolean search space $\set{0, 1}^n$ for fixed constant $n \in \N$ is usually considered in the context of discrete optimization.
$(1 + 1)$ EA is an Evolutionary Algorithm, which was rigorously analyzed on several pseudo-boolean toy problems, for example, see~\cite{jansen2013methods, doerr2020theory}.
The algorithm maintains a candidate solution $\BF{x}_{\text{best}} \in \{0, 1\}^n$ with the smallest (in the case of minimization) observed value of the objective function $F$.
The optimization process works in multiple iterations.
At every iteration of the discussed $(1+1)$ EA, a number $k$ is sampled from Binomial Distribution $\mathcal{B}(n, r/n)$.
Here $n$ is the size of the vector, which represents the candidate solution, and $r \in (0, n]$ is a real constant, which represents the \emph{mutation rate}.
If $k = 0$, then we change it to 1 (so-called ``shift'' strategy), since it is more efficient for practical purposes~\cite{pinto2018towards}.
Then $k$ components in the solution $\BF{x}_{\text{best}}$ are chosen uniformly at random.
A new candidate solution $\BF{x}_{\text{new}}$ is created by flipping the bits in each of the $k$ chosen components.
The new candidate solution $\BF{x}_{\text{new}}$ is taken as new observed solution if and only if $F(\BF{x}_{\text{best}}) \ge F(\BF{x}_{\text{new}})$.
Such iterations repeat until the computation budget is exhausted, or other termination conditions are satisfied.

\subsubsection{$(\mu + \lambda)$ EA}
\label{sec:ea-simple}

\utwo{In this paper we study a variant of the $(1+1)$ EA for integer variables as proposed in~\cite{doerr2018static}. Each variable can take any integer value from 1 to $c$. During mutation, a selected variable is replaced by an integer chosen uniformly at random from 1 to $c$. In~\cite{doerr2018static} this operator is called uniform mutation and it includes a check that the new value must differ from the old one. In our implementation we skip that check. This means we allow the new value to be the same as the old one. Since $c$ is large and equals the number of filters $L$, the probability of drawing the same value twice is very small. Apart from this simplification, we follow the standard $(1+1)$ EA rule for choosing how many variables to mutate.}

$(\mu + \lambda)$ EA is an implementation of this algorithm with a bigger population size for solving OFS. 
In this algorithm:
\begin{itemize}
    \item Population Size $\mu$: a population of $\mu$ candidate solutions is maintained.
    \item Selection and Mutation: during each iteration $\lambda$ offspring are created. 
    Each offspring is generated in two steps: 1) a single parent solution is selected uniformly at random from the current population of $\mu$ candidates; 
    2) The \utwo{uniform} mutation operation described \utwo{in this section} is applied to this parent solution.
    \item Elitist Survival Selection: The algorithm then combines the parent population ($\mu$ candidates) with the newly generated $\lambda$ offspring ($\mu + \lambda$ total solutions). 
    From this combined pool the algorithm selects the $\mu$ best individuals (those with the smallest objective function values) to form the population for the next iteration. 
\end{itemize}

\utwo{To justify our design choices in the $(\mu+\lambda)$ EA, we note the following points.
We choose a population of size $\mu$ and generate $\lambda$ offspring in each generation to balance exploration and computational cost. 
A larger population improves diversity and reduces the risk of premature convergence. 
Selecting each parent uniformly at random maintains an unbiased search, meaning that it avoids over‐emphasising any individual.
We apply elitist survival selection ensures that the best solutions are never lost and that the algorithm makes steady progress toward better filters for the OFS problem.}

In our implementation, we used $\mu = 10, \lambda = 20$.

\subsubsection{$(\mu/2 + \lambda)$ EA}
\label{sec:ea-simple-cross}

Apart from mutations, evolutionary algorithms use crossover to generate new candidate solutions~\cite{back2012handbook}.
In this work, we consider uniform crossover which produces new candidate solutions by combining two given solutions.
Specifically, given solutions $\BF{x}, \BF{y}$ uniform crossover generates solution $\BF{z}$ such that every component $z_i$ is set to either $x_i$ or $y_i$ with equal probability. \utwo{The Paired Crossover Evolutionary Algorithm~\cite{prugel2015run} relies solely on uniform crossover and selection. It has strong theoretical guarantees on simple test functions and it works well on various SCOPs~\cite{aishwaryaprajna2019noisy, aishwaryaprajna2022noisy}.} 

\utwo{However, this algorithm only uses the values that appeared in the initial population. In the OFS problem, we need to explore new combinations of filters from a large library. To overcome this limitation, we add our mutation operator to the procedure. We embed uniform crossover and mutation in the $(\mu+\lambda)$ EA framework from Section~\ref{sec:ea-simple}.} 

\utwo{In our $(\mu/2 + \lambda)$ EA, we pick two parents at random from the current population of size $\mu$. We apply uniform crossover to these parents to produce one child. Then we propose to apply the mutation described in Section~\ref{sec:ea-simple}. We repeat this process $\lambda$ times each iteration to generate $\lambda$ offspring. We combine the $\mu$ parents with the $\lambda$ offspring and select the best $\mu$ solutions to form the next population. The notation $\mu/2$ highlights that two parents are involved in each crossover step.}

\subsubsection{Integer Fast GA}
\label{sec:integer-fast-ga}

The Nevergrad platform~\cite{rapin2018nevergrad} contains many general-purpose black-box optimization meta-heuristics.
The one that the authors recommend to apply for a discrete problem at hand is Fast GA.
It is an extension of $(1 + 1)$ EA with a scheme to control the mutation rate by sampling it from a power-law distribution as proposed in~\cite{doerr2017fast}.
We applied a particular instance of Fast GA, namely algorithm \textit{DoubleFastGADiscreteOnePlusOne} from Nevergrad, which extends the mutation rate control to integer domains~\cite{descreteBenchmarks}.
For the sake of shortness, we refer to \textit{DoubleFastGADiscreteOnePlusOne} as Integer Fast GA.

\subsubsection{Noisy Portfolio}
\label{sec:noisy-portfolio}

This algorithm corresponds to the method called \textit{RecombiningPortfolioOptimisticNoisyDiscreteOnePlusOne} in the library Nevergrad~\cite{rapin2018nevergrad}.
According to the documentation in the source, this algorithm constitutes an extension of $(1 + 1)$ EA.
The extension includes genetic crossover applied at every iteration and a so-called optimistic noise handling scheme.
One can find more details about those methods in the source code of Nevergrad~\cite{rapin2018nevergrad}.
Moreover, the mutation rate in this algorithm is controlled during the optimization process using uniform mixing proposed in~\cite{dang2016self}.

\subsection{Response Surface Methodologies}
\label{sec:optimization:selection:RSM}

Response Surface Methodology (RSM) is an optimization algorithm, where a model, called response surface, approximates the behavior of the objective function based on a set of known input-output data points~\cite{amaran2016simulation}.
RSM starts optimization by creating a Design of Experiment (DoE), meaning a set of points \utwo{(for example uniformly selected)}, which is used to build the initial model.
Then RSM iteratively uses the surface as a surrogate to find the next point, where the objective value is evaluated.
This point is then added to the model.
Those iterations are repeated until the computational budget is exhausted.

\subsubsection{Bayesian Optimization}
\label{sec:bo}

Bayesian Optimization (BO)~\cite{mockus1975bayesian} and Efficient Global Optimization (EGO)~\cite{jones_efficient_1998} are particular cases of RSM, where the response surface constitutes a statistical model.
We considered the conventional Gaussian Process (GP)~\cite{williams2006gaussian} as the probabilistic model.
This model is especially appealing for our application because it allows modeling the inherent Gaussian noise in the objective function.
The time complexity of the algorithm that creates GP is $O(n^3)$, where $n$ is the number of points added to the model~\cite{quinonero2007approximation}.
It becomes computationally inefficient to add more than $10^3$ points to GP, therefore we limit the computational budget for BO by $10^3$ evaluations of $\Fest_K$.
We used BO implementation from skopt library~\cite{skoptBO}.
The number of the DoE points was set to $110$, and the noise parameter of the algorithm was set to Gaussian, due to the reasons discussed in Section~\ref{sec:objf:noise}.
For each of the $M$ components, we defined the bounding box, which restricts the search space for BO, to be $[1, L]$.
All float values that are encountered in the evaluated points are converted to integers by disregarding the decimal parts.
The rest of the parameters of BO were set to their default values used in skopt.

\subsection{Model-Based Methods}
\label{sec:optimization:selection:MBs}

Model-based methods maintain probabilistic models to navigate in the search space.
These models are built and refined iteratively to guide the search towards promising regions.
Here we consider two algorithms that fall under this category.
Generally speaking, those algorithms could also be classified as Evolutionary Algorithms.
Notably, those optimizers are different from the ones described in Section~\ref{sec:optimization:selection:GAs} because the model-based methods learn and adapt their inherent models throughout the optimization process.

\subsubsection{$(\mu/2, \lambda)$ Mixed Integer Evolution Strategy }
\label{sec:mies}

Mixed Integer Evolution Strategy (MIES)~\cite{li2013mixed} is proposed to address the problems with domains that contain real and discrete parts.
The probabilistic model is used in this algorithm to sample mutations.
Since in our OFS problem filters are encoded as integers, we apply only part of the algorithm that addresses integer domains.
For such search spaces, MIES borrows the method originally proposed for unbounded sets of values in every component~\cite{rudolph1994evolutionary}.
We used the implementation of this algorithm from the library~\cite{van2019automatic}.
The selection scheme used in the algorithm is non-elitist, meaning only the created $\lambda$ individuals take part in the selection.
Before the mutation, MIES applies uniform crossover to a pair of parents selection uniformly at random from the population of $\mu$ solutions.
In our experiments, we used values $\mu = 10, \lambda = 20$.

\subsubsection{Univariate Marginal Distribution Algorithm }
\label{sec:umda}
 
\begin{algorithm2e}[!tb]
    \caption{UMDA for integer programming, where objective function $F: \mathbb{L}^M \to \R$ is minimized. Value at every component of every point is an integer from the set $\mathbb{L} \coloneqq \overline{1,L}$. Given positive integers $\mu \le \lambda$, the algorithm produces $\lambda$ offspring and selects $\mu$ ones from them. The maximal number of $F$ evaluations is limited by constant $b$.}
    \label{alg:umda}
    For every $i \in \overline{1, M}, j \in \overline{1, L}$ initialize: $\rm{P}^{(1)}_{i,j} \gets 1/L$\;
    \label{alg:umda:pmin} Set min for probability: $p_{\text{min}} \gets \dfrac{1}{(L-1)\cdot M}$\;
    Initialize best-so-far: $\BF{x}^* \gets \uar{\mathbb{L}^M}$\;
    \For{$t \gets 1, 2, \ldots, \lfloor b/\lambda \rfloor $}{
        \For{$k \gets 1, 2, \ldots, \lambda$}{
            \For{$i \gets 1, 2, \ldots, M$}{
                Sample integer: $\rm{x}_i^{(k)} \sim \BF{P}_{i,:}^{(t)}$\;
            }
        }
        Set $\br{\BF{y}^{(k)}}_{k=1}^{\lambda}$ to reordered $\br{\BF{x}^{(k)}}_{k=1}^{\lambda}$:
        $F(\BF{y}^{(1)}) \le F(\BF{y}^{(2)}) \le \ldots \le F(\BF{y}^{(\lambda)})$\;
        \For{$i \gets 1, 2, \ldots, M$}{
            \For{$j \gets 1, 2, \ldots, L$}{
                Update distribution: $\rm{P}^{(t+1)}_{i,j} \gets \dfrac{1}{\mu}\sum\limits_{k=1}^{\mu} \mathbbm{1}\!\cl{\rm{y}_i^{(k)} = j}$\;
                Adjust to lower/upper bounds: $p_{\text{min}} \le \rm{P}^{(t+1)}_{i,j} \le 1 - p_{\text{min}}$\;
            }
        }
        \If {$F(\BF{x}^*) \ge F(\BF{y}^{(1)})$} {
            Update best-so-far: $\BF{x}^* \gets \BF{y}^{(1)}$\;
        }
    }
    \Return $\BF{x}^*$\;
\end{algorithm2e}

Estimation of Distribution Algorithms (EDAs)~\cite{larranaga2001estimation, hauschild2011introduction} explicitly define the probabilistic distribution over the search space of an optimization problem.
Univariate Marginal Distribution Algorithm (UMDA)~\cite{muhlenbein1996recombination} is an instance of EDA suited for problems with pseudo-boolean search spaces.
UMDA maintains a global probability distribution over the search space, considering the value in every component as a random variable.
The core assumption behind UMDA is that those random variables are independent. 
UMDA is an iterative optimization algorithm.
On every iteration, it samples $\lambda$ candidate solutions from the global distribution and uses the first $\mu$ solutions in the ranking according to the objective function to update this probabilistic model. 
Since usually the algorithm is considered only for pseudo-boolean \utwo{or continuous} domains, we propose our version of UMDA for integer domains in Algorithm~\ref{alg:umda}.
In our implementation, we used constants $\mu = 5, \lambda = 50$.

\subsection{Automated Algorithm Selection}
\label{sec:ngopt}

\utwo{In the previous sections, we introduced several optimization algorithms for SCOP. However, the best choice often depends on subtle features of the considered black-box objective function and can only be determined by empirical testing. To avoid having to run every algorithm on every instance of the OFS problem, we employ Automated Algorithm Selection (AAS)~\cite{kerschke2019automated}, which uses measurable problem properties to recommend or switch between optimizers on the fly. Usually, AAS methods are trained for the specific instances of optimization problems and then applied to previously unseen tasks. In our study, we rely on Nevergrad library and its NGOpt to handle this selection process dynamically. NGOpt is a so-called wizard, which selects an optimization algorithm based on the properties of the considered problem and can switch the algorithm during the minimization process~\cite{meunier2021black}.
The properties include the number of dimensions, presence of noise, type of constraints, etc.}

\subsection{Numerical Comparison}
\label{sec:optimization:comparison}

\begin{figure*}[!tb]
    \centering
    \begin{tabular}{ p{0.46\textwidth} p{0.01\textwidth} p{0.46\textwidth} }
         & \adjustbox{valign=m, center}{\includegraphics[width=0.96\textwidth, trim=00mm 00mm 00mm 00mm, clip]{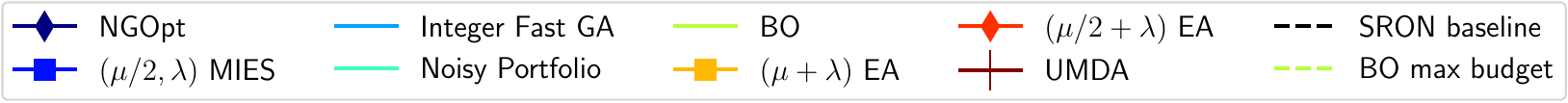}} & \\
        \adjustbox{valign=m, center}{\includegraphics[width=\linewidth, trim=00mm -02mm 8mm 00mm]{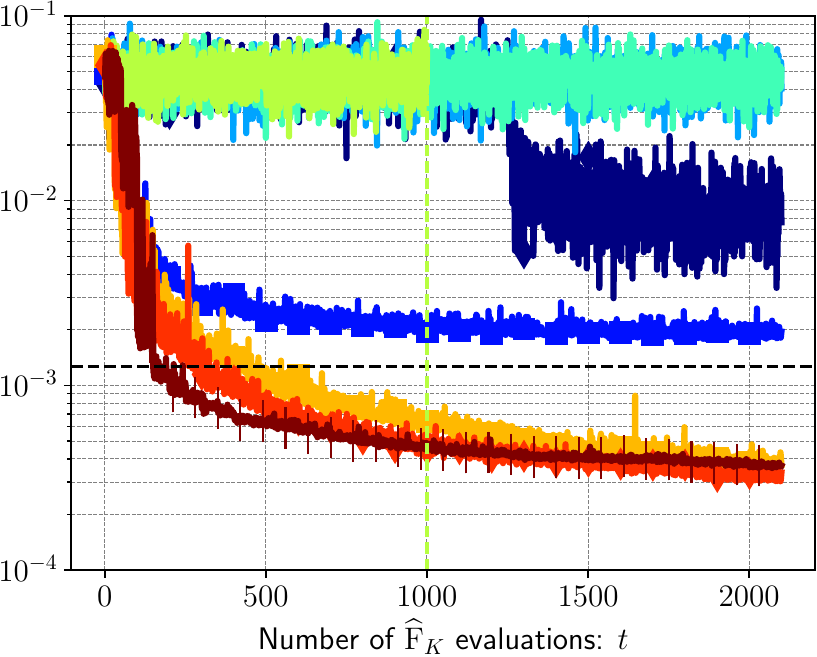}} & &
        \adjustbox{valign=m, center}{\includegraphics[width=\linewidth, trim=8mm -02mm 00mm 00mm]{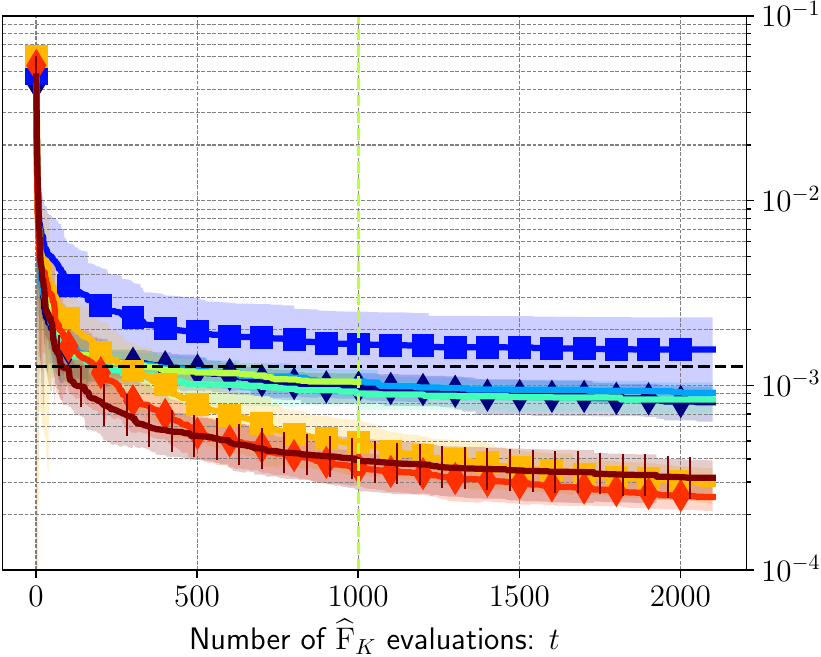}} \\
        \adjustbox{valign=t, center}{ \scalebox{0.8} {\begin{minipage}{0.5\textwidth} \centering \textbf{(a) Average objective value: 
        $ \Eest_n\! \left( \mathrm{f}^{(t)}\!\cl{A(\xi)} \right) $} \end{minipage}} } & \adjustbox{valign=t, center}{} &
        \adjustbox{valign=t, center}{ \scalebox{0.8} {\begin{minipage}{0.5\textwidth} \centering \textbf{(b) Average best-so-far: $ \Eest_n\! \left( \mathrm{g}^{(t)}\!\cl{A(\xi)} \right) $,\\ shaded area: $\sqrt{\Varest_n\! \left( \mathrm{g}^{(t)}\!\cl{A(\xi)} \right)}$ } \end{minipage}} } \\

    \end{tabular}
    \caption{Convergence plot for every algorithm $A \in \mathbb{A}_1$. Part (a) represents the approximated mean of $\mathrm{f}$ defined in Eq.~\eqref{eq:cr1} as the quality of the solution considered by algorithm $A$ when computational budget $t$ is spent. Part (b) represents the approximated mean and standard deviation of $\mathrm{g}$ defined in Eq.~\eqref{eq:cr2} as the quality of best-so-far solution found by algorithm $A$ when computational budget $t$ is spent. The number of independent runs per algorithm is $n = 20$.
    }
    \label{fig:application1}
\end{figure*}

In this section, we demonstrate the application of the selected algorithms to the OFS problem\utwo{, see sections~\ref{sec:optimization:selection:GAs}},~\ref{sec:optimization:selection:RSM},~\ref{sec:optimization:selection:MBs},~\ref{sec:ngopt}. 
Let us denote the set of those algorithms as $\mathbb{A}_1$.
We used $\Fest_K$ as the objective function with fixed constant $K = 10^3$ for all the algorithms.
In each algorithm except BO, we used the computational budget of $2.1 \cdot 10^3$ evaluations of the objective function $\Fest_K$.
For BO we made only $10^3$ evaluations of $\Fest_K$, due to the reasons discussed in Section~\ref{sec:bo}.
Every algorithm is run $n = 20$ times and every time the set of initial solutions is taken uniformly at random from the search space $\mathscr{L}$\utwo{, so the initial population is different for each run of every algorithm.}

In this section, we study how well the $\mathbb{A}_1$ algorithms perform on the OFS problem.
Each algorithm in $\mathbb{A}_1$ is randomized.
Let $\Xi$ be the sample space such that $\xi \in \Xi$ determines the behavior of algorithms.
Therefore, for a fixed $\xi \in \Xi$ and $A \in \mathbb{A}_1$ we will consider $A(\xi)$ as a deterministic algorithm.
Random vector $\BF{x}^{(t)}\!\br{{A(\xi)}}$ denotes the candidate solution that is evaluated when the objective function $\Fest_K$ is called in $t$-th time by algorithm $A$.
As the measure of performance of an algorithm $A$ we consider two following criteria. 
\begin{enumerate}
    \item \textbf{Quality‐of‐solution.}~\label{cr:f} Quality of the solution considered by algorithm $A$ when computational budget $t$ is spent: 
    \begin{equation}\mathrm{f}^{(t)}\!\cl{A(\xi)} \coloneqq \Fest_K\!\cl{\BF{x}^{(t)}(A(\xi))}; \label{eq:cr1} \end{equation}
    \item \textbf{Best-so-far.}~\label{cr:best-so-far} Quality of best-so-far solution found by algorithm $A$ when computational budget $t$ is spent: 
    \begin{equation}\mathrm{g}^{(t)}\!\cl{A(\xi)} \coloneqq \min \set{\Fest_K\!\cl{\BF{x}^{(j)}(A(\xi))} \mid j \in \overline{1, t} }. \label{eq:cr2} \end{equation}
\end{enumerate}

\utwo{We introduce these definitions to be precise about how we measure algorithm performance. Although quality‐of‐solution and best‐so‐far are standard concepts, we did not find their formal definitions in the related literature and so we give them here for clarity and reproducibility.}

The results of the numerical experiment are summarized in Figure~\ref{fig:application1}.
\utwo{All the algorithms are given the same total number of the allowed objective function evaluations, except BO, for the reasons discussed in Sec.~\ref{sec:bo}.}
In both plots on this figure, the X axis depicts the number of times $t$ function $\Fest_K$ was evaluated.
The Y axis in part (a) of Figure~\ref{fig:application1} represents the approximated mean of quality-of-solution criterion.
This plot shows that candidate solutions in Integer Fast Ga, Noisy Portfolio, and BO vary and do not improve on average when the optimization process proceeds.
However, NGOpt starts producing solutions that perform better on average when approximately half of the given computational budget is exhausted.
This can happen due to the switch of the underlying optimization algorithm that NGOpt uses.
Average candidate solutions considered by $(\mu/2, \lambda)$ MIES significantly improve in the first half of the optimization process, but then stagnate until the end.

The Y axis in part (b) of Figure~\ref{fig:application1} represents the approximated mean of best-so-far criterion and the shaded area represents the approximated standard deviation of best-so-far for each algorithm in $\mathbb{A}_1$.
The line for every algorithm in this plot constitutes the lower envelope of the corresponding line in the plot (b) of Figure~\ref{fig:application1}.
All the algorithms except $(\mu/2, \lambda)$ MIES on average produced some solutions that outperform the baseline solution.
It is clear, that algorithm $(\mu + \lambda)$ EA, $(\mu/2 + \lambda)$ EA and UMDA outperform all the other considered algorithms in both criteria; \utwo{quality-of-solution and best-so-far}.
We call those three algorithms \emph{leading algorithms} and aim to improve their performance in our further investigation.

\section{Employing New Distances to Navigate in the Domain of OFS}
\label{sec:dd}

When an optimization algorithm is equipped with domain information it can efficiently disregard areas of the search space with sub-optimal candidate solutions~\cite{antonov2023curing}.
While there are many possible ways to account for problem-specific information, we use a distance metric as proposed in~\cite{antonov2023curing}.
The metric represents our knowledge of the filters in the library $\mathbb{L}$ and the search space symmetry to permutations.

\subsection{\utwo{Definition of distance} between candidate solutions}
\label{sec:distance}

In this section, we propose two distance metrics on the space $\mathscr{L}$.
We remind the following classical definition of the metric space in \utwo{\ref{sec:metric}}.


For this application the distance metric on the set $\mathbb{L}$ should relate to the similarity between different filters. We propose and tested two following distance measures.
\begin{enumerate}
    \item \label{dist:ratio} Ratio of correlation with methane absorption lines, which we call $d_1$;
    \item \label{dist:fourier} Second moment of the Fourier transformed transmission profiles, which we call $d_2$.
\end{enumerate}
The distance metric~\ref{dist:ratio} is the ratio between the total transmission and integrated transmission profile multiplied by the trace gas absorption lines. Let $T^x_q$ be the (discretized) transmission of filter $x$ for wavelength $q$ and $a_q$ the absorption of the trace gas molecule of interest for wavelength $q$. Then the distance metric $d_1(x,y)$ is given by
\begin{equation}
    d_1(x,y) = \left| \frac{\sum_q a_q T^x_q }{\sum_q T^x_q} - \frac{\sum_q a_q T^y_q }{\sum_q T^y_q} \right|.
\end{equation}
This ratio represents how sensitive this method is to differences in trace gas concentration. 
The distance metric~\ref{dist:fourier} computes the difference in the second moment of the absolute value of the Fourier transformed transmission profile. Let $\tilde{T}_{\zeta}$ denote the absolute value of the discrete Fourier transform of the transmission profile and $\zeta$ the associated (spatial) frequency. The distance metric~\ref{dist:fourier} is then given by
\begin{equation}
    d_2(x,y) = \left| \sum_{\zeta} \zeta^2 \tilde{T^x_\xi} -\sum_{\zeta} \zeta^2 \tilde{T^y_\zeta} \right|.
\end{equation}
A larger second moment represents a filter with sharper and or more spectral features. 
We numerically validated that there are no different filters $x, y \in \mathbb{L}$ such that $d_1(x, y) = 0 $ or $ d_2(x, y) = 0$.
The other conditions of the metric are satisfied for $d_1, d_2$ due to the properties of the modulo.
Therefore, $d_1, d_2$ are metrics on the space $\mathbb{L}$ (in the mathematical sense).

\subsection{Filter multisets}
\label{sec:multisets}

The described distances explain the relation between filters, however, the objective of the application is to find a \emph{multiset}~\cite{knuth1997art} of filters.
We propose to take into account the fact that the order of filters in a sequence  $\BF{x} \in \mathscr{L}$ does not change the value of the objective function $\F(\BF{x})$.
In this regard, we use linear-assignment problem~\cite{akgul1992linear} for two chosen sequences of filters to compare their similarity. 

\begin{definition}
    For the metric space $(\mathbb{L}, \dL)$ and fixed integer constant $M > 0$ we define a mapping \lap ~on the space $\mathbb{L}^M$ as follows. 
    For any $\bm{x}, \bm{y} \in \mathbb{L}^M$: 
    \[\textsc{LAP}(\dL, \bm{x}, \bm{y}) \coloneqq \min\limits_{\pi \in \mathbb{P}} \set{\sum\limits_{i = 1}^{M}\dL\!\br{x_{i}, y_{\pi(i)}}},\]
    where $\mathbb{P}$ is the set of all bijections (permutations) $\overline{1,M} \to \overline{1,M}$.
\end{definition}

Since permutations of integers in $\bm{x} \in \mathbb{L}^M$ do not influence the value of $\F(\bm{x})$, it is convenient to consider all solutions made of the same number of the same integers as one solution.
Therefore, we define the following equivalence relation $\equiv$ on the space $\mathbb{L}^M$.

\begin{definition}
    Any two points $\bm{x}, \bm{y} \in \mathbb{L}^M$ are in relation $\bm{x} \equiv \bm{y}$ if $$\textsc{LAP}(\dL, \bm{x}, \bm{y}) = 0.$$
    This relation induces the quatient space $\mathbb{M}(\dL) \coloneqq \mathbb{L}^M / \equiv .$
    \label{def:quatient}
\end{definition}

Formally, elements of the space $\mathbb{M}(\dL)$ are sets, which are called equivalence classes: $\set{\bm{x}, \bm{y} \in \mathbb{L}^M \mid \bm{x} \equiv \bm{y}}.$
When we define a metric on the space $\mathbb{M}(\dL)$ we show the axioms of metric on points $\bm{x}, \bm{y}, \bm{z} \in \mathbb{L}^M.$
This can be trivially extended to all the elements from the corresponding equivalence classes of $\bm{x}, \bm{y}, \bm{z}$ using properties of the quotient space.
\utwo{We demonstrate that all the axioms of the metric are upheld in~\ref{sec:metric}, therefore $\textsc{LAP}(\dL, \cdot, \cdot)$ is metric on the space $\mathbb{M}(\dL)$ for any metric $\dL$.}

\subsection{Distance-Driven Versions of $(\mu + \lambda)$ EA and $(\mu/2 + \lambda)$ EA}
\label{sec:dd-ea}

\newcommand{\figuretext}[1]{\scalebox{0.8}{\begin{minipage}{1.2\linewidth} #1 \end{minipage}}}


\begin{figure}[!tb]
    \centering
    \includegraphics[width=0.8\linewidth]{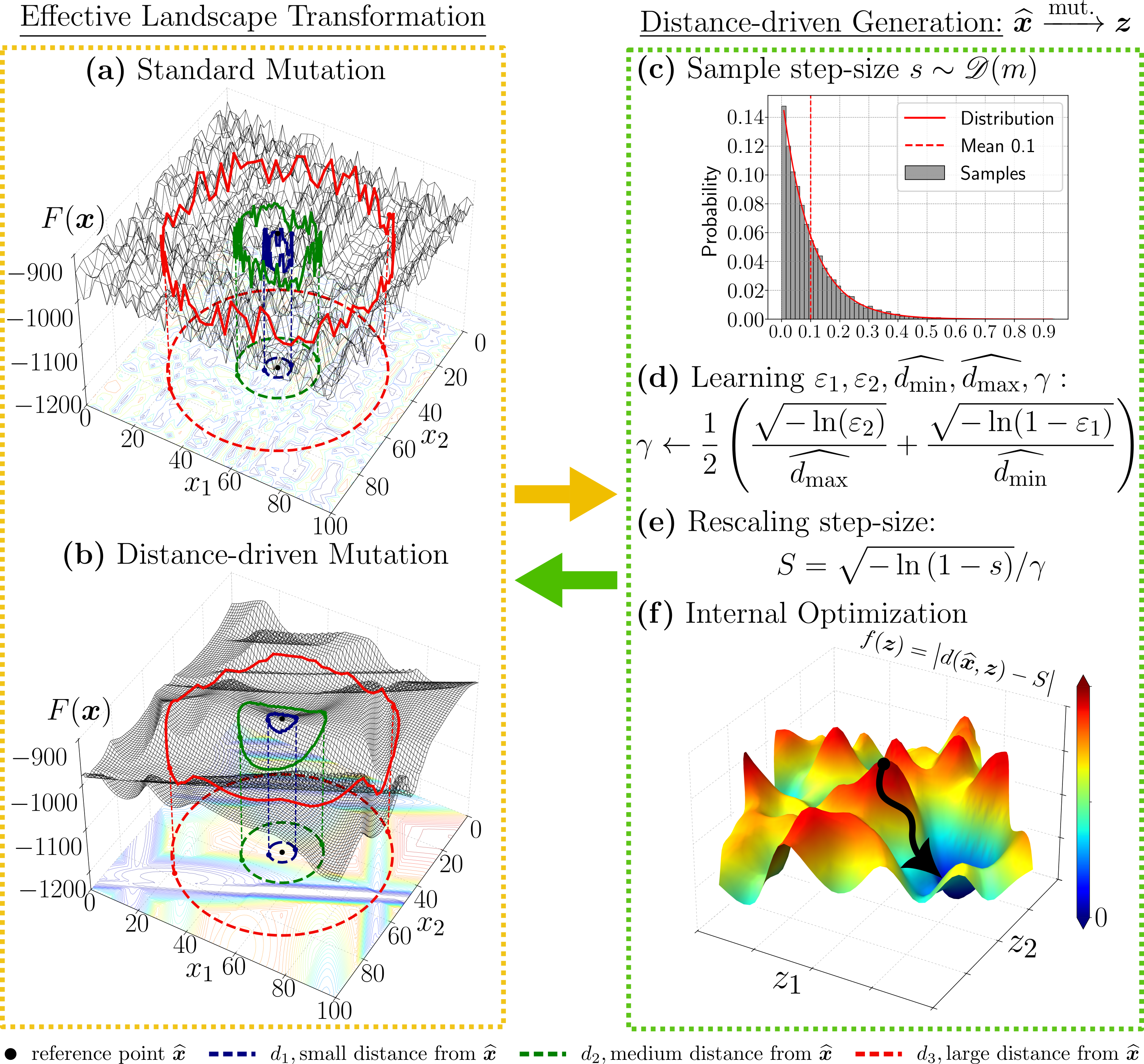}
    \caption{Distance-driven optimization. 
    \textbf{Left:} Standard mutation (a) samples points on contours $d_1,d_2,d_3$ of $F(\bm{x})$; distance-driven mutation (b) “flattens” objective values on closer contours to the reference point $\bm{\widehat{x}}$ in the transformed landscape.  
    \textbf{Right:} To generate a new candidate $\bm{z}$ in the transformed landscape, (c) draw step‐size $s \sim \mathscr{D}(m)$ (see Section~\ref{sec:step-size-distr}), (d) learn scaling parameter $\gamma$ (see Section~\ref{sec:gamma}), (e) rescale the step-size $ S = {\sqrt{-\ln(1-s)}}/{\gamma},$ (see Section~\ref{sec:gamma}) and (f) solve $ \min_{\bm{z}}\bigl|d(\bm{\widehat{x}},\bm{ z})-S\bigr| $ to place $\bm{ z}$ exactly at distance $S$ from $\bm{\widehat{x}}$ (see Section~\ref{sec:internal-opt}).}
    \label{fig:dd-opt}
\end{figure}

We apply the method proposed by us earlier in~\cite{antonov2023curing} to employ the distance metric $d$ in EA. 
\uthree{The main steps of this method are visualized in Figure~\ref{fig:dd-opt}.}
It extends a mutation operator in the optimization algorithm, which aims to find at least one $\bm{x}^* \in \mathbb{X}$ such that $F(\bm{x}^*) = \min\nolimits_{\bm{x} \in \mathbb{X}} \set{F(\bm{x})},$ where $\mathbb{X}$ is some search space and $F$ is the objective function defined on it.
In order to make this extension, the method defines a \emph{mapping of the distance} between any two candidate solutions to the interval $(0, 1)$ and creates a \emph{distribution $\mathcal{D}$ over the step-sizes} on this interval\uthree{, visualized in Fig.~\ref{fig:dd-opt}~(c)}.
To do a mutation of the solution $\bm{x} \in \mathbb{X}$, the algorithm samples a step size $s \in (0, 1)$ from the distribution $\mathcal{D}$ and maps this step size back to the distance $S \in \R$ in the original domain $\mathbb{X}$\uthree{, visualized in Fig.~\ref{fig:dd-opt}~(d) and (e)}.
Then it \emph{solves the internal optimization problem} to generate a solution $\bm{y} \in \mathbb{X}$ which steps away from $\bm{x}$ to the distance $S$:
\begin{equation}\bm{y} \in \argmin\limits_{\bm{z} \in \mathbb{X}}\set{|d(\bm{x}, \bm{z}) - S|}. \label{eq:internal-opt} \end{equation}
\uthree{The landscape of this optimization problem is visualized in Fig.~\ref{fig:dd-opt}~(f).
Those steps result in effective transformation of the optimization landscape for every mutation, such that the landscape is flattened on closer contours to the mutated point $\bm{\widehat{x}}$. This is shown in Fig.~\ref{fig:dd-opt}~(a) and (b).}

We explain how we applied every step (emphasized with italic font) of this general methodology in later sections: Section~\ref{sec:gamma}, Section~\ref{sec:internal-opt}, and Section~\ref{sec:step-size-distr}.
Now we directly describe the proposed extension of $(\mu/2 + \lambda)$ EA summarized in Algorithm~\ref{alg:dd-ea}.
The extension of $(\mu + \lambda)$ EA is the same, except that instead of the lines~\ref{alg:dd-ea:repr-selection},~\ref{alg:dd-ea:crossover1}, and~\ref{alg:dd-ea:crossover2} it chooses $\BF{c}$ uniformly at random from $\mathfrak{P}^{(t)}.$

\begin{algorithm2e}[!tb]
    \caption{Proposed DD-$(\mu/2 + \lambda)$ EA, equipped with a distance-driven heuristic for integer programming, where objective function $F: \mathbb{L}^M \to \R$ is minimized. Points in its domain have dimensionality $M$. Value at every component of every point is an integer from the set $\mathbb{L} \coloneqq \overline{1,L}$, which is a metric space with distance $\dL$. The maximal number of $F$ evaluations is limited by constant $b$. On top of that, the algorithm uses integer constants $\mu, \lambda, R$ and a real constant $0\le m \le 1$.}
    \label{alg:dd-ea}
    Select points for exploration for each $i \in \overline{1,R}$: $\BF{e}^{(i)} \gets \uar{\mathbb{L}^M}$\;
    $\gamma \gets \textsc{explore}\br{\set{\BF{e}^{(i)}}_{i=1}^R, \dL}$\label{alg:dd-ea:explore}\;
    Build distribution over step-sizes: $\mathcal{D} \gets \textsc{distribution}(m)$\label{alg:dd-ea:step-size-distr-build}\;
    For each $i \in \overline{1, \mu}$ create a solution for the first population : $\BF{x}^{(i)} \gets \uar{\mathbb{L}^M}$\;
    Initialize population: $\mathfrak{P}^{(1)} \gets \set{\BF{x}^{(i)}}_{i=1}^{\mu}$\;
    \For{$t \gets 1, 2, \ldots, \lfloor b/\lambda \rfloor$ \label{alg:dd-ea:loop}}{ 
        \For{$i \gets 1,2,\ldots,\lambda $ \label{alg:dd-ea:lambda}}{ 
            Select $\BF{p}^{(1)}, \BF{p}^{(2)}$ u.a.r. from $\mathfrak{P}^{(t)}$\label{alg:dd-ea:repr-selection}\;
            \For{$k \gets 1, 2, \ldots, M$ \label{alg:dd-ea:crossover1} }{
                Uniform crossover: $c_k \gets \uar{\set{\rm{p}^{(1)}_k, \rm{p}^{(2)}_k}}$\label{alg:dd-ea:crossover2}\; 
            }
            Sample step-size: $\rm{s} \sim \mathcal{D}$\label{alg:dd-ea:step-size}\;
            $\rm{S} \gets \sqrt{-\ln\!\br{1 - \rm{s}}}/\gamma $\label{alg:dd-ea:transform-dist}\;
            $\BF{x}^{(i)} \gets \textsc{DDMutation}(\BF{c}, \dL, \rm{S})$\label{alg:dd-ea:dd-mutation}\; 
        }
        Set $\set{\BF{y}^{(i)}}_{i=1}^{\mu + \lambda}$ to reordered solutions in $\mathfrak{P}^{(t)} \cup \set{\BF{x}^{(i)}}_{i=1}^{\lambda}$:
        $F\!\br{\BF{y}^{(1)}} \le F\!\br{\BF{y}^{(2)}} \le \dots \le F\!\br{\BF{y}^{(\mu + \lambda)}}$\label{alg:dd-ea:reorder}\;
        $\mathfrak{P}^{(t + 1)} \gets \set{\BF{y}^{(i)}}_{i=1}^{\mu}$\label{alg:dd-ea:survivale-selection}\;
        $\BF{x}^* \gets \BF{y}^{(1)}$\label{alg:dd-ea:update}\;
    }
    \Return $\BF{x}^*$\;
\end{algorithm2e}

The algorithm starts in line~\ref{alg:dd-ea:explore} by exploring the search space as proposed in~\cite{antonov2023curing}. 
The function $\textsc{Explore}$ and value $\gamma$ are describe in Section~\ref{sec:gamma}.
In line~\ref{alg:dd-ea:step-size-distr-build}, the algorithm builds the distribution of step sizes $\mathcal{D}$  using the fixed mean $m = 0.1$.
The distribution is described in Section~\ref{sec:step-size-distr}.

The algorithm \utwo{then} runs the main optimization loop in line~\ref{alg:dd-ea:loop}.
EA iteratively makes changes to the candidate solutions and selects more promising candidates until the budget $b$ is exhausted.
The value $b$ limits the number of objective function evaluations that the algorithm is allowed to make to find an estimate of the optimized function minimizer.
We used $\Fest_K$ as the objective function with fixed $K = 10^3.$

\utwo{We also experimented with application of Welch's T-test to the values of the underlying random variables to distinguish statistically significant differences in the objective function from insignificant ones. 
This application did not lead to significant improvement, and in some cases it even worsened performance of the corresponding algorithms.
Therefore, we leave application of statistical selection methods for future work.}

The candidates for the best estimate are stored in the set $\mathfrak{P}^{(t)}$.
The algorithm makes changes to the elements of this set in the loop in line~\ref{alg:dd-ea:lambda}.
In every iteration of this loop algorithm selects a pair of candidate solutions uniformly at random from the set $\mathfrak{P}^{(t)}$ in line~\ref{alg:dd-ea:repr-selection}. 
This uniform selection is made to avoid a biased search process when the algorithm may prefer one part of the search space over another.

Changes to the selected solutions are made in lines~\ref{alg:dd-ea:crossover2},~\ref{alg:dd-ea:dd-mutation}.
In line~\ref{alg:dd-ea:crossover2} algorithm performs uniform crossover.
In line~\ref{alg:dd-ea:dd-mutation}, the distance-driven mutation is applied.
The details of the function $\textsc{Mutation}$ are discussed in Section~\ref{sec:internal-opt}.
The ``strength'' of the mutation is defined by the number $s \in (0,1)$ sampled from $\mathcal{D}$ in line~\ref{alg:dd-ea:step-size}.
This number is transformed to the distance in $\mathbb{M}(\dL)$ in line~\ref{alg:dd-ea:transform-dist}.
The formula for this transformation is explained in Section~\ref{sec:gamma}.

After the loop in line~\ref{alg:dd-ea:lambda} is finished, algorithm selects $\mu$ most fit individuals from the set $\mathfrak{P}^{(t)}$ and produced candidates in line~\ref{alg:dd-ea:reorder}.
The next population is set to the selected set in line~\ref{alg:dd-ea:survivale-selection}.
Then the current best estimate of the minimizer $\BF{x}^*$ is updated in line~\ref{alg:dd-ea:update}.
When the computation budget is exhausted the algorithm returns this solution $\BF{x}^*$.

\subsubsection{Mapping of Distances.}
\label{sec:gamma}

The algorithm from~\cite{antonov2023curing} finds approximations $\widehat{\delta_{\text{min}}}, \widehat{\delta_{\text{max}}}$ of maximal and minimal distance in the domain.
Then it uses the found distances to compute the parameter $\gamma$ needed to define the mapping for any $\bm{x}, \bm{y} \in \mathbb{L}^M$: $$\tau(\bm{x}, \bm{y}) = 1 - \exp\!\br{-\gamma^2\cdot d^2(\bm{x}, \bm{y})}.$$
The inverse of this transformation that we use to map the value from the interval $(0,1)$ back to the original space $\mathbb{L}^M$:
$$d(\bm{x}, \bm{y}) = \sqrt{- \ln \! \br{1 - \tau(\bm{x}, \bm{y})}}/\gamma.$$

The latter equation defines the transformation made in line~\ref{alg:dd-ea:transform-dist} of Algorithm~\ref{alg:dd-ea}.

We compute OFS-specific values for approximated minimal $\widehat{\delta_{\text{min}}}(d_{\mathbb{L}})$ and maximal $\widehat{\delta_{\text{max}}}(d_{\mathbb{L}})$ distances in $\mathbb{M}(\dL)$ for each metric on $\mathbb{L}: d_{\mathbb{L}} \in \set{d_1, d_2}$ and use them in Algorithm 2 from~\cite{antonov2023curing} to compute $\gamma$.

For the given $d_{\mathbb{L}}$ the smallest distance on $\mathbb{L}^M$ is: 
$$\delta_{\text{min}}(d_{\mathbb{L}}) = \min\limits_{\bm{x}, \bm{y} \in \mathbb{L}^M} \set{ \textsc{LAP}(d_{\mathbb{L}}, \bm{x}, \bm{y}) \mid \bm{x} \not\equiv \bm{y}}.$$
Let us consider the value: $$\delta^*(\dL) \coloneqq \min\limits_{x, y \in \mathbb{L}} \set{\dL(x, y) \mid x \ne y}.$$
If there exist points $\bm{x}, \bm{y} \in \mathbb{L}^M$ such that $\bm{x} \not\equiv \bm{y}$ and $\textsc{LAP}(\dL, \bm{x}, \bm{y}) < \delta^*,$ then for some permutation $\pi$ we have $\sum_{i=1}^M{\dL(x_i, y_{\pi(i)})} < \delta^*(\dL).$
However, it is impossible since metric $\dL$ returns only non-negative values and at least one positive.
Therefore, $\forall \bm{x}, \bm{y} \in \mathbb{L}^M: $ either $\textsc{LAP}(\dL, \bm{x}, \bm{y}) = 0$ or $\textsc{LAP}(\dL, \bm{x}, \bm{y}) \ge \delta^*(\dL).$
Equality is attained for example for such pair $\bm{x}, \bm{y} \in \mathbb{L}^M$ that $x_i = y_i$ for all $i \in \overline{1, M-1}$ and $\dL(x_M, y_M) = \delta^*(\dL).$
We have just demonstrated that $\delta_{\text{min}}(\dL) = \delta^*(\dL).$

When candidate solution $\bm{x} \in \mathbb{L}^M$ is fixed, we can analogically show that: 
\begin{equation*}
    \delta_{\text{min}}(\dL, \bm{x}) \coloneqq \min\limits_{\bm{y} \in \mathbb{L}^M} \set{\textsc{LAP}(\dL, \bm{x}, \bm{y}) \mid \bm{x} \not\equiv \bm{y}} = \min\limits_{ \substack{y \in \mathbb{L}, \\ i \in \overline{1,M} } }  \set{\dL(x_i, y) \mid x_i \ne y}.
\end{equation*}
Now we consider the value:
\begin{equation*}
\delta^{**}(\dL) \coloneqq \max\limits_{\bm{x} \in \mathbb{L}^M} \set{\delta_{\text{min}}(\dL, \bm{x})} 
= \max\limits_{x \in \mathbb{L}} \set { \min\limits_{y \in \mathbb{L}} \set{\dL(x, y) \mid x \neq y} }.
\end{equation*}

One of the assumptions on the search space in~\cite{antonov2023curing} is so called consistency, which includes the following condition.
For any two points $\bm{x}, \bm{y} \in \mathbb{L}^M$ the smallest distances are similar: $\delta_{\text{min}}(\dL, \bm{x}) \approx \delta_{\text{min}}(\dL, \bm{y}).$
If this assumption is valid for our application, then $\delta^{**}(\dL)/\delta^*(\dL) \approx 1.$
However, our computations show that $\delta^{**}(d_1)/\delta^*(d_1) \approx 2 \cdot 10^6$ and $\delta^{**}(d_2)/\delta^*(d_2) \approx 7 \cdot 10^8,$ which means that the consistency assumption is not upheld.
Therefore, we considered $R=10$ points $\BF{x}^{(i)} \in \mathbb{L}^M, i \in \overline{1,R}$ chosen uniformly at random.
They are used to compute the following value:
$$\widetilde{\delta_{\text{min}}}(\dL) \coloneqq \dfrac{1}{R} \sum\limits_{i=1}^R{\delta_{\text{min}}\!\br{\dL, \BF{x}^{(i)}}}. $$
We do not set $\widehat{\delta_{\text{min}}}(\dL)$ equal to $\widetilde{\delta_{\text{min}}}(\dL)$ due to the limitations of floating-point precision in computer.
The floating-point representation of numbers can not bijectively map every possible distance to a unique number in the interval $(0,1).$
Therefore, we multiply $\widetilde{\delta_{\text{min}}}(\dL)$ by a sufficiently big constant factor: 
$$\widehat{\delta_{\text{min}}}(\dL) = 100 \cdot \widetilde{\delta_{\text{min}}}(\dL).$$

The biggest distance $\delta_{\text{max}}(d_{\mathbb{L}})$ on $\mathbb{L}^M$ is defined by analogy to the smallest:
$\delta_{\text{max}}(d_{\mathbb{L}}) = \max \set{ \textsc{LAP}(d_{\mathbb{L}}, \bm{x}, \bm{y}) \mid \bm{x}, \bm{y} \in \mathbb{L}^M}.$
We did not observe the same properties for the maximal distance as for the minimal.
Therefore, we applied the general procedure proposed in~\cite{antonov2023curing} to compute $\widetilde{\delta_{\text{max}}}\!\br{\dL, \BF{x}^{(i)}}, i \in \overline{1,R}$ for $R=10$ points chosen uniformly at random from $\mathbb{L}^M.$
The maximal of this values is used as $\widehat{\delta_{\text{max}}}(\dL)$: $$\widehat{\delta_{\text{max}}}(\dL) = \max\limits_{i \in \overline{1,M}} \set{\widetilde{\delta_{\text{max}}}\!\br{\dL, \BF{x}^{(i)}}}.$$

In line~\ref{alg:dd-ea:explore} of Algorithm~\ref{alg:dd-ea} we denoted the method that takes $R$ points from $\mathbb{L}^M$ and computes values $\widehat{\delta_{\text{min}}}(\dL), \widehat{\delta_{\text{max}}}(\dL), \gamma$ as $\textsc{Explore}.$
It returns only $\gamma$ since it is used later in the algorithm for the transformation of step size to distance in line~\ref{alg:dd-ea:transform-dist}.

\subsubsection{Solving Internal Optimization Problem.}
\label{sec:internal-opt}

Specification of Eq.~\eqref{eq:internal-opt} to OFS problem tells that we aim at locating a point $\bm{y}$ in the set: $$\argmin\limits_{\bm{z} \in \mathbb{L}^M}\set{|\textsc{LAP}(\dL, \bm{x}, \bm{z}) - S|},$$ for the given metric $\dL$, constant $S$ and a point $\bm{x}$.
Such point $\bm{y}$ constitutes the mutant of $\bm{x}$ that steps away from it to the distance $\widetilde{S}$ according to the metric $\textsc{LAP}(\dL, \cdot, \cdot).$
The value of $\widetilde{S}$ is closest possible to $S$.
It is perfect if $\widetilde{S} = S$, however, this might be not always possible, since the space $\mathbb{L}^M$ is not continuous.

We propose a heuristic approach that aims at locating $\bm{y}$ to minimize $|\textsc{LAP}(\dL, \bm{x}, \bm{y}) - S|,$ within the limited number of $\textsc{LAP}$ function evaluations.
In line~\ref{alg:dd-ea:dd-mutation} of Algorithm~\ref{alg:dd-ea} this method is called $\textsc{DDMutation},$ \utwo{which stands for \emph{distance-driven mutation}.}
This non-trivial internal optimization task can be seen as an inverse linear assignment problem because the goal is to produce a vector that gives a specific solution to this problem.
Algorithm~\ref{alg:inv-lap} summarizes the proposed method to solve this problem\utwo{, called Distance‑Driven Assignment Evolutionary Algorithm (DDA‑EA).} 

The number of $\textsc{LAP}$ function evaluations is restricted since it might be a computationally expensive procedure, especially for bigger values of $M$.
In contrast to $\textsc{LAP}$, which returns only one real value, we use notation $\textsc{LAP}^*(\dL, \bm{x}, \bm{y})$ to emphasize that it returns two objects:
1) the permutation $\pi^*$ on which the minimum of $\sum_{i=1}^M{\dL(x_i, y_{\pi^*(i)})}$ is reached, and 2) the inherent smallest value, achievable through the permutation $\pi^*.$
Notably, the second number of $\textsc{LAP}^*$ equals to the one returned by $\textsc{LAP}.$

The general scheme of Algorithm~\ref{alg:inv-lap} is the same as in $(1 + \lambda)$ EA.
It treats $|\textsc{LAP}(\dL, \bm{x}, \bm{y}) - s|$ as the minimized objective function in the search space $\mathbb{L}^M$ and iteratively produces $\lambda$ candidate solutions by making variations of the best-so-far solution.
However, this algorithm takes advantage of the permutation $\pi$ available through $\textsc{LAP}^*$.
The maintained bijection $\pi^{(i)}$ represents the assignment of components in the initial solution $\bm{x}^{(0)}$ and the current best-so-far $\BF{x}^{(i)}.$
In the variation phase Algorithm~\ref{alg:inv-lap} in line~\ref{alg:inv-lap:k} selects a component with number $\rm{k}$ in $\bm{x}^{(0)}$.

\SetKw{Break}{break}
\scalebox{1}{
\begin{algorithm2e}[H]
    \caption{\utwo{Distance‑Driven Assignment Evolutionary Algorithm (DDA‑EA) -- }the proposed method to locate a vector $\bm{y} \in \mathbb{L}^M$ such that $\textsc{LAP}(\dL, \bm{y}, \bm{x}^{(0)}) \approx s$, given $M\in \N$, $\bm{x}^{(0)} \in \mathbb{L}^M$, constant $s$ and metric $\dL$ on the set $\mathbb{L} \coloneqq \overline{1, L}$.  The maximal number of $\textsc{LAP}$ function evaluations is limited by the constant $b$. On top of that, the algorithm uses constants $\lambda, R \in \N$. \utwo{The method is used to solve the internal optimization problem for distance-driven algorithm applied to OFS.}}
    \label{alg:inv-lap}
    Initialization of parent: $\BF{x}^{(1)}, \rm{v}^{(1)} \gets \bm{x}^{(0)}, 0$\;
    Init. permut.: $\pi^{(1)} \gets \br{\forall i \in \overline{1, M}: i \mapsto i}$\;
    \For{$i \gets 1, 2, \ldots, \lfloor b/\lambda \rfloor$}{
        $\BF{x}^{(i + 1)}, \pi^{(i + 1)}, \rm{v}^{(i + 1)} \gets \BF{x}^{(i)}, \pi^{(i)}, \rm{v}^{(i)}$\;
        \For{$j \gets 1, 2, \ldots, \lambda$}{
            \nlnonumber \DontPrintSemicolon \textbf{Variation Phase:}\; \PrintSemicolon
            $\BF{y} \gets \BF{x}^{(i)}$\;
            \For{$r \gets 1, 2, \ldots, R$}{
                $\rm{k} \gets \uar{\overline{1,M}}$\label{alg:inv-lap:k}\;
                Set $\br{f_t}_{t=1}^L$ to reordered ints. in $\mathbb{L}$: $\dL(x^{(0)}_\rm{k}, f_1) < \dL(x^{(0)}_\rm{k}, f_2) < \ldots < \dL(x^{(0)}_\rm{k}, f_L)$\label{alg:inv-lap:reorder}\;
                Find position $\rm{p}$ such that: $f_\rm{p} = \rm{x}^{(i)}_{\pi^{(i)}(\rm{k})}$\label{alg:inv-lap:p}\;
                \If{$\rm{p} \ne 1$ \textbf{\textsc{and}} $\rm{v}^{(i)} > s$\label{alg:inv-lap:bigger}}{
                    $\rm{y}_{\pi^{(i)}(\rm{k})} \gets $ Harmonic Sample from $(f_{\rm{p}-1}, f_{\rm{p}-2}, \dots, f_1)$\label{alg:inv-lap:bigger-sample}\;
                }
                \If{$\rm{p} \ne L$ \textbf{\textsc{and}} $\rm{v}^{(i)} < s$\label{alg:inv-lap:smaller}}{
                    $\rm{y}_{\pi^{(i)}(\rm{k})} \gets $ Harmonic Sample from $(f_{\rm{p}+1}, f_{\rm{p}+2}, \dots, f_L)$\label{alg:inv-lap:smaller-sample}\;
                }
                \If{$\BF{y} \ne \BF{x}^{(i)}$\label{alg:inv-lap:changed}}{
                    \Break\;
                }
            }
            \nlnonumber \DontPrintSemicolon \textbf{Selection Phase:}\; \PrintSemicolon
            $\pi_{\BF{y}}, \rm{v}_{\BF{y}} \gets \textsc{LAP}^*(\dL, \bm{x}^{(0)}, \BF{y}) $\;
            \If{$\rm{v}_{\BF{y}} = s$}{
                \Return $\BF{y}$\;
            }
            \If{$|\rm{v}_{\BF{y}} - s| < |\rm{v}^{(i)} - s| $}{
                $\BF{x}^{(i + 1)}, \rm{v}^{(i + 1)} \gets \BF{y}, \rm{v}_{\BF{y}}$\;
                $\pi^{(i + 1)} \gets \pi_{\BF{y}}$\;
            }
        }
    }
    \Return $\BF{x}^{\br{\lfloor b/\lambda \rfloor}}$\;
\end{algorithm2e}
}

\uthree{Figure~\ref{fig:permutation} shows the initial (reference) solution, $\bm x^{(0)}$, and the mutant solutions, $\mathbf x^{(i)}$, at a curtain stage of Algorithm~\ref{alg:inv-lap}. At each stage, the algorithm samples a position $k$, and finds a position in the mutant solution that corresponds to position $k$ in the initial solution, under the current permutation $\pi^{(i)}$. Thus, the filter $x ^{(0)}_k$ at position $k $ corresponds to a component in $\mathbf{x}^{(i)} $ through the permutation  $ \pi^{( i )}$. The fact that the position of the mutant solution may not be the same as the position in the reference solution is indicated by the curved red line. }

\begin{figure}[!tb]
\centering
\resizebox{0.55\linewidth}{!}{%
\begin{circuitikz}
\tikzstyle{every node}=[font=\normalsize]
\draw  (5,14.5) rectangle (6.25,13.25);
\node [font=\normalsize] at (6,11.75) {};
\node [font=\normalsize] at (4.5,11.5) {};
\draw  (6.25,14.5) rectangle (7.5,13.25);
\draw  (11.25,14.5) rectangle (12.5,13.25);
\node [font=\normalsize] at (4,13.75) {$\bm{x}^{(0)}:$};
\node [font=\normalsize] at (8.25,14) {};
\node [font=\normalsize] at (8.25,14) {};
\node [font=\normalsize] at (8.75,13.75) {};
\node [font=\normalsize] at (8.75,13.75) {};
\draw [short] (7.5,14.5) -- (7.75,14.5);
\draw [short] (7.5,13.25) -- (7.75,13.25);
\draw [short] (8.5,14.5) -- (8.75,14.5);
\draw [short] (8.5,13.25) -- (8.75,13.25);
\node [font=\normalsize] at (8,13.75) {$\ldots$};
\node [font=\normalsize] at (5.5,14.75) {1};
\node [font=\normalsize] at (6.75,14.75) {$2$};
\node [font=\normalsize] at (9.25,14.75) {$\rm{k}$};
\node [font=\normalsize] at (8.25,14) {};
\draw [short] (10,14.5) -- (10.25,14.5);
\draw [short] (11,14.5) -- (11.25,14.5);
\draw [short] (10,13.25) -- (10.25,13.25);
\draw [short] (11,13.25) -- (11.25,13.25);
\node [font=\normalsize] at (10.5,13.75) {$\ldots$};
\node [font=\normalsize] at (11.75,14.75) {$M$};
\node [font=\normalsize] at (7,13.75) {${x}^{(0)}_2$};
\node [font=\normalsize] at (5.75,13.75) {${x}^{(0)}_1$};
\node [font=\normalsize] at (9.5,13.75) {${x}^{(0)}_\rm{k}$};
\node [font=\normalsize] at (12,13.75) {${x}^{(0)}_M$};
\draw  (5,12) rectangle (6.25,10.75);
\node [font=\normalsize] at (6.5,11) {};
\node [font=\normalsize] at (5.25,10.75) {};
\draw  (6.25,12) rectangle (7.5,10.75);
\draw  (11.25,12) rectangle (12.5,10.75);
\node [font=\normalsize] at (4,11.25) {$\BF{x}^{(i)}:$};
\node [font=\normalsize] at (8.25,11.5) {};
\node [font=\normalsize] at (8.25,11.5) {};
\node [font=\normalsize] at (8.75,11.25) {};
\node [font=\normalsize] at (8.75,11.25) {};
\draw [short] (7.5,12) -- (7.75,12);
\draw [short] (7.5,10.75) -- (7.75,10.75);
\draw [short] (8.5,12) -- (8.75,12);
\draw [short] (8.5,10.75) -- (8.75,10.75);
\node [font=\normalsize] at (8,11.25) {$\ldots$};
\node [font=\normalsize] at (5.5,10.5) {1};
\node [font=\normalsize] at (6.75,10.5) {$2$};
\node [font=\normalsize] at (9.35,10.45) {$\pi^{(i)}(\rm{k})$};
\node [font=\normalsize] at (8.25,11.5) {};
\draw [short] (10,12) -- (10.25,12);
\draw [short] (11,12) -- (11.25,12);
\draw [short] (10,10.75) -- (10.25,10.75);
\draw [short] (11,10.75) -- (11.25,10.75);
\node [font=\normalsize] at (10.5,11.25) {$\ldots$};
\node [font=\normalsize] at (11.75,10.5) {$M$};
\node [font=\normalsize] at (7,11.25) {$\rm{x}^{(i)}_2$};
\node [font=\normalsize] at (5.75,11.25) {$\rm{x}^{(i)}_1$};
\node [font=\normalsize] at (9.35,11.25) {$\rm{x}^{(i)}_{\pi^{(i)}(\rm{k})}$};
\node [font=\normalsize] at (12,11.25) {$\mathrm{x}^{(i)}_M$};
\draw [ color={rgb,255:red,240; green,0; blue,0}, line width=1pt, ->, >=Stealth] (9.25,13.25) .. controls (7.75,12.5) and (11,13) .. (9.25,12) ;
\node [font=\normalsize] at (10.75,12.5) {$\pi^{(i)}(\rm{k})$};
\draw [ color={rgb,255:red,240; green,0; blue,0} ] (8.75,14.5) rectangle (10,13.25);
\draw [ color={rgb,255:red,240; green,0; blue,0} ] (8.75,12) rectangle (10,10.75);
\end{circuitikz}
 }%
\caption{Example of selected position $k$ in the initial solution $\bm{x}^{(0)}$ and the corresponding position in individual $\BF{x}^{(i)}$ according to the permutation $\pi^{(i)}$.}
\label{fig:permutation}
\end{figure}

The key idea of our heuristic is to increase the distance $\dL(x^{(0)}_{\rm{k}}, \rm{x}^{(i)}_{\pi^{(i)}(\rm{k})})$ between those filters, when $\textsc{LAP}(\dL, \bm{x}^{(0)}, \BF{x}^{(i)}) < s$ and reduce otherwise.
This is implemented in lines~\ref{alg:inv-lap:reorder},~\ref{alg:inv-lap:p},~\ref{alg:inv-lap:bigger},~\ref{alg:inv-lap:bigger-sample},~\ref{alg:inv-lap:smaller} and~\ref{alg:inv-lap:smaller-sample}, that generate candidate solution $\BF{y}$ with the described change in one of its components.
The implementation details from the pseudocode are visualized in Figure~\ref{fig:sorted}.
For the sake of shortness in our pseudocode Algorithm~\ref{alg:dd-ea} and Figure~\ref{fig:sorted} we represent $\textsc{LAP}(\dL, \bm{x}^{(0)}, \BF{x}^{(i)})$ by $\rm{v}^{(i)}$, and $\textsc{LAP}(\dL,\bm{x}^{(0)}, \BF{y})$ by $\rm{v}_{\BF{y}}$.

\uthree{In Figure~\ref{fig:sorted}, we demonstrate the selection procedure of the new filter in both cases of how the total distance $\textsc{LAP}$ compares with $s$. When this distance is bigger than $s$, the filters to the left are considered. Otherwise, filters to the right of $f_p$ are considered. The selection of a filter from the considered ones is done by chance according to the Harmonic distribution~\cite{doerr2018static}. This distribution selects filters closer to the initial one with higher probability than selecting those that are further.}   
It is not guaranteed that $\rm{v}_{\BF{y}}$  is smaller or bigger than $\rm{v}^{(i)}$ because the assignment $\pi_{\BF{y}}$ might be not equal to $\pi^{(i)}$.
Therefore, we randomize this heuristic and use it as a variation operator in the inherent structure of EA.

\begin{figure}[!tb]
\centering
\resizebox{0.6\linewidth}{!}{%
\begin{circuitikz}
\tikzstyle{every node}=[font=\normalsize]
\draw  (3.75,13.25) rectangle (5,12);
\draw  (5,13.25) rectangle (6.25,12);
\draw  (7.5,13.25) rectangle (8.75,12);
\draw  (10,13.25) rectangle (11.25,12);
\draw  (12.5,13.25) rectangle (13.75,12);
\draw  (13.75,13.25) rectangle (15,12);
\draw [short] (6.25,13.25) -- (6.5,13.25);
\node [font=\normalsize] at (7,12.75) {};
\draw [short] (6.25,12) -- (6.5,12);
\node [font=\normalsize] at (7.5,12.25) {};
\node [font=\normalsize] at (7,12.75) {};
\node [font=\normalsize] at (7.5,12) {};
\draw [short] (7.25,13.25) -- (7.5,13.25);
\node [font=\normalsize] at (7.5,13.25) {};
\node [font=\normalsize] at (6.75,10.75) {};
\draw [short] (7.25,12) -- (7.5,12);
\node [font=\normalsize] at (7.25,12) {};
\node [font=\normalsize] at (11,11.5) {};
\draw [short] (11.25,13.25) -- (11.5,13.25);
\node [font=\normalsize] at (10.5,12.75) {};
\node [font=\normalsize] at (12,10.75) {};
\draw [short] (11.25,12) -- (11.5,12);
\node [font=\normalsize] at (11.5,12) {};
\node [font=\normalsize] at (12.25,11.5) {};
\draw [short] (12.25,13.25) -- (12.5,13.25);
\node [font=\normalsize] at (11.75,12.75) {};
\node [font=\normalsize] at (12.5,10.75) {};
\draw [short] (12.25,12) -- (12.5,12);
\node [font=\normalsize] at (12,12) {};
\node [font=\normalsize] at (4.25,12.5) {$f_1$};
\node [font=\normalsize] at (5.5,12.5) {$f_2$};
\node [font=\normalsize] at (8.1,12.5) {$f_{\rm{p}-1}$};
\node [font=\normalsize, color={rgb,255:red,240; green,0; blue,0}] at (9.25,12.5) {$f_{\rm{p}}$};
\node [font=\normalsize] at (10.6,12.5) {$f_{\rm{p}+1}$};
\node [font=\normalsize] at (13.1,12.5) {$f_{L-1}$};
\node [font=\normalsize] at (14.25,12.5) {$f_L$};
\node [font=\normalsize] at (7,12.5) {\dots};
\node [font=\normalsize] at (12,12.5) {\dots};
\draw [->, >=Stealth] (8.25,11.5) -- (3.75,11.5);
\draw [->, >=Stealth] (8.25,11.5) -- (8.25,9.5);
\node [font=\normalsize, rotate around={270:(0,0)}] at (8.5,10.5) {Probability};
\node [font=\normalsize] at (13.5,10.75) {};
\draw [->, >=Stealth] (10.5,11.5) -- (15,11.5);
\draw [->, >=Stealth] (10.5,11.5) -- (10.5,9.5);
\node [font=\normalsize, rotate around={90:(0,0)}] at (10.25,10.5) {Probability};
\draw [ color={rgb,255:red,240; green,0; blue,0}, line width=1pt, dashed] (10.75,9.5) .. controls (11,11.25) and (12.5,11.25) .. (15,11.25);
\node [font=\normalsize, color={rgb,255:red,240; green,0; blue,0}, rotate around={-360:(0,0)}] at (9.5,14.25) {$\rm{x}^{(i)}_{\pi^{(i)}(\rm{k})}$};
\draw [ color={rgb,255:red,240; green,0; blue,0}, ->, >=Stealth] (9.25,13.75) -- (9.25,13);
\draw [ color={rgb,255:red,240; green,0; blue,0}, line width=1pt, dashed] (8,9.5) .. controls (7.75,11.25) and (6.25,11.25) .. (3.75,11.25);
\node [font=\normalsize, rotate around={-360:(0,0)}] at (3.25,12.5) {$\mathbb{L}:$};
\draw [ color={rgb,255:red,0; green,0; blue,240} ] (7.9,10) circle (0.25cm);
\draw [ color={rgb,255:red,0; green,0; blue,240}, ->, >=Stealth] (7.9,10.25) -- (7.9,12);
\draw [ color={rgb,255:red,0; green,0; blue,240}, ->, >=Stealth] (5.5,11.5) -- (5.5,12);
\draw [ color={rgb,255:red,0; green,0; blue,240} ] (5.5,11.25) circle (0.25cm);
\draw [ color={rgb,255:red,0; green,0; blue,240}, ->, >=Stealth] (4.25,11.5) -- (4.25,12);
\draw [ color={rgb,255:red,0; green,0; blue,240} ] (4.25,11.25) circle (0.25cm);
\draw [ color={rgb,255:red,0; green,0; blue,240}, ->, >=Stealth] (10.85,10.25) -- (10.85,12);
\draw [ color={rgb,255:red,0; green,0; blue,240} ] (10.85,10) circle (0.25cm);
\draw [ color={rgb,255:red,0; green,0; blue,240}, ->, >=Stealth] (13,11.5) -- (13,12);
\draw [ color={rgb,255:red,0; green,0; blue,240} ] (13,11.25) circle (0.25cm);
\draw [ color={rgb,255:red,0; green,0; blue,240}, ->, >=Stealth] (14.25,11.5) -- (14.25,12);
\draw [ color={rgb,255:red,0; green,0; blue,240} ] (14.25,11.25) circle (0.25cm);
\draw [ color={rgb,255:red,0; green,0; blue,240}, line width=2pt, ->, >=Stealth] (9,13.5) -- (3.75,13.5);
\draw [ color={rgb,255:red,0; green,0; blue,240}, line width=2pt, ->, >=Stealth] (9.75,13.5) -- (15,13.5);
\node [font=\normalsize, color={rgb,255:red,0; green,0; blue,240}, rotate around={-360:(0,0)}] at (5.5,14) {$s < \rm{v}^{(i)}$};
\node [font=\normalsize, color={rgb,255:red,0; green,0; blue,240}, rotate around={-360:(0,0)}] at (13.25,14) {$s > \rm{v}^{(i)}$};
\draw [ color={rgb,255:red,240; green,0; blue,0} ] (8.75,13.25) rectangle (10,12);
\end{circuitikz}
 }%
\caption{Example of ordered filters from the library $\mathbb{L}$ in ascending order of the value $\dL(x_k^{(0)}, f_i)$ for each filter $f_i$. The filter  $f_p$ is colored in red because it is changed in $\BF{x}^{(i)}$ to create the mutant $\BF{y}$. If the current distance $\rm{v}^{(i)}$ between $\bm{x}^{(0)}$ and $\BF{x}^{(i)}$ is bigger than the target distance $s$, then filters to the left of $f_p$ are considered to substitute $f_p$. Otherwise, filters to the right of $f_p$ are considered. The selection of a filter from the considered ones is done by chance according to Harmonic distribution. It is visualized in the lower part of the picture by plots, that depict the probability of selecting each filter.}
\label{fig:sorted}
\end{figure}

The choice of the new value for the considered component is done through Discrete Harmonic Distribution in lines~\ref{alg:inv-lap:bigger-sample} and~\ref{alg:inv-lap:smaller-sample} of Algorithm~\ref{alg:inv-lap}.
It is defined and applied for mutations, for example, in~\cite{doerr2018static}, where it is shown that sampling from it gives the best possible number of iterations to locate the optimum of the problem considered in that work. 
Our choice of this distribution is motivated by \utwo{this} rigorous result.
The probabilities assigned to integers via such harmonic sampling are visualized in lower part of Figure~\ref{fig:sorted}.
Notably, the described variation may not change the candidate solution, therefore, we repeat the procedure $R$ times unless $\BF{y}$ becomes different from $\BF{x}^{(i)}$ earlier.

When the Algorithm~\ref{alg:inv-lap} is applied via function $\textsc{DDMutation}$ it takes three arguments: candidate solution $\bm{x}$, metric $\dL$ and constant $\rm{S}$.
To define the parameters of the algorithm we set $\bm{x}^{(0)}$ to the passed $\bm{x}$, constant $s = \rm{S}$, $b = 10^3$, $\lambda = 5$ and $R = 10$.
The rest attributes of Algorithm~\ref{alg:inv-lap} are set to the ones with the same name from the context of OFS problem.
We used the implementation of $\lap$ and $\lap^*$ from the library skopt~\cite{skoptLAP}.

\subsubsection{Distribution $\mathcal{D}$ Over Step-Sizes.}
\label{sec:step-size-distr}

The distribution $\mathcal{D}$ has maximal entropy on the interval $(0, 1)$ over all distribution with the fixed mean $m$.
We used constant factor $m = 0.1$ and applied methodology from~\cite{antonov2023curing} to derive this distribution.
Visualization of $\mathcal{D}$ is proposed in Figure~\ref{fig:step-sizes-distr}.

\begin{figure}[!ht]
    \centering
    \includegraphics[width=0.5\textwidth]{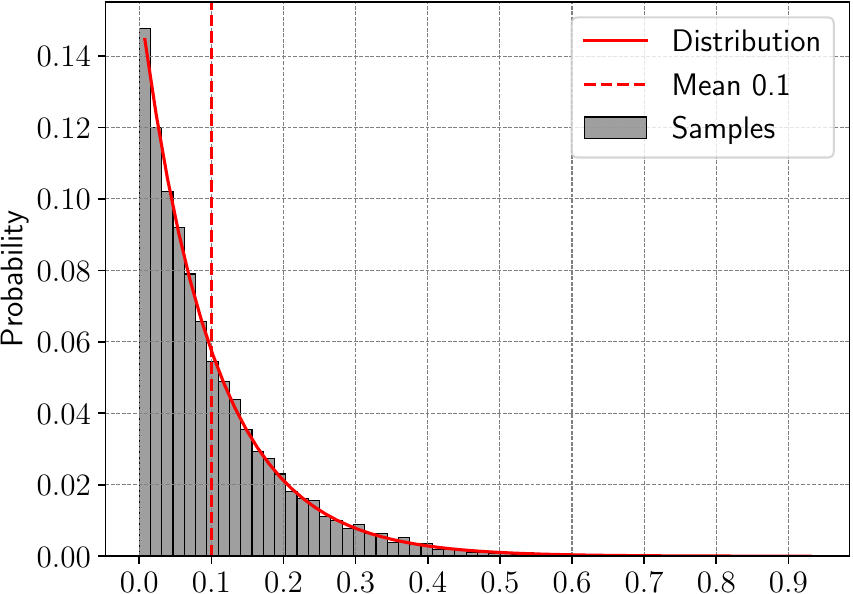}
    \caption{Distribution $\mathcal{D}$ over the step sizes in mutation operator.}
    \label{fig:step-sizes-distr}
\end{figure}

In order to demonstrate that methods described in Section~\ref{sec:gamma}, Section~\ref{sec:internal-opt} and Section~\ref{sec:step-size-distr} can efficiently work together, we designed the following computational experiment.
We take 100 evenly spaced values $\set{s^{(i)}}_{i=1}^{100}$ from $(10^{-4}, 1-10^{-5}) \subset (0, 1)$.
Then we map each of them to the distance $\set{S^{(i)}}_{i=1}^{100}$, as described in Section~\ref{sec:internal-opt}.
For each $S^{(i)}$ we choose 10 points $\set{\BF{x}^{(i, j)}}_{j=1}^{10}$ uniformly at random from $\mathbb{L}^M$.
For every $\BF{x}^{(i,j)}$ we generated mutant: $$\BF{y}^{(i,j)} \coloneqq \textsc{DDMutation}\!\br{\dL, \BF{x}^{(i,j)}, S^{(i)}},$$ and computed the values: 
$$\mathbb{V}^{(i)} \coloneqq \set{  \dfrac{1}{S^{(i)}} \cdot \mid \lap\!\br{\dL,\BF{x}^{(i,j)},\BF{y}^{(i,j)}} - S^{(i)}\mid}_{j=1}^{10}.$$
Numbers in this family of sets represent the relative deviation from the target distance and the distance between the mutant and the original point.
The smaller those values are, the smaller the \utwo{gas retrieval} error is, and so the more precise our mutation operator is.
We observed that for each $i \in \overline{1, 100}$ the mean $m^{(i)}$ of $\mathbb{V}^{(i)}$ was smaller than $0.05$ for both $d_1$ and $d_2$ taken as $\dL$.
The median of $\set{m^{(i)}}_{i=1}^{100}$ is $\approx 4 \cdot 10^{-5}$ for $d_1$ and $\approx 10^{-4}$ for $d_2$.

Based on this experiment, we conclude that the proposed Algorithm~\ref{alg:inv-lap} works very precisely in the context of OFS, for the considered constants and distance metrics $d_1, d_2$.

\subsection{Comparison of the Defined Metrics}
\label{sec:metrics-cmp}

\begin{figure*}[!tb]
    \begin{tabular}{p{0.3\textwidth}  p{0.3\textwidth}  p{0.3\textwidth}}
         & \adjustbox{valign=t, center}{ \includegraphics[width=0.5\textwidth, trim=0mm 00mm 0mm 0mm,clip]{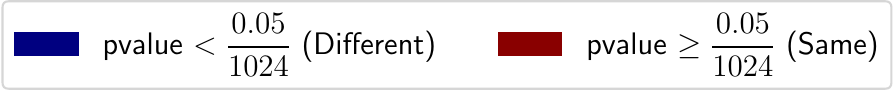}} & \\
        \adjustbox{valign=t, left}{\includegraphics[width=0.4\textwidth, trim=28mm 0mm 0mm 10mm,clip]{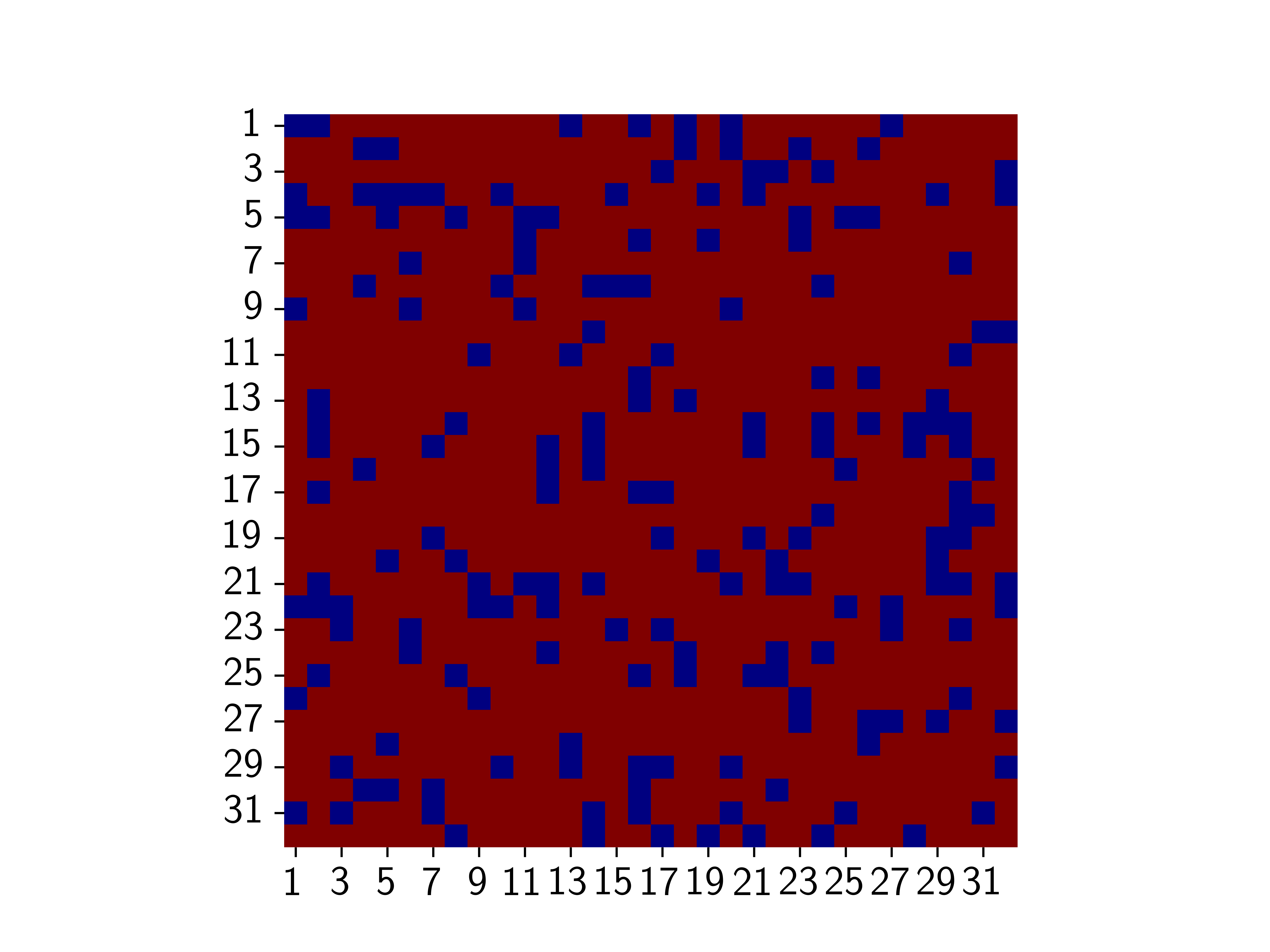}} &
        \adjustbox{valign=t, left}{\includegraphics[width=0.4\textwidth, trim=28mm 0mm 0mm 10mm,clip]{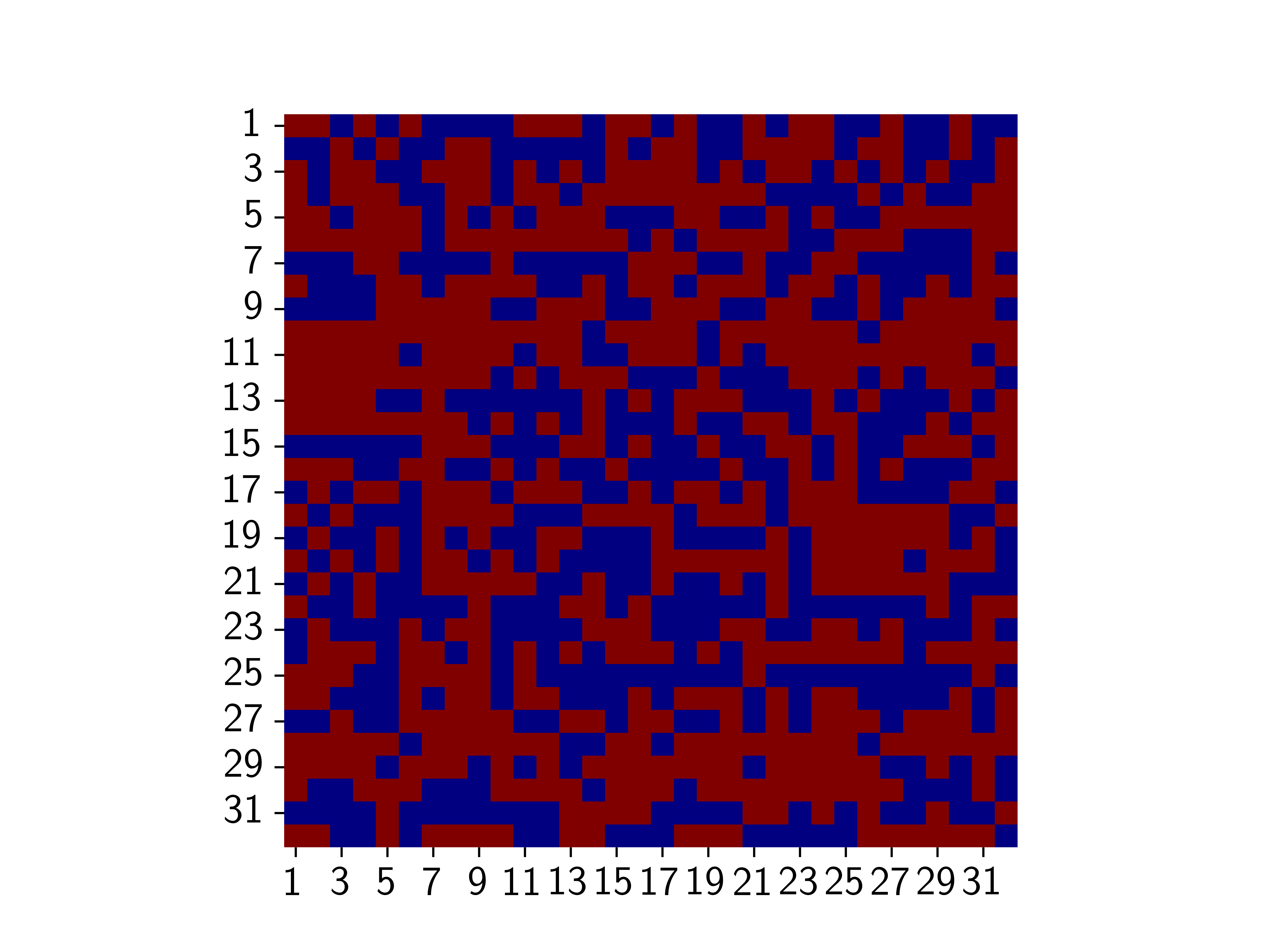}} &
        \adjustbox{valign=t, left}{\includegraphics[width=0.4\textwidth, trim=28mm 0mm 0mm 10mm,clip]{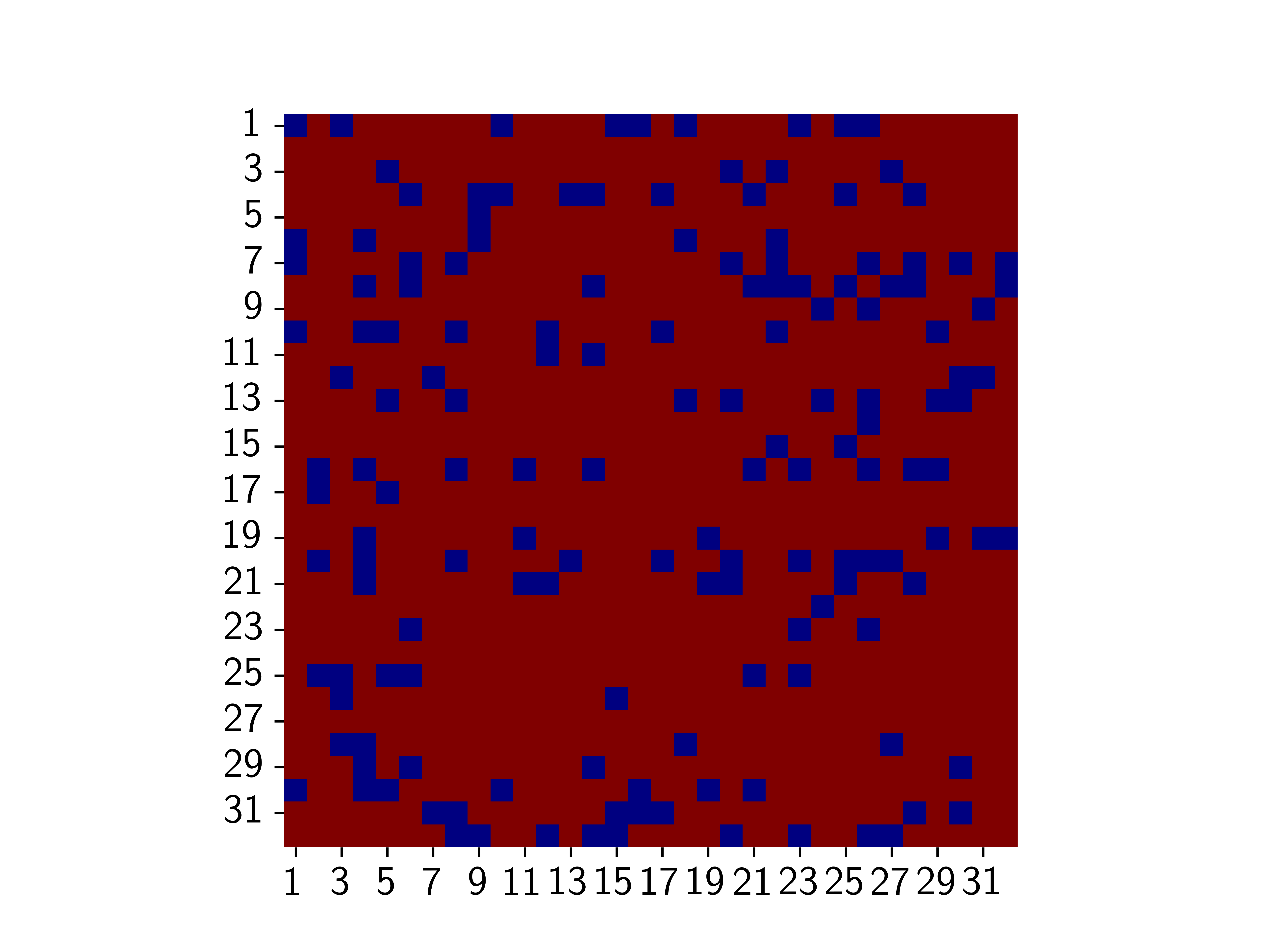}} \\
        \adjustbox{valign=t}{ \scalebox{0.75}{\begin{minipage}{0.4\textwidth}
        \begin{itemize} 
        \item[\textbf{(a)}] \textbf{Hamming distance: \\ $\mathscr{H}(\BF{x}^{(i)}, \BF{y}^{(i,j)}) = 1,$ rejected $18\%$} 
        \end{itemize}
        \end{minipage} } } &
        \adjustbox{valign=t}{ \scalebox{0.75}{
        \begin{minipage}{0.4\textwidth}
        \begin{itemize} 
        \item[\textbf{(b)}] \textbf{$\lap\!\br{d_1, \BF{x}^{(i)}, \BF{y}^{(i,j)}} \in (10 \cdot \delta_1, 15 \cdot \delta_1) ,$ rejected $44\%$} 
        \end{itemize}
        \end{minipage} } } &
        \adjustbox{valign=t}{ \scalebox{0.75}{ \begin{minipage}{0.4\textwidth}
        \begin{itemize} 
        \item[\textbf{(c)}] \textbf{$\lap\!\br{d_2, \BF{x}^{(i)}, \BF{y}^{(i,j)}} \in (10 \cdot \delta_2, 15 \cdot \delta_2),$ rejected $15\%$} 
        \end{itemize}        
        \end{minipage}}}
    \end{tabular}
    \caption{Flatness of neighborhoods in the three cases of the metric space: (a) $(\mathbb{L}^M, \mathscr{H})$, (b) $\br{\mathbb{M}(d_1), \lap(d_1, \cdot, \cdot)}$ and (c) $\br{\mathbb{M}(d_2), \lap(d_2, \cdot, \cdot)}$. }
    \label{fig:n}
\end{figure*}

\utwo{Having defined several distance metrics to quantify differences between solutions, we now turn to comparing the local structure of the search space induced by each metric. 
We begin with the Hamming distance, which underlies the conventional mutation operator in the $(\mu + \lambda)$ EA and the $(\mu/2 + \lambda)$ EA.}
This metric is defined for any $\bm{x}, \bm{y} \in \mathbb{L}^M$ as: 
$$\mathscr{H}(\bm{x}, \bm{y}) = \sum\limits_{i=1}^M \mathbbm{1}\!\cl{x_i \ne y_i}.$$

Since it is not known which distance is better for OFS, it is relevant to make their comparison.
In this section, we propose the first step towards a comprehensive understanding of difference between metrics.
Specifically, we analyze how many solutions with statistically different values of $\Fest_K$ are located in neighborhoods for the three cases of inherent metric: 1) Hamming distance $\mathscr{H}$, 2) $\lap(d_1, \cdot, \cdot)$, 3) $\lap(d_2, \cdot, \cdot)$.
We designed the experiment outlined in Algorithm~\ref{alg:n} to perform this analysis.

\SetKwProg{Cases}{Case}{ is}{}
\begin{algorithm2e}[!tb]
    \caption{Proposed measurement of the number of statistically different solutions in neighborhoods of OFS search space}
    \label{alg:n}
    \KwData{$d \gets $ one of $\set{\mathscr{H}, d_1, d_2}$;}
    \For{$i \gets 1, 2, \ldots, n$}{
        $\BF{x}^{(i)} \gets \uar{\mathbb{L}^M}$\label{alg:n:x}\;
        $\BF{w}^{(i)} \gets \br{\mathrm{D}^2\!\cl{\BF{x}^{(i)}, \xi^{(p)}}}_{p=1}^{K}$\label{alg:n:Dx}\;
        \For{$j \gets 1, 2, \ldots, n$}{
            \If{$d = \mathscr{H}$}{
                $\BF{y}^{(i,j)} \gets \BF{x}^{(i)}$\;
                $k \gets 1 + ((j-1) \mod M)$\label{alg:n:k}\;
                $\rm{y}^{(i, j)}_{k} \gets \uar{\overline{1, L} \setminus \set{\rm{x}^{(i)}_{k}}}$\label{alg:n:y-hamming}\;
            }
            \If{$d \in \set{d_1, d_2} $ }{
                $\delta \gets \widehat{\delta_{\text{min}}}(d)$\label{alg:n:delta}\;
                $\rm{s} \gets \uar{10 \cdot \delta, 15 \cdot \delta }$\label{alg:n:step-size}\;
                $\BF{y}^{(i,j)} \gets \textsc{DDMutation}\br{\BF{x}^{(i)}, d, \rm{s}}$\label{alg:n:y-lap}\;
            }
            $\BF{z}^{(i, j)} \gets \br{\mathrm{D}^2\!\cl{\BF{y}^{(i, j)}, \xi^{(p)}}}_{p=1}^{K}$\label{alg:n:Dz}\;
            $\rm{P}_{i, j} \gets \textsc{Welch}\!\br{\BF{w}^{(i)}, \BF{z}^{(i, j)}}$\label{alg:n:welch}\;
        }
    }
    \Return $\set{\set{\BF{x}^{(i)}}_{i=1}^{n}, \set{\BF{y}^{(i, j)}}_{i,j=1}^{n}, \BF{P}}$\label{alg:n:ret}\;
\end{algorithm2e}

In Algorithm~\ref{alg:n} the following procedure is repeated $n$ times.
In line~\ref{alg:n:x} of Algorithm~\ref{alg:n} a point $\BF{x}^{(i)}$ is chosen uniformly at random from the search space $\mathbb{L}^M$.
The random variable $\Sdiv^2$ is sampled $K$ times for the point $\BF{x}^{(i)}$ in line~\ref{alg:n:Dx}.
Such samples are used to find value $\Fest_K$ as defined in Eq.~\eqref{eq:approx-objf}.

For each $i$ the following generation of mutants is repeated $n$ times.
If the mutation is made according to the Hamming Distance $\mathcal{H}$, then the algorithm sets $k$ to the number of the changed component in line~\ref{alg:n:k}.
Notably, it is cyclic with period $M$, which emphasizes that 1) each component is chosen and 2) $n$ can be greater than $M$.
The filter at the selected component $k$ is changed in line~\ref{alg:n:y-hamming} to a different filter taken uniformly at random from $\mathbb{L}$.
The hamming distance $1$ is the smallest mutation that a conventional mutation operator can produce.

If the considered inherent metric on $\mathbb{L}$ is either $d_1$ or $d_2$, then in line~\ref{alg:n:delta} Algorithm~\ref{alg:n} proceeds as follows. 
It sets the value $\delta$ to the approximated value of the smallest distance on $\mathbb{M}(d)$, defined in Section~\ref{sec:gamma}.
Then it chooses a step size uniformly at random from $10\cdot \delta$ to $15\cdot \delta$ in line~\ref{alg:n:step-size}.
The probability of obtaining such a step size from $\mathcal{D}$ that maps to the distance smaller than $10\cdot \delta$ is $\approx 0.025$ for both $d_1, d_2$.
This is a relatively small chance.
Therefore, the factors for the distance in line~\ref{alg:n:step-size} constitute small magnitudes that are possible to get in one mutation.
In line~\ref{alg:n:y-lap} the mutant is created using the function $\textsc{DDMutation}$ defined in Section~\ref{sec:internal-opt}.

In all the cases of the distance, the Algorithm~\ref{alg:n} in line~\ref{alg:n:Dz} makes $K$ samples of the random variable $\Sdiv^2$ on the mutant.
Then the samples of the parent and mutant are used for statistical comparison in line~\ref{alg:n:welch}.
The Weltch t-test is applied as discussed in Section~\ref{sec:objf:noise}.
The pvalue of this test is used to define one element of matrix $\BF{P}$.

Eventually, Algorithm~\ref{alg:n} returns all parents, corresponding mutants, and the matrix $\BF{P}$.
This matrix is visualized in Figure~\ref{fig:n} for the three cases of inherent distance: $\mathscr{H}, \lap(d_1, \cdot, \cdot), \lap(d_2, \cdot, \cdot)$.
When one statistical comparison is made, it is natural to reject the null hypothesis when pvalue is smaller than \utwo{chosen $\alpha=0.05$}.
Since the goal of the experiment is to draw conclusions on the number of statistically different neighbors, we take into account Bonferroni correction.
That is, we define the maximal pvalue that is still acceptable to reject the null hypothesis by $n^2 = 1024$.
So, when the pvalue is smaller than $0.05/1024$, the corresponding points are considered different and the square is colored in blue in Figure~\ref{fig:n}.
Otherwise, the null hypothesis is not rejected and the square is colored in red.

Figure~\ref{fig:n} depicts that $\lap(d_1, \cdot, \cdot)$ brought more variability to $\mathbb{M}(d_1)$ than the other two metrics on their spaces.
Hamming distance makes the considered parts of the search space $\mathbb{L}^M$ slightly less flat than $\lap(d_2, \cdot, \cdot)$ makes its space $\mathbb{M}(d_2)$.
Evidently, $\lap(d_1, \cdot, \cdot)$ stands out in this experiment, however, it does not mean that it is the best distance for OFS problem.
This is due to the possibly big jumps of value $\Fest_K$ between the points with statistically different performances.
We plan to study how rugged the landscapes are in future works.

\subsection{Distance-Driven UMDA}
\label{sec:dd-umda}

In this section, we propose an extension of UMDA, introduced in Section~\ref{sec:umda}, that takes advantage of the distances $d_1, d_2.$
This is done in multiple steps.

Notice that it is wasteful to maintain a distribution for every component of $\bm{x} \in \mathbb{L}^M$, since permutations of components in $\bm{x}$ do not influence the performance of $\bm{x}$.
Therefore, we make the first modification of UMDA and store a single vector $\BF{p}$ of size $L$, that contains probabilities for filters in the library.
This distribution $\BF{p}$ is updated for every component in the same way as in Algorithm~\ref{alg:umda}.
We will call this modification of UMDA as UMDA-U.

\utwo{The global distribution in UMDA‑U does not tend to concentrate its probability mass on a few filters.
As a result, when we sample a solution from this distribution, it often contains duplicates and exceeds the allowed number of distinct filters. 
To fix this, we introduce a second modification that adjusts sampling to respect the filter‐count constraint.}

The idea of Probabilistic Logic Sampling, introduced for Bayesian Networks in~\cite{henrion1988propagating}, is applied in Estimation of Distribution algorithms that address permutations-based COPs~\cite{ceberio2012review}.
We employ this idea to avoid repetitions of the same filter multiple times.
Specifically, we create the conditional distribution $P(\rm{x}_i | \rm{x}_1, \rm{x}_2, \dots, \rm{x}_{i-1})$ to sample the component $\rm{x}_i$ of the candidate solution $\BF{x}$.
If $\rm{x}_i \in \set{\rm{x}_1, \rm{x}_2, \dots, \rm{x}_{i-1}}$, then $P(\rm{x}_i | \rm{x}_1, \rm{x}_2, \dots, \rm{x}_{i-1})$ is set to $0$, otherwise it is set to $C \cdot \BF{p}_{\rm{x}_i}$.
Constant $C$ is chosen in such a way that: $\sum_{\rm{x}_i \in \mathbb{L}} P(\rm{x}_i | \rm{x}_1, \rm{x}_2, \dots, \rm{x}_{i-1}) = 1.$
We call this modification of UMDA-U as UMDA-U-PLS.

Finally, we account for metric $\dL$ in the modification of UMDA-U-PLS called UMDA-U-PLS-Dist, which we summarize in Algorithm~\ref{alg:umda-u-pls-dist}.
On every iteration of the loop in line~\ref{alg:u:t}, this algorithm makes $\lambda$ evaluations of the passed objective function $F$, so the content of the loop is repeated $\lfloor b/\lambda \rfloor$ times.
The loop in line~\ref{alg:u:lambda} is used to create the population of $\lambda$ candidate solutions.
In line~\ref{alg:u:sample1}, an integer in the first component is sampled as in the original UMDA.
Then the conditional distribution is defined for every filter $j \in \mathbb{L}$.
The expression for this distribution is proposed in line~\ref{alg:u:cond-distr}, where the constant $C$ is chosen for every $i$.
This constant should be such that $\sum_{j=1}^L P^{(t, i)}\!\br{j  \mid \rm{x}_1^{(k)},\dots, \rm{x}_{i-1}^{(k)}} = 1$.
The value of $P\!\br{j \mid \rm{x}_1, \rm{x}_2, \dots, \rm{x}_{i-1}}$ becomes bigger, when the filter $j$ further from filters $\set{\rm{x}_1, \rm{x}_2, \dots, \rm{x}_{i-1}}$.
To represent the distance between the set and one filter, we used the sum of distances according to the given metric $\dL$.
The value of another component $i$ is then sampled in line~\ref{alg:u:sample2} from the constructed distribution $P^{(t, i)}$.

\scalebox{1}{
\begin{algorithm2e}[H]
    \caption{Proposed UMDA-U-PLS-Dist, equipped with a distance-driven heuristic for integer programming, where objective function $F: \mathbb{L}^M \to \R$ is minimized. Points in its domain have dimensionality $M$. Value at every component of every point is an integer from the set $\mathbb{L} \coloneqq \overline{1,L}$, which is a metric space with distance $\dL$. Given positive integers $\mu \le \lambda$, the algorithm produces $\lambda$ offspring and selects $\mu$ ones from them. The maximal number of $F$ evaluations is limited by constant $b$.}
    \label{alg:umda-u-pls-dist}
    For every $j \in \overline{1, L}$ initialize: $\rm{p}^{(1)}_{j} \gets 1/L$\;
    Set min for probability: $p_{\text{min}} \gets \dfrac{1}{(L-1)\cdot M}$\;
    Initialize best-so-far: $\BF{x}^* \gets \uar{\mathbb{L}^M}$\;
    \For{$t \gets 1, 2, \ldots, \lfloor b/\lambda \rfloor $\label{alg:u:t}}{
        \For{$k \gets 1, 2, \ldots, \lambda$\label{alg:u:lambda}}{
            Sample integer: $\rm{x}_1^{(k)} \sim \BF{p}^{(t)}$\label{alg:u:sample1}\;
            \For{$i \gets 2, 3, \ldots, M$\label{alg:u:M}}{
                \nlnonumber \DontPrintSemicolon \textbf{Construction of conditional distr. $P^{(t, i)}\!\br{\cdot  \mid \rm{x}_1^{(k)},\dots, \rm{x}_{i-1}^{(k)}}$ over integers in $\mathbb{L}$ :}\; \PrintSemicolon
                \nl \For{$j \gets 1, 2, \ldots, L$\label{alg:u:L}}{
                    $P^{(t, i)}\!\br{j  \mid \rm{x}_1^{(k)}, \rm{x}_2^{(k)}, \dots, \rm{x}_{i-1}^{(k)}} \gets 
                    \mathbbm{1}\!\cl{j \notin \set{\rm{x}_1^{(k)}, \rm{x}_2^{(k)}, \dots, \rm{x}_{i-1}^{(k)}}} \cdot C \cdot \rm{p}_j^{(t)} \cdot \sum\limits_{r=1}^{i-1}\dL(j, x_r) $\label{alg:u:cond-distr}\;
                }
                Sample from the distribution: $\rm{x}_i^{(k)} \sim P^{(t, i)}\!\br{\cdot  \mid \rm{x}_1^{(k)},\dots, \rm{x}_{i-1}^{(k)}}$\label{alg:u:sample2}\;
            }
        }
        Set $\br{\BF{y}^{(k)}}_{k=1}^{\lambda}$ to reordered $\br{\BF{x}^{(k)}}_{k=1}^{\lambda}$:
        $F(\BF{y}^{(1)}) \le F(\BF{y}^{(2)}) \le \ldots \le F(\BF{y}^{(\lambda)})$\label{alg:u:reorder}\;
        \For{$j \gets 1, 2, \ldots, L$\label{alg:u:loop-update}}{
            Update distribution: $\rm{p}^{(t+1)}_{j} \gets \dfrac{1}{\mu \cdot M}\sum\limits_{k=1}^{\mu}\sum\limits_{i=1}^{M} \mathbbm{1}\!\cl{\rm{y}_i^{(k)} = j}$\label{alg:u:update}\;
            Adjust to lower/upper bounds: $p_{\text{min}} \le \rm{p}^{(t+1)}_{j} \le 1 - p_{\text{min}}$\label{alg:u:adjust}\;
        }
        \If {$F(\BF{x}^*) \ge F(\BF{y}^{(1)})$} {
            Update best-so-far: $\BF{x}^* \gets \BF{y}^{(1)}$\label{alg:u:best}\;
        }
    }
    \Return $\BF{x}^*$\label{alg:u:ret}\;
\end{algorithm2e}
}

When $\lambda$ new points are created, they are ordered in ascending order of the value assigned by $F$ in line~\ref{alg:u:reorder}.
Then, the underlining distribution $\BF{p}$ is updated.
In line~\ref{alg:u:update} we calculate the number of times that filter $j$ occurred in the $\mu$ top-performing candidate solutions.
This number is divided by the total number of components in those points.
The resulting probability $\rm{p}_j$ is then adjusted in line~\ref{alg:u:adjust} to belong to the range $[p_{\text{min}}, 1-p_{\text{min}}]$, by the following rule: 
$$\rm{p}_j = \max \set{p_{\text{min}}, \min \set{\rm{p}_j,  1-p_{\text{min}}}}.$$
The Algorithm~\ref{alg:umda-u-pls-dist} maintains the best-so-far found solution $\BF{x}^*$, which is updated in line~\ref{alg:u:best}.
This solution is returned in line~\ref{alg:u:ret}, when the computation budget is exhausted.

We applied all version of UMDA using $F = \Fest_K$ with the fixed $K = 10^3$, $\mu = 5, \lambda = 50, b = 2100$.    

\subsection{Numerical Results}

In this section, we consider the leading algorithms and their distance-driven analogs.
The algorithms are categorized in Table~\ref{tab:algo-legend}. We group them by their broader classes for clarity.
That is $(\mu + \lambda)$ EA (Section~\ref{sec:ea-simple}), $(\mu/2 + \lambda)$ EA (Section~\ref{sec:ea-simple-cross}), UMDA (Section~\ref{sec:umda}), 
DD-$(\mu/2 + \lambda)$ EA with distances $d_1, d_2$ (Section~\ref{sec:dd-ea}), UMDA-U, UMDA-U-PLS-Dist and UMDA-U-PLS-Dist with distances $d_1, d_2$ (Section~\ref{sec:dd-umda}).
We denote the set of these algorithms as $\mathbb{A}_2$.
Notably, we did not include DD-$(\mu + \lambda)$ EA to $\mathbb{A}_2$ because it was outperformed by DD-$(\mu/2 + \lambda)$ EA, which emphasizes the usefulness of the crossover operator on this SCOP.

\begin{table}[htbp]
  \centering
  \caption{Algorithms and their legend used in Figure~\ref{fig:application2}}
  \label{tab:algo-legend}
  \begin{tabularx}{0.9\linewidth}{@{}llX@{}}
    \toprule
    \textbf{Algorithm} & \textbf{Discussed in} & \textbf{Legend in Figure~\ref{fig:application2}} \\
    \midrule

    \multicolumn{2}{@{}l}{\textbf{Evolutionary Algorithms}}\\[2pt]
    $(\mu+\lambda)$ EA               & Section~\ref{sec:ea-simple} & \adjustbox{valign=m}{\includegraphics[scale=1]{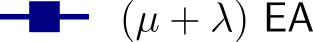}} \\
    $(\mu/2 + \lambda)$ EA           & Section~\ref{sec:ea-simple-cross} & \adjustbox{valign=m}{\includegraphics[scale=1]{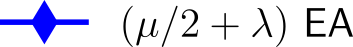}} \\
    \addlinespace[4pt]

    \midrule
    \multicolumn{2}{@{}l}{\textbf{UMDA and Extensions}}\\[2pt]
    UMDA                        & Section~\ref{sec:umda} & \adjustbox{valign=m}{\includegraphics[scale=1]{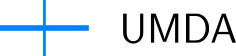}} \\
    UMDA-U                      & Section~\ref{sec:dd-umda} & \adjustbox{valign=m}{\includegraphics[scale=1]{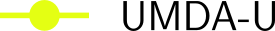}} \\
    UMDA-U-PLS                  & Section~\ref{sec:dd-umda} & \adjustbox{valign=m}{\includegraphics[scale=1]{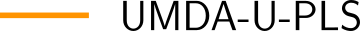}} \\
    \addlinespace[4pt]

    \midrule
    \multicolumn{2}{@{}l}{\textbf{Distance-Driven Algorithms}}\\[2pt]
    DD-$(\mu/2 + \lambda), d_1$ EA                  & Section~\ref{sec:dd-ea} & \adjustbox{valign=m}{\includegraphics[scale=1]{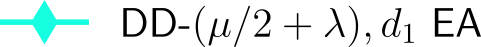}} \\
    DD-$(\mu/2 + \lambda), d_2$ EA                  & Section~\ref{sec:dd-ea} & \adjustbox{valign=m}{\includegraphics[scale=1]{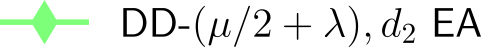}} \\
    UMDA-U-PLS-Dist, $d_1$                          & Section~\ref{sec:dd-umda} & \adjustbox{valign=m}{\includegraphics[scale=1]{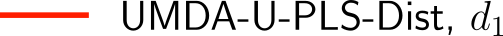}} \\
    UMDA-U-PLS-Dist, $d_2$                          & Section~\ref{sec:dd-umda} & \adjustbox{valign=m}{\includegraphics[scale=1]{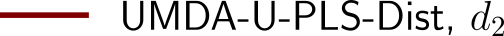}} \\
    \bottomrule
  \end{tabularx}
\end{table}

The results of the application of those algorithms to OFS are presented in Figure~\ref{fig:application2}.
\utwo{Part (a) shows the estimated average quality of the candidate solution proposed at each evaluation step.
The degree of short-term variability reflects how the output variations influence the sampling process.
When the best-so-far curve looks smooth, but the quality-of-solution curve oscillates, it indicates that the sampling distribution does not stabilize.
}

\begin{figure*}[!tb]
    \centering
    \begin{tabular}{ p{0.46\textwidth} p{0.01\textwidth} p{0.46\textwidth} }
         & \adjustbox{valign=m, center}{\includegraphics[width=0.96\textwidth, trim=00mm 00mm 00mm 00mm, clip]{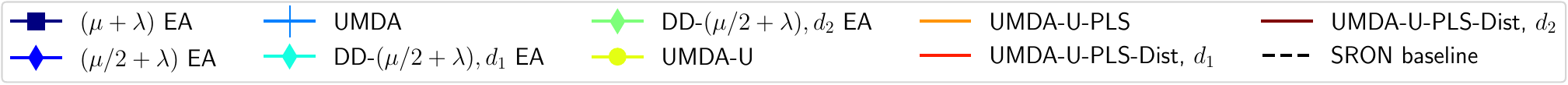}} & \\
        \adjustbox{valign=m, center}{\includegraphics[width=\linewidth, trim=00mm -02mm 8mm 00mm]{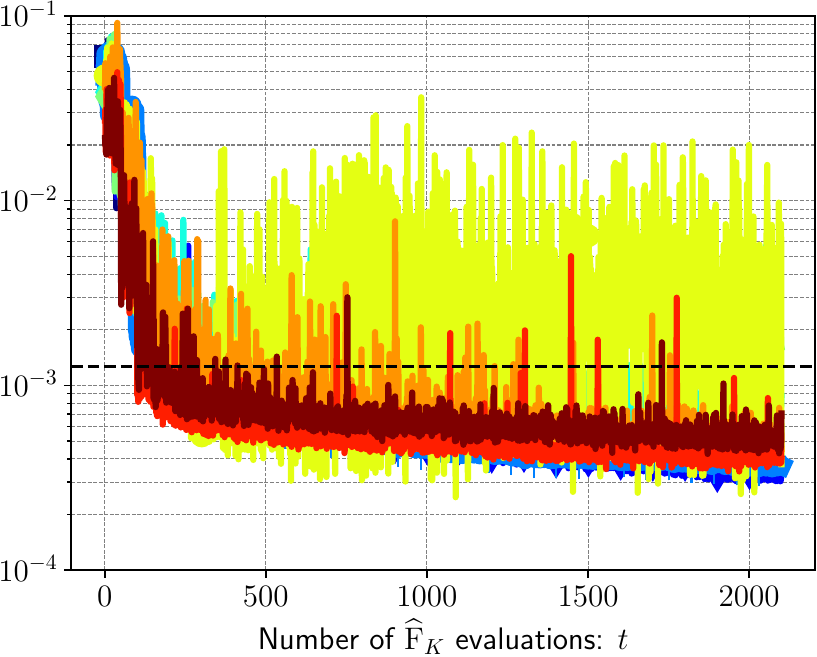}} & &
        \adjustbox{valign=m, center}{\includegraphics[width=\linewidth, trim=8mm -02mm 00mm 00mm]{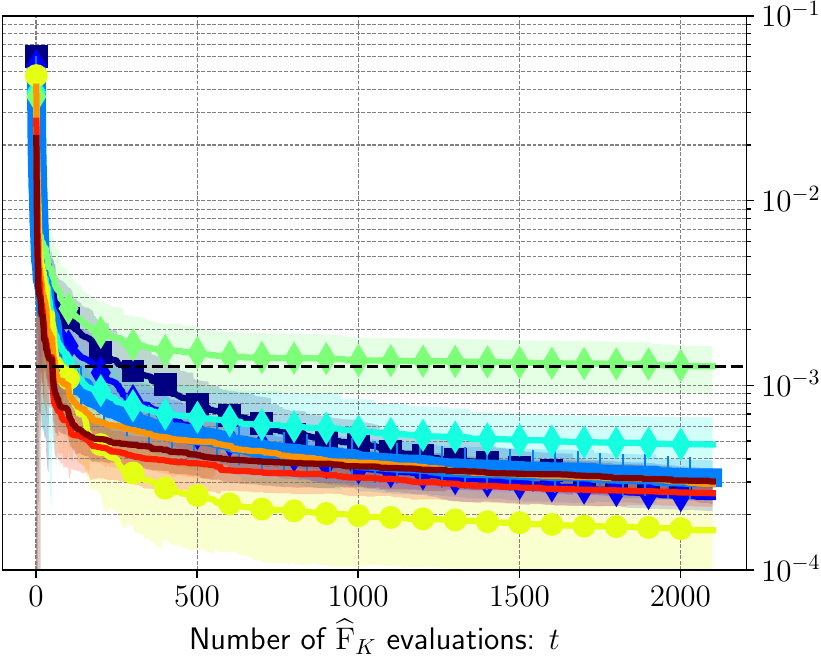}} \\
        \adjustbox{valign=t, center}{ \scalebox{0.8} {\begin{minipage}{0.5\textwidth} \centering \textbf{(a) Average objective value: 
        $ \Eest_n\! \left( \mathrm{f}^{(t)}\!\cl{A(\xi)} \right) $} \end{minipage}} } & \adjustbox{valign=t, center}{} &
        \adjustbox{valign=t, center}{ \scalebox{0.8} {\begin{minipage}{0.5\textwidth} \centering \textbf{(b) Average best-so-far: $ \Eest_n\! \left( \mathrm{g}^{(t)}\!\cl{A(\xi)} \right) $,\\ shaded area: $\sqrt{\Varest_n\! \left( \mathrm{g}^{(t)}\!\cl{A(\xi)} \right)}$ } \end{minipage}} } \\
    \end{tabular}
    \caption{Convergence plot for every algorithm $A \in \mathbb{A}_2$. Part (a) represents the approximated mean of $\mathrm{f}$ defined in Eq.~\eqref{eq:cr1} as the quality of the solution considered by algorithm $A$ when computational budget $t$ is spent. Part (b) represents the approximated mean and standard deviation of $\mathrm{g}$ defined in Eq.~\eqref{eq:cr2} as the quality of best-so-far solution found by each algorithm $A \in \mathbb{A}_2$ when computational budget $t$ is spent. The number of independent runs per algorithm is $n = 20$.
    }
    \label{fig:application2}
\end{figure*}

From part (b) we see that the two worst algorithms are distance-driven versions of EA with $d_1$ and $d_2$.
While DD-$(\mu/2 + \lambda), d_1$ EA manages to find solutions better than $\bm{x}_\text{base}$ relatively quickly, DD-$(\mu/2 + \lambda), d_2$ EA slowly converges to the baseline and then stagnates.
Notably, their versions without distance-driven mutation operator outperforms both algorithms at the end of the optimization.
However, $(\mu/2 + \lambda)$ EA and $(\mu + \lambda)$ EA are slightly worse than DD-$(\mu/2 + \lambda), d_1$ EA, when very small computational budget is spent $\approx 200$ evaluations of $\Fest_K$.
Therefore, a distance-driven analog of EA can produce better solutions when given a tiny computational budget.

\newcommand{\yes}{\textcolor{ForestGreen}{\ding{52}}}
\newcommand{\no}{\textcolor{red}{\ding{56}}}

\begin{figure*}[!tb]
    \centering
    \begin{tabular}{ p{0.45\textwidth} p{0.45\textwidth} }
        \adjustbox{valign=t}{
            \hspace{0pt}\begin{minipage}{0.45\textwidth}
            \centering
                \begin{tabular}{|p{0.14\textwidth}|p{0.04\textwidth}|p{0.04\textwidth}|p{0.04\textwidth}|p{0.04\textwidth}|p{0.04\textwidth}|p{0.04\textwidth}|p{0.04\textwidth}| }
                \hline  
                \adjustbox{center}{\scalebox{0.55}{ Reject $H_0$?}} & \no & \no & \no & \yes & \yes & \yes & \yes \\
                \hline
                \adjustbox{center}{\scalebox{0.55}{ $\mathcal{C}(\bm{x}) > 12$ ?}} & \no & \yes & \yes & \yes & \yes & \yes & \yes \\
                \hline
                \adjustbox{center}{\scalebox{0.55}{ $\mathcal{C}(\bm{x}) = 16$ ?}} & \no & \no & \no & \yes & \yes & \no & \no \\
                \hline
                \end{tabular}
            \end{minipage}
        } &
        \adjustbox{valign=t}{
        \adjustbox{valign=t}{
            \hspace{0pt}\begin{minipage}{0.45\textwidth}
            \centering
                \begin{tabular}{|p{0.26\textwidth}|p{0.21\textwidth}|p{0.21\textwidth}|}
                \hline  
                \adjustbox{center}{\scalebox{0.65}{ \hspace{-10pt} Reject $H_0$?}} & \adjustbox{center}{\yes} & \adjustbox{center}{\yes} \\
                \hline
                \adjustbox{center}{\scalebox{0.65}{ \hspace{-10pt} $\mathcal{C}(\bm{x}) > 12$ ?}} & \adjustbox{center}{\yes} & \adjustbox{center}{\yes} \\
                \hline
                \adjustbox{center}{\scalebox{0.65}{ \hspace{-10pt} $\mathcal{C}(\bm{x}) = 16$ ?}} & \adjustbox{center}{\no} & \adjustbox{center}{\no} \\
                \hline
                \end{tabular}
            \end{minipage}
        }            
        } \\
        \adjustbox{valign=t, center}{\includegraphics[width=0.45\textwidth, trim=00mm -02mm 00mm 00mm, clip]{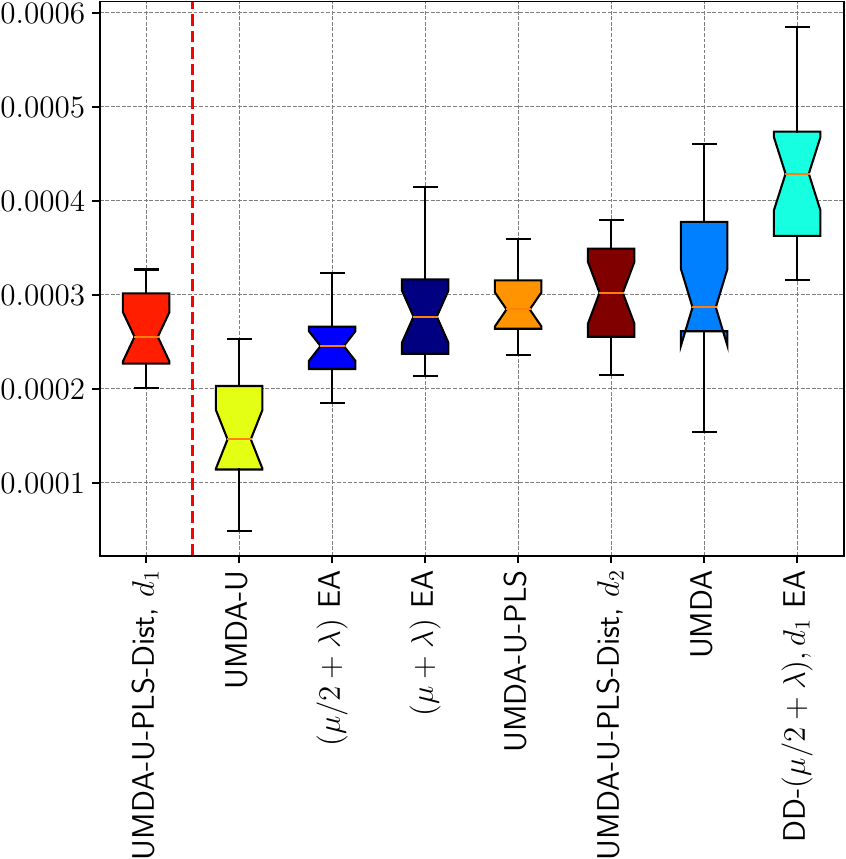}} & 
        \adjustbox{valign=t, center}{\includegraphics[width=0.45\textwidth, trim=00mm -02mm 00mm 00mm, clip]{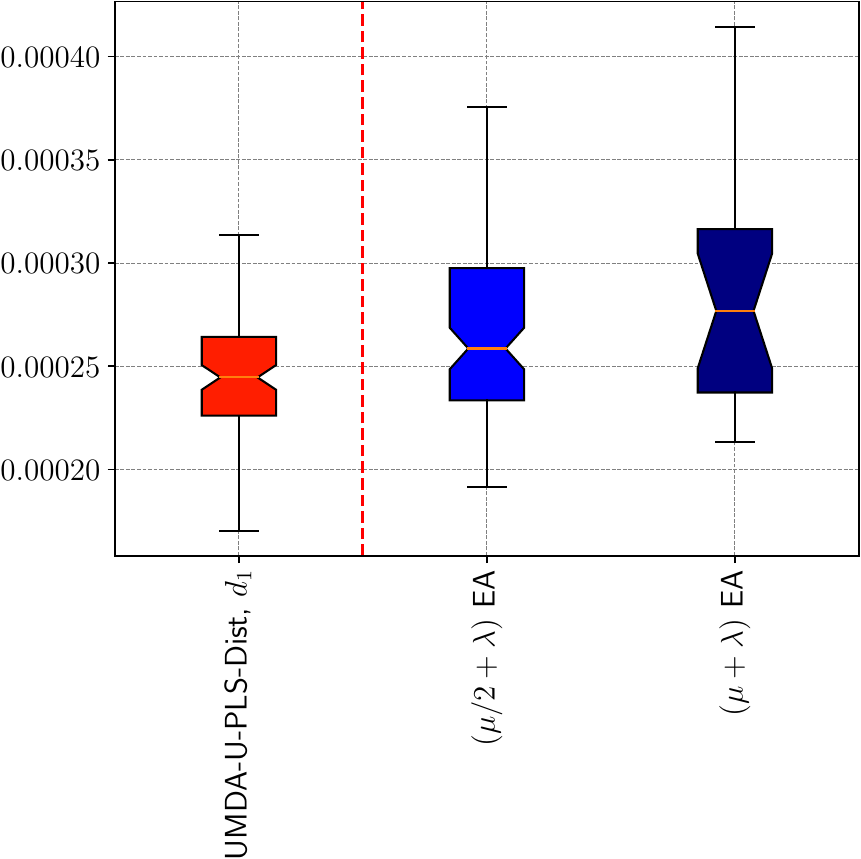}} \\
        \adjustbox{valign=t}{ \scalebox{0.8}{ \begin{minipage}{0.55\textwidth} \begin{itemize} \item[\textbf{(a)}] \textbf{Comparsion against the rest algorithms in $\mathbb{A}_2$ using $T_1 = T_2 = 20$ in Eq.~\eqref{eq:stat-test}} \end{itemize}\end{minipage} } } &
        \adjustbox{valign=t}{ \scalebox{0.8}{ \begin{minipage}{0.55\textwidth} \begin{itemize} \item[\textbf{(b)}] \textbf{Comparison against $(\mu/2 + \lambda)$ EA using $T_1 = T_2 = 100$ and against $(\mu + \lambda)$ EA using $T_1 = 100, T_2 = 20$} \end{itemize}\end{minipage} } }
        \\

    \end{tabular}
    \caption{Statistical comparison defined in Eq.~\eqref{eq:stat-test}, where UMDA-U-PLS-Dist with $d_1$ is compared against the rest algorithms in $\mathbb{A}_2$. 
    The table in the upper part of the figure summarizes the result of this analysis. 
    The first row of the table tells if the null hypothesis $H_0$ is rejected, which happens when pvalue of Mann-Whitney U rank test is smaller than $0.05$.
    The second row tells if the number of different filters $\mathcal{C}(\bm{x})$ in each best-so-far solution $\bm{x}$ produced by every run of the corresponding algorithm is greater than 12.
    It constitutes the small violation of the constraint defined in Eq.~\eqref{eq:ofs-optimization:different}.
    The last row tells if this constraint is not violated, meaning that every solution has exactly 16 different filters.
    The rows of the table represent the algorithms whose names are written vertically at the bottom of the figure.
    Every solution obtained in runs of UMDA-U-PLS-Dist with $d_1$ has exactly 16 different filters.
    }
    \label{fig:stat-test}
\end{figure*}

One of the possible reasons why distance-driven EAs are worse than conventional EA is big jumps in the search space caused by the \utwo{big} value of the mean in the distribution over step-sizes $\mathcal{D}.$ 
It is possible to use parameter-control methods to adjust this parameter during the optimization process, as proposed, for example, in Evolutionary Strategies~\cite{rudolph2012handbook}.
Another way is to use a fixed parameter sampled from a heavy-tailed distribution~\cite{doerr2017fast}.

From part (b) of Figure~\ref{fig:application2} we see that UMDA-U performs significantly better than UMDA.
However, analysis of the best-so-far solutions obtained by UMDA-U shows that each of them has at most 5 different filters, which violates the constraint in Eq.~\eqref{eq:ofs-optimization:different} by far.
It is evident from part (a) of Figure~\ref{fig:application2} that objective values of the solutions produced in different runs of UMDA are significantly different.
This causes the rugged behavior of the line that corresponds to UMDA in the considered chart.
Therefore, an extension of this algorithm, called UMDA-U-PLS was proposed.
UMDA-U-PLS produces solutions with exactly 16 different filters \utwo{(this number of filters is used for the reasons discussed in Sec.~\ref{sec:addressing-ofs})}, surpassing UMDA in terms of average best-so-far solution throughout the whole optimization process.
This demonstrates that it is faster to learn one distribution over $\mathbb{L}$ than $M$ distributions over $\mathbb{L}$.
However, part (a) of Figure~\ref{fig:application2} depicts jumps in the line that corresponds to UMDA-U-PLS.
The jumps happen less frequently than in UMDA-U, however, certain runs still produce relatively bad-performing individuals sampled from the learned conditional distribution.

Notably, \utwo{the improved} UMDA-U-PLS-Dist algorithm with $d_1$ performs better than UMDA-U-PLS in terms of the average best-so-far solution, as shown in part (b) of Figure~\ref{fig:application2}. 
This suggests that the distance-based heuristic used to define the conditional distribution over filters helps to generate better candidate solutions.
However, UMDA-U-PLS-Dist algorithm with $d_2$ performs worse than its counterpart with $d_1$ during the entire optimization process.
We can conclude from it that metric $d_1$ suits better suited for this type of heuristic as well as for distance-driven EA.
At the same time, UMDA-U-PLS-Dist with both distances is robust, as there are no significant jumps in the corresponding lines of Figure~\ref{fig:application2} (a).

Overall, algorithm UMDA-U-PLS-Dist with $d_1$ performs the best among the algorithms that do not violate the constraint Eq.~\eqref{eq:ofs-optimization:different} on the number of different filters.
Let us denote this algorithm as $\mathcal{A}_1$.
In Figure~\ref{fig:stat-test} we demonstrate the statistical analysis of the best-so-far solutions found by $\mathcal{A}_1$ and by the rest of the algorithms in $\mathbb{A}_2$, when the computational budget $b$ is exhausted.
The null hypothesis $H_0$ in the comparison is that the algorithm $\mathcal{A}_1$ does not generate stochastically less best-so-far solutions than other algorithms.
We applied Mann-Whitney U rank test to reject or accept the null hypothesis for every $A \in \mathbb{A}_2$:
\begin{equation} \textsc{MWU} \! \br{ \br{\rm{g}^{(b)}\!\cl{\mathcal{A}_1(\xi^{(i)})}}_{i=1}^{T_1}, \br{\rm{g}^{(b)}\!\cl{A(\xi^{(i)})}}_{i=1}^{T_2} }. \label{eq:stat-test} \end{equation}
$H_0$ is rejected when the pvalue in $\textsc{MWU}$ is smaller than $\alpha=0.05$.
We will say that the vector $\br{\rm{g}^{(b)}\!\cl{A(\xi^{(i)})}}_{i=1}^{T}$ is \emph{result} of the algorithm $A \in \mathbb{A}_2$ obtained in $T$ independent runs.

The statistical comparison is summarized in Figure~\ref{fig:stat-test}.
Algorithm UMDA-U-PLS-Dist with $d_1$ does not violate constraints in Eq.~\eqref{eq:ofs-optimization:different} because all its inherent solutions have exactly 16 different filters.
However, as visualized in Figure~\ref{fig:stat-test} (a), the null hypothesis $H_0$ is not rejected for $(\mu/2 + \lambda)$ EA and for $(\mu + \lambda)$ EA when only $n = 20$ independent runs of each algorithm are made.
Therefore, we conducted an experiment shown in Figure~\ref{fig:stat-test} (b), where we increased the number of independent runs.
It is shown that the result of our reference algorithm UMDA-U-PLS-Dist with $d_1$ is stochastically less than the results of the considered counterparts.
Moreover, the reference algorithm does not violate constraints in Eq.~\eqref{eq:ofs-optimization:different} unlike other algorithms in Figure~\ref{fig:stat-test} (b).
Therefore, we demonstrated that the proposed UMDA-U-PLS-Dist with $d_1$ is the robust, fast, and efficient solver of OFS optimization problem.

\section{Analysis of the Obtained Solutions to OFS Problem}
\label{sec:solutions}

In this section, we analyze a subset of all candidate solutions $\mathbb{S}$ considered in $n = 100$ independent runs by the algorithm UMDA-U-PLS-Dist with $d_1$.
We now describe how we selected the subset $\mathbb{G} \subset \mathbb{S}$ \utwo{to locate a diverse set of high-performing solutions}.
Remember that the reference solution, described in Section~\ref{sec:description-trace-gas-measurement-device}, is denoted as $\bm{x}_{\text{base}}$ \utwo{and we solve the minimization problem}.
Considering the points of $\mathbb{S}$ in ascending order of the objective value \utwo{(from worst to best)}, that was assigned to them during the optimization process.
We select the set $\mathbb{G}$ such that:
\begin{enumerate}
    \item it has a maximal size;
    \item pairwise distance according to $d_1$ between its elements is at least $D_{\text{min}} \in \R$;
    \item the objective value of every element is at most $f_{\text{max}} \in \R$.
\end{enumerate}
Algorithmically we greedily add points from $\mathbb{S}$ to $\mathbb{G}$ in the ascending order of the objective value, that was assigned to the solutions during the optimization process.
If at least one of the conditions is not satisfied, then the point is not included in $\mathbb{G}.$
We used constants $D_{\text{min}} = 0.05$ and $f_{\text{max}} = \Fest_K \! \br{\bm{x}_{\text{base}}} / 4 \approx 3 \cdot 10^{-4},$ where $K = 10^4.$
According to our experiments, the chosen distance $D_{\text{min}}$ is twice smaller than the distance between a pair of points chosen in a series of experiments uniformly at random.
Therefore, the set of solutions $\mathbb{G}$ is diverse and high-performing since every element is at least 4 times better than $\bm{x}_{\text{base}}.$
An example of selection $\mathbb{G}$ according to the described procedure is visualized in Figure~\ref{fig:contours}.
We will say that the set $\mathbb{G}$ contains \emph{final solutions} to the OFS problem.

\uthree{Figure~\ref{fig:contours}, the starts represent solutions found in optimization landscape. Colors depict values of the objective function, with more blue representing better values. Set $\mathbb G$ contains a subset of these solutions. Not all blue solutions are selected for this set due to the minimal pairwise distance constraint $D_\text{min}$. Also, solutions distant from selected ones not chosen for $\mathbb {G}$ due to their objective value being worse than $f_\text {max}$ constraint. }

Let us call the solutions in $\mathbb{G}$ as $\set{\bm{x}^{(i)} \mid i \in \overline{1,m}}$.
In our numerical run, we obtained $m = |\mathbb{G}| = 114$ solutions.
Following the recommendation in Section 2.8.3 of~\cite{fu2015handbook}, we reevaluated the objective value for each of them using $K = 10^4$.
We will assume that $\Fest_K \! \br{\bm{x}^{(1)}} \le \Fest_K \! \br{\bm{x}^{(2)}} \le \dots \le \Fest_K \! \br{\bm{x}^{(m)}} $.

In Figure~\ref{fig:all-final-solutions}, we visualized the objective values of all the found final solutions $\mathbb{G}$.
Notably, all of them outperform the baseline solution $\bm{x}_{\text{base}}$ by far.
The diversity of the set $\mathbb{G}$ provides practitioners with the flexibility to choose solutions beyond the one with the optimal objective value, denoted by $\bm{x}^{(1)}$.
This can be particularly relevant when unforeseen constraints arise after the optimization process.
These constraints might include minimizing the manufacturing cost, given that the construction cost is different for filters in $\mathbb{L}$.

\begin{figure}[!tb]
\centering
\includegraphics[width=0.5\linewidth]{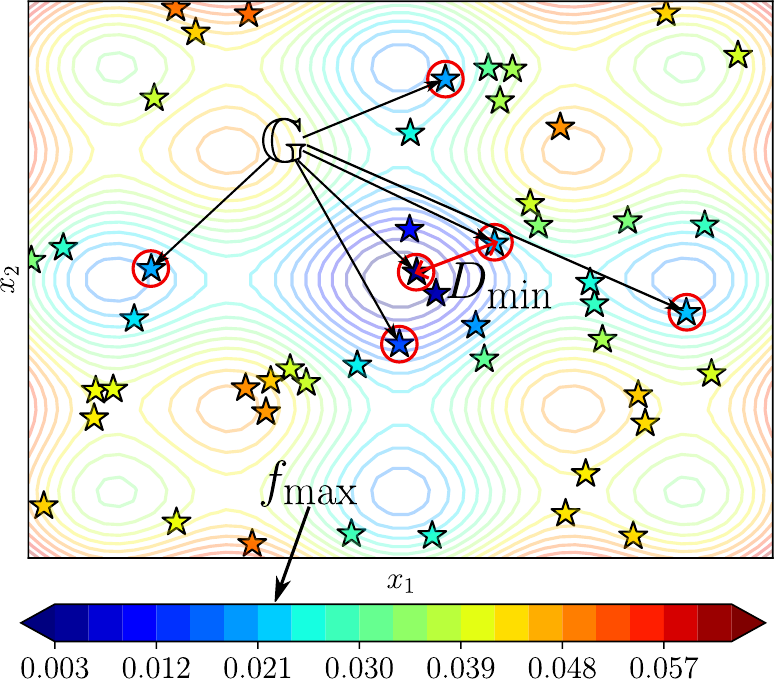}
\caption{Example of the optimization landscape with two decision variables $x_1, x_2$, where candidate solutions are shown as stars.
The color of isocounters and stars represent the value of the objective function, with a legend provided in the color bar at the bottom.
All displayed stars represent elements within the set denoted by $\mathbb{S}$.
Stars encircled in the figure belong to a subset $\mathbb{G} \subset \mathbb{S}$.
The smallest distance $D_{\text{min}}$ between selected solutions is highlighted.
Additionally, the color corresponding to the maximum allowed in $\mathbb{G}$ objective function value $f_{\text{max}}$ is indicated.
}
\label{fig:contours}
\end{figure}

\begin{figure}[!tb]
    \centering
    \includegraphics[width=0.6\linewidth]{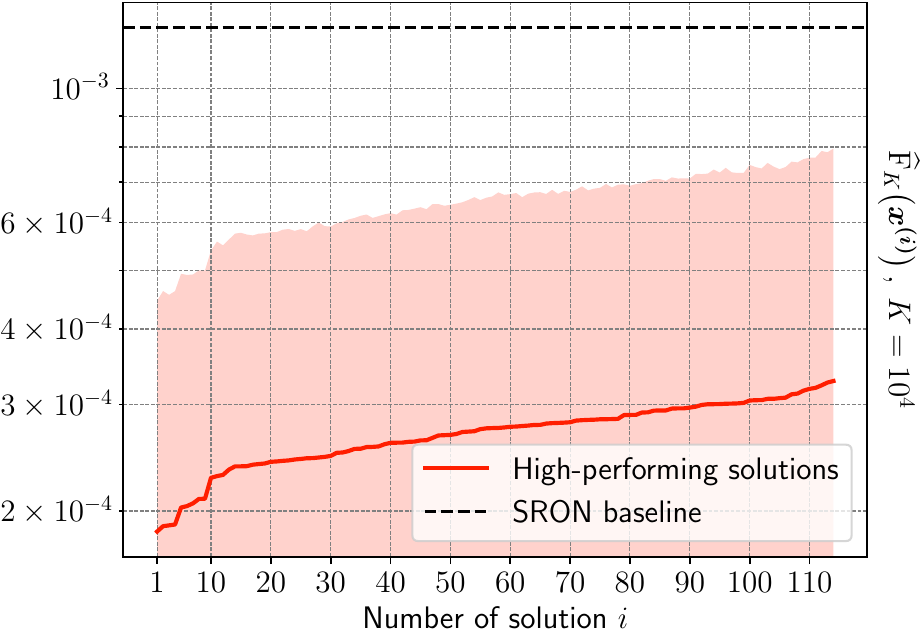}
    \caption{ 
    Overview of the objective values of all the found final solutions to the OFS problem.
    The red line shows the objective values and the shaded area represents the standard deviation from the $\Fest_K \! \br{\bm{x}^{(i)}}$.
    The horizontal axis denotes the number of the candidate solution in the ascending order of the objective function.
    }
    \label{fig:all-final-solutions}
\end{figure}

In Figure~\ref{fig:profiles}, we outline the transmission profiles of every filter in the top 3 final solutions.
Surprisingly, it appears that the found solutions contain transmission profiles with little or no sharp features, but more smooth sections with either a low or high transmission. 
These results show that filters with these features improve the performance of the Trace Gas Measurement Device.

\begin{figure}[!tb]
    \centering
    \begin{tabular}{ p{0.97\linewidth} }
        \adjustbox{valign=t, center}{\includegraphics[width=0.5\linewidth]{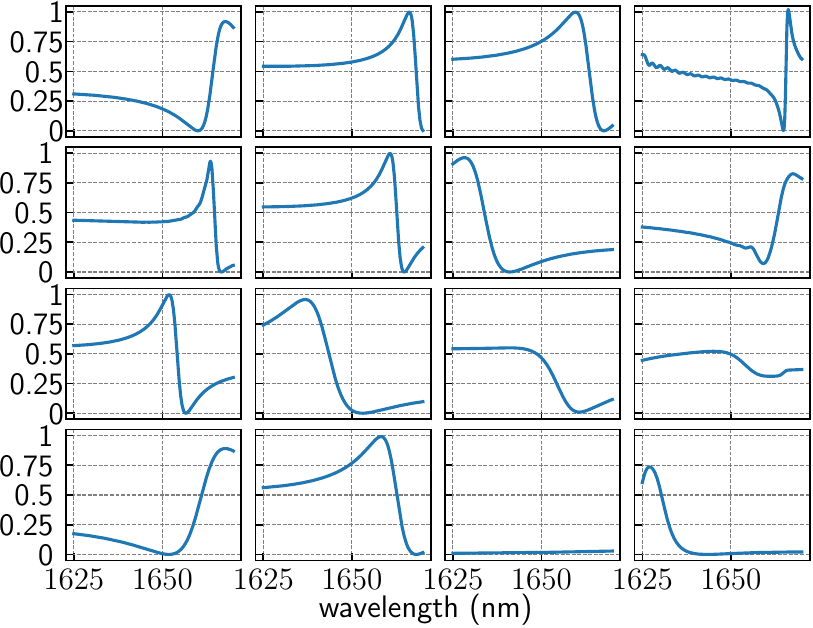}} \\
        \adjustbox{valign=t, center}{ \scalebox{0.8}{ \begin{minipage}{1.1\linewidth} \centering \textbf{(a) $\Fest_K \!\br{\bm{x}^{(1)}} = 1.84\cdot 10^{-4}$} \end{minipage} } } \\
        \adjustbox{valign=t, center}{\includegraphics[width=0.5\linewidth]{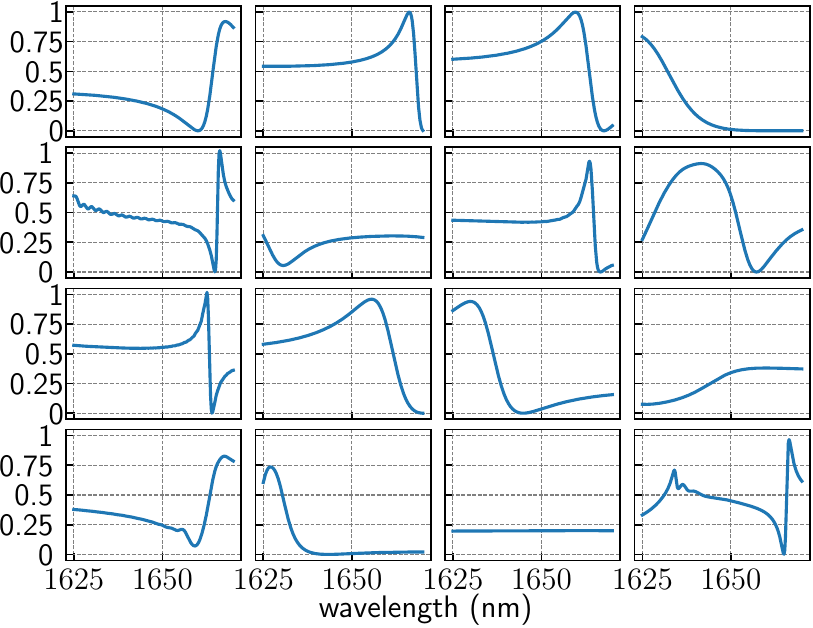}} \\
        \adjustbox{valign=t, center}{ \scalebox{0.8}{ \begin{minipage}{1.1\linewidth} \centering \textbf{(b)  $\Fest_K \!\br{\bm{x}^{(2)}} = 1.88\cdot 10^{-4}$ } \end{minipage} } } \\
        \adjustbox{valign=t, center}{\includegraphics[width=0.5\linewidth]{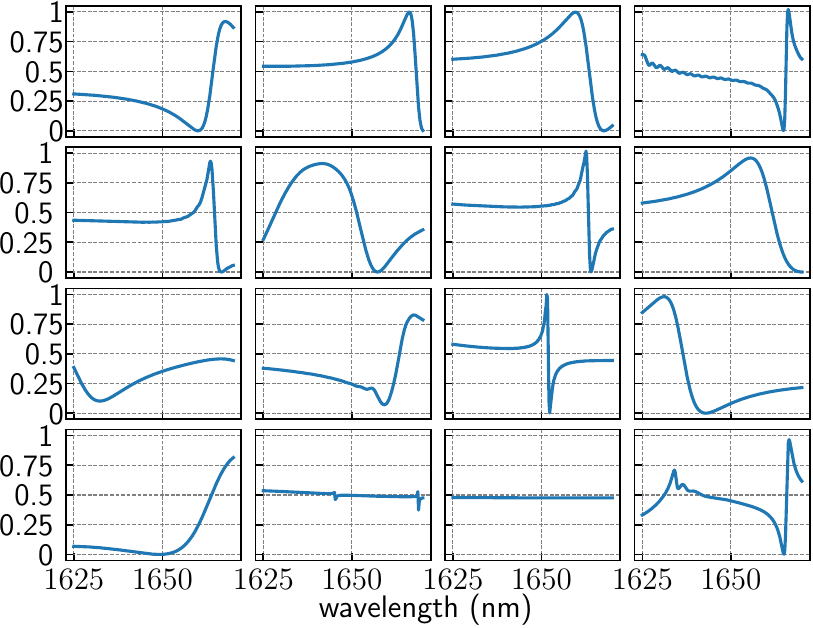}} \\
        \adjustbox{valign=t, center}{ \scalebox{0.8}{ \begin{minipage}{1.1\linewidth} \centering \textbf{(c)  $\Fest_K \!\br{\bm{x}^{(3)}} = 1.89\cdot 10^{-4}$ } \end{minipage} } } 
    \end{tabular}
    \caption{ Transmission profiles of 3 top solutions from the diverse set $\mathbb{G}$. Each sub-panel shows the transmission profile as a function of wavelength. 
    The objective function was evaluated with $K = 10^4.$
    }
    \label{fig:profiles}
\end{figure}

\begin{figure*}[!tb]
    \centering
    \begin{tabular}{ p{0.46\textwidth} p{0.01\textwidth} p{0.46\textwidth} }
        & \adjustbox{valign=t, center}{\includegraphics[width=0.45\textwidth, trim=00mm 65mm 00mm 00mm, clip]{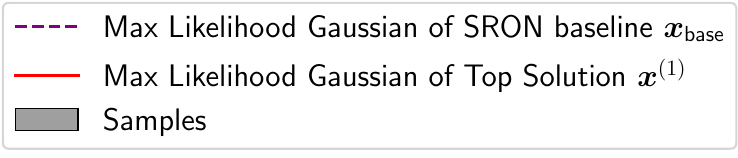}} & \\
        \includegraphics[width=\linewidth, trim=00mm 00mm 00mm 00mm, clip]{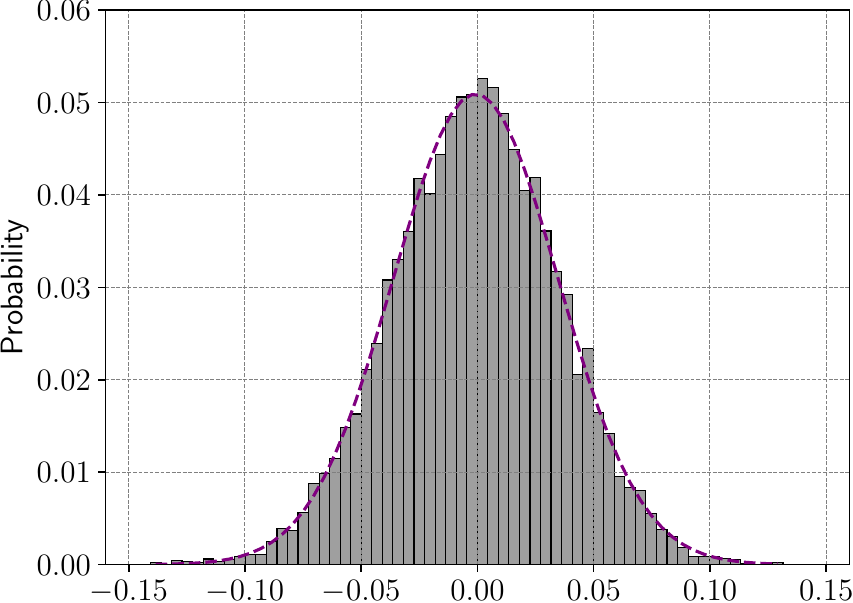} & &
        \includegraphics[width=\linewidth, trim=00mm 00mm 00mm 00mm, clip]{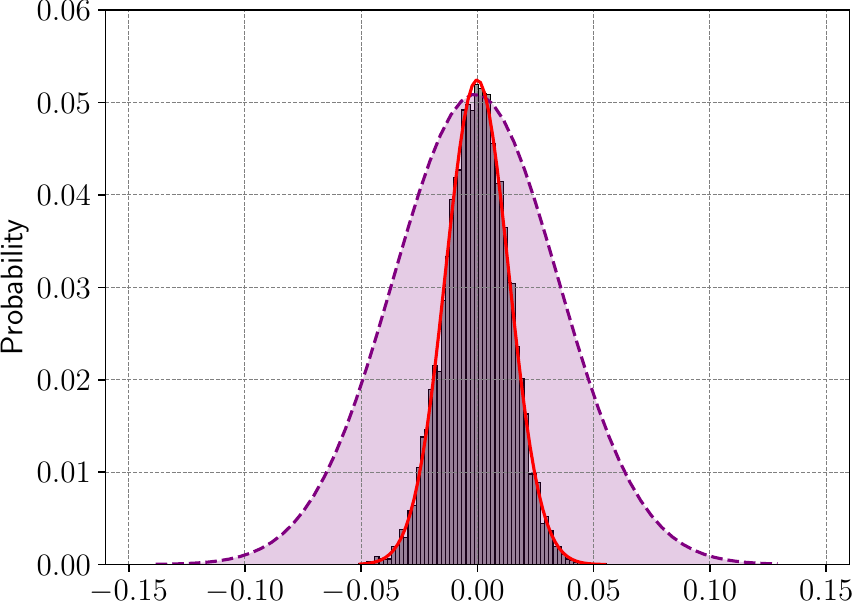} \\
        \adjustbox{valign=t, center}{ \scalebox{0.8}{ \begin{minipage}{1.1\linewidth} \centering \textbf{(a) Distribution of $\Sdiv\!\br{\bm{x}_{\text{base}}}$ } \end{minipage} } } & &
        \adjustbox{valign=t, center}{ \scalebox{0.8}{\begin{minipage}{1.1\linewidth} \centering \textbf{(b) Distribution of $\Sdiv\!\br{\bm{x}^{(1)}}$ } \end{minipage} } } \\

    \end{tabular}
    \caption{ Distribution of relative deviation from the ground-truth in one gas retrieval measurement for the baseline solution in (a) and the top-1 found solution in (b). Each random variable is sampled $K = 10^4$ times and the resulting values are split into 60 segments of equal length. We plot the maximum likelihood Gaussian Distribution that produces the observed samples. In part (b) we copy the Gaussian from part (a) to visualize the stochastic dominance of the solution $\bm{x}^{(1)}$ over $\bm{x}_{\text{base}}.$
    }
    \label{fig:noise-comparison}
\end{figure*}

In Figure~\ref{fig:noise-comparison}, we propose a comparison of the best-found solution against the baseline solution.
This analysis shows that the trace gas retrieval algorithm precision follows a Gaussian distribution for both the reference solution (base) as well as the found top solution. 
The reduced width of the distribution indicates an increase in the performance of the Trace Gas Measurement Device.

\section{Conclusion}

This work investigated the Optimal Filter Selection (OFS) problem as an instance of the Stochastic Combinatorial Optimization Problem (SCOP). 
We developed a simulator that mimics the gas retrieval process, including the introduction of artificial noise to the input signal, \utwo{which translates to the output of the retrieved gas}. 
Our objective was to identify a set of filters that minimizes the expected error after gas retrieval.
To achieve this, we proposed a formalization of OFS, which allowed us to define and analyze an objective function that aligns with our intuitive criteria for an optimal solution. 
Finally, various metaheuristics were employed to generate sets of filters with minimal expected error based on the defined objective function.

We evaluated the performance of various metaheuristics applied to the OFS problem. 
Based on this analysis, we identified the leading algorithms. 
To further enhance their efficiency, we developed distance-driven extensions that leverage domain-specific information. This information is encoded as a metric on the search space of all possible OFS solutions. 
We analyzed and compared the effectiveness of two proposed metrics.
Furthermore, to address the internal optimization problem within one of the algorithms, we proposed a novel heuristic\utwo{, called Distance‑Driven Assignment Evolutionary Algorithm (DDA‑EA),} for solving the inverse of the Linear Assignment Problem (LAP). 
We then compared the performance of the distance-driven algorithms against their original counterparts. 
The comparison was reinforced with statistical testing.
This analysis revealed that the proposed first-of-its-kind algorithm UMDA-U-PLS-Dist with the $d_1$ metric stands out as the most efficient and robust solver among the considered approaches.
The code and tools used to simulate trace gas measurement devices, apply optimization algorithms for filter selection, and perform the experiments described in this work can be accessed at \url{https://github.com/AntKirill/CompressSpecLIACS-Public}.

To analyze the performance of the selected filters, we examined all candidate solutions generated by the algorithm across multiple independent runs. 
We then identified a diverse subset of high-performing solutions for further analysis and comparison. 
These findings suggest that filters exhibiting smooth sections  (no sharp transitions) with a large local difference in the transmission lead to improved performance in the Trace Gas Measurement Device. 
Furthermore, we observed a significant improvement in the top-performing solution compared to the baseline, indicating the effectiveness of the proposed approach. 
Additionally, the distribution of trace gas retrieval algorithm precision across different filter sets follows a Gaussian distribution. 

\section*{Declaration of generative AI and AI-assisted technologies in the manuscript preparation process}

During the preparation of this work, the authors utilized ChatGPT to correct grammatical errors and rephrase some sentences to enhance clarity. After using this service, the authors reviewed and edited the content as needed and take full responsibility for the content of the published article.

\bibliographystyle{ieeetr}
\bibliography{bib}

\newpage
\appendix 

\section{Properties of the OFS Objective Function}
\label{sec:objf:properties}
Let us show the properties of the OFS objective function defined in Eq.~\eqref{eq:objfunction}.
First of all, it conforms to our intuitive criteria \upd{\ref{req:objf-informal}} for the assessment of filter selections.
Indeed, for the fixed solutions $\bm{x}, \bm{y} \in \mathscr{L}$ the requirement can be rewritten as follows:
\begin{equation} 
        \E^2\cl{\Sdiv(\bm{x}, \xi)} \le \E^2\cl{\Sdiv(\bm{y}, \xi)}
        \Var\cl{\Sdiv(\bm{x}, \xi)} < \Var\cl{\Sdiv(\bm{y}, \xi)} .
    \label{eq:case1} 
\end{equation}

We used the square of the expected value since the expected value can be negative.
If $\bm{x}$ and $\bm{y}$ are such that Eq.~\eqref{eq:case1} is upheld, then $\E^2\left[\Sdiv(\bm{x}, \xi)\right] + \Var\left[\Sdiv(\bm{x}, \xi)\right] < \E^2\left[\Sdiv(\bm{y}, \xi)\right] + \Var\left[\Sdiv(\bm{y}, \xi)\right]$.
Which is the same as $\E\left[\Sdiv^2(\bm{x}, \xi)\right] < \E\left[\Sdiv^2(\bm{y}, \xi)\right]$, meaning $\F(\bm{x}) < \F(\bm{y})$.

Now we clarify what happens when the mean and variance are in conflict, meaning when $\bm{x}$ and $\bm{y}$ are such that
\begin{equation} \begin{cases} \E^2\cl{\Sdiv(\bm{x}, \xi)} < \E^2\cl{\Sdiv(\bm{y}, \xi)} \\ \Var\cl{\Sdiv(\bm{x}, \xi)} > \Var\cl{\Sdiv(\bm{y}, \xi)} \end{cases} 
\label{eq:case2} \end{equation}
Let us demonstrate how function $F$ influences the additional difference in means needed to compensate for the increase in variances.
Consider some fixed selections of filters $\bm{x}, \bm{y} \in \mathscr{L}$ that satisfy Eq.~\eqref{eq:case2}.
Due to the symmetry between $\bm{x}$ and $\bm{y}$, we assume without loss of generality that: \begin{equation} \F(\bm{x}) \le \F(\bm{y}) \label{eq:xVSy} \end{equation}
In order to shorten the notation, we will use the following random variables $\mathrm{X} \coloneqq \Sdiv(\bm{x}, \xi), \mathrm{Y} \coloneqq \Sdiv(\bm{y}, \xi). $
According to Eq.~\eqref{eq:objfunction}, the latter Eq.~\eqref{eq:xVSy} is equivalent to: \begin{equation} \E^2(\mathrm{X}) + \Var(\mathrm{X}) \le \E^2(\mathrm{Y}) + \Var(\mathrm{Y}). \label{eq:xVSy:transformed} \end{equation}
There are three cases: 1) $\E(\mathrm{X}) = 0, \Var(\mathrm{Y}) = 0;$ 2) $\Var(\mathrm{Y}) \ne 0;$ 3) $\E(\mathrm{X}) \ne 0$.

\paragraph{Case 1)} If $\E(\mathrm{X}) = 0, \Var(\mathrm{Y}) = 0$, then trivially $\Var(\mathrm{X}) < \E^2(\mathrm{Y})$.
\paragraph{Case 2)} Let us now consider the case when $\Var(\mathrm{X}) \ne 0$.
Multiplication of inequality Eq.~\eqref{eq:xVSy:transformed} by a positive factor ${\E^2(\mathrm{Y})}/{\Var(\mathrm{Y})}$ and truthful transformations gives inequality Eq.~\eqref{eq:ontheway_k} equivalent to Eq.~\eqref{eq:xVSy}:
\begin{equation}    
    \dfrac{\E^2(\mathrm{Y})}{\Var(\mathrm{Y})} \cdot \br{\E^2(\mathrm{Y}) - \E^2(\mathrm{X})} \ge  \dfrac{\E^2(\mathrm{Y})}{\Var(\mathrm{Y})} \cdot \br{\Var(\mathrm{X}) - \Var(\mathrm{Y})}.
    \label{eq:ontheway_k}
\end{equation}

Due to the considered case Eq.~\eqref{eq:case2} we have $\E^2(\mathrm{Y}) - \E^2(\mathrm{X}) > 0$, so division of inequality Eq.~\eqref{eq:ontheway_k} by the positive factor $\E^2(\mathrm{Y}) - \E^2(\mathrm{X})$ and truthful transformations gives the following inequality Eq.~\eqref{eq:k}, which is equivalent to Eq.~\eqref{eq:xVSy}:
\begin{equation}
    \dfrac{\E^2(\mathrm{Y})}{\Var(\mathrm{Y})} \ge \dfrac{ \dfrac{\Var(\mathrm{X})}{\Var(\mathrm{Y})} - 1 }{1 - \left(  \dfrac{\E(\mathrm{X})}{\E(\mathrm{Y})} \right)^2 }.
    \label{eq:k}
\end{equation}

\begin{figure*}[!tb]
    \begin{tabular}{p{0.45\textwidth} p{0.45\textwidth}}
    \includegraphics[width=0.45\textwidth, trim=0mm 00mm 00mm 00mm, clip]{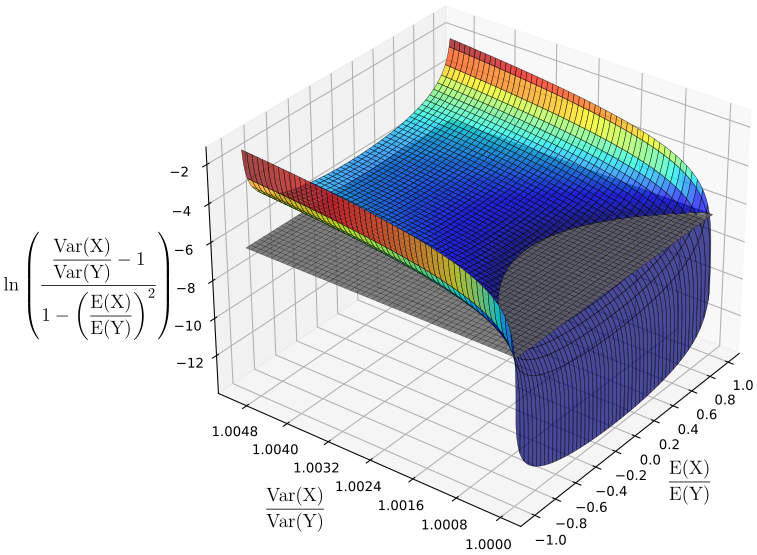}
    &
    \includegraphics[width=0.45\textwidth, trim=0mm 00mm 00mm 00mm, clip]{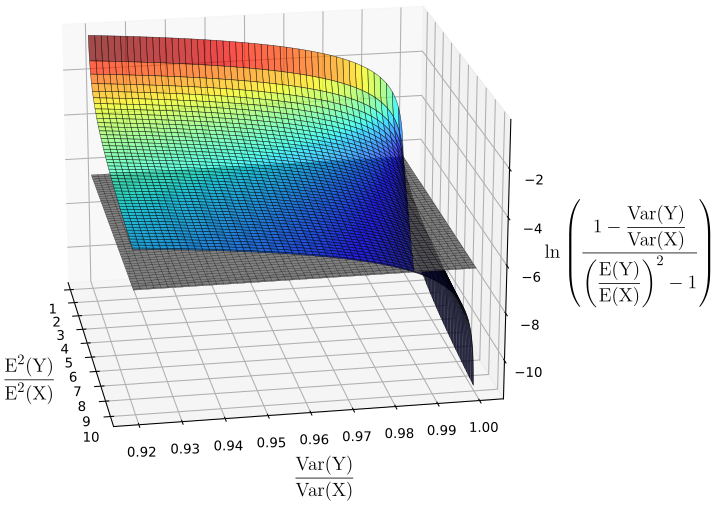}
    \\
    \adjustbox{center}{\textbf{(a) $\Var(\mathrm{Y}) \ne 0$}} & \adjustbox{center}{\textbf{(b) $\E(\mathrm{X}) \ne 0$ }}
    \end{tabular}
    \caption{Function of ratios of variances $\Var(\mathrm{X})/\Var(\mathrm{Y})$ and means $\E(\mathrm{X})/\E(\mathrm{Y})$ between $\mathrm{X}$ and $\mathrm{Y}$ is shown as colored surface. 
    Given that $\E(\mathrm{X}) < \E(\mathrm{Y})$, it represents the lower bound for the quantity $\E^2(\mathrm{Y})/\Var(\mathrm{Y})$, which is necessary and sufficient for function $F$ to assess $\mathrm{X}$ better than $\mathrm{Y}$. 
    The gray plane depicts the average of experimentally computed values of this quantity.}
    \label{fig:trade-off}
\end{figure*}

\paragraph{Case 3)} 
When $\E(\mathrm{X}) \ne 0$ we can make analogical equivalent transformations to obtain the following Eq.~\eqref{eq:k3}:
\begin{equation}
    \dfrac{\E^2(\mathrm{X})}{\Var(\mathrm{X})} \ge \dfrac{ \dfrac{\Var(\mathrm{Y})}{\Var(\mathrm{X})} - 1 }{1 - \left(  \dfrac{\E(\mathrm{Y})}{\E(\mathrm{X})} \right)^2 } .
    \label{eq:k3}
\end{equation}

\paragraph{Analysis.}
The right-hand side (RHS) of Eq.~\eqref{eq:k} and Eq.~\eqref{eq:k3} establishes the lower bound for the ratio of the mean and variance %
necessary and sufficient for function $F$ to assess $\mathrm{X}$ better than $\mathrm{Y}$.
Figure~\ref{fig:trade-off} (a) visualizes the RHS of Eq.~\eqref{eq:k} and Figure~\ref{fig:trade-off} (b) visualizes Eq.~\eqref{eq:k3} in logarithmic scale for various ratios that conform to our assumption Eq.~\eqref{eq:case2}.

To interpret the left-hand side (LHS) of Eq.~\eqref{eq:k} and Eq.~\eqref{eq:k3}, we conduct the following experiment.
We sample $1000$ points uniformly at random in the space $\mathscr{L}$ defined in Section~\ref{sec:searchSpace}.
For each of the sampled point $\bm{x}$ we estimate $\E\cl{\Sdiv(\bm{x}, \xi)}$ using sample mean and $\Var\cl{\Sdiv(\bm{x}, \xi)}$ using sample variance of $K = 10^3$ samples.
This procedure produces the set: $$\mathbb{V} \coloneqq \set{\dfrac{\Eest_K^2\cl{\Sdiv(\bm{x}^{(i)}, \xi)}}{\Varest_K\cl{\Sdiv(\bm{x}^{(i)}, \xi)}}}_{i = 1}^{1000}.$$

We compute the sample mean of the set $\mathbb{V}$, which gives an approximation of the value in LHS of Eq.~\eqref{eq:k} and Eq.~\eqref{eq:k3} expected while optimizing the function $F$.
The $\ln$ of this value equals $-6.6$, and it is depicted by the gray plane in Figure~\ref{fig:trade-off}.
The part of the surface that is lower than the gray plane in Figure~\ref{fig:trade-off} depicts the set of ratios of expected values and means that lead to grading $\mathrm{X}$ better than $\mathrm{Y}$ in average over the set $\mathbb{V}$.
Conversely, the region above the plane represents ratios leading to worse average grading of $\mathrm{X}$ compared to $\mathrm{Y}$.

When $\Var(\mathrm{X})$ grows relative to $\Var(\mathrm{Y})$, the value ${\Var(\mathrm{X})}/{\Var(\mathrm{Y})}$ becomes bigger.
As demonstrated in Figure~\ref{fig:trade-off}, in this situation it becomes less likely to assess $\mathrm{X}$ better than $\mathrm{Y}$.
The intersection between the surface and the gray plane in Figure~\ref{fig:trade-off} occurs when the ratio of variances, $\Var(\mathrm{X})/\Var(\mathrm{Y})$, is less than $1.0024$.
This intersection depicts the maximum tolerable ratio in variances where $\mathrm{X}$ can still be graded better than $\mathrm{Y}$ by the function $F$, on average over the set $\mathbb{V}$.
When such grading is possible, we see that it becomes less likely when values of $\E(\mathrm{X})$ and $\E(\mathrm{Y})$ are closer in absolute value.
Moreover, it is clear that for the fixed ratio of expected value the RHS of Eq.~\eqref{eq:k} is monotonically increasing, so grading of $\mathrm{X}$ better than $\mathrm{Y}$ under the assumption in Eq.~\eqref{eq:case2} is possible only when variances are not too different but means compensate this difference by far.
\upd{In summary, we demonstrated how the means must adjust to offset the rise in variances.
Although the results of this analysis are not directly applied in the subsequent sections of the paper, they remain crucial as they provide guarantees regarding the candidate solutions we favor.}

\section{Metric space for OFS}
\label{sec:metric}

We remind the following classical definition of the metric space in Definition~\ref{def:metricSpace}.

\begin{definition}[Metric Space] A \emph{metric space} is a pair $(X, d)$, where $X$ is a set and $d$ is a real-value function on $X \times X$ which satisfies that, for any $x,y,z \in X$, 
\begin{enumerate}
    \item $d(x,y) \geq 0$ and $d(x,y) = 0 \iff x = y$,
    \item $d(x,y) = d(y, x)$,
    \item $d(x,z) \leq d(x,y) + d(y,z)$.
\end{enumerate} 
The function $d$ is called the \emph{metric} on $X$. 
\label{def:metricSpace}
\end{definition}

Formally, elements of the space $\mathbb{M}(\dL)$ are sets, which are called equivalence classes: $\set{\bm{x}, \bm{y} \in \mathbb{L}^M \mid \bm{x} \equiv \bm{y}}.$
When we define a metric on the space $\mathbb{M}(\dL)$ we show the axioms of metric on points $\bm{x}, \bm{y}, \bm{z} \in \mathbb{L}^M.$
This can be trivially extended to all the elements from the corresponding equivalence classes of $\bm{x}, \bm{y}, \bm{z}$ using properties of the quotient space.

\begin{proposition}
    If $(\mathbb{L}, \dL)$ is metric space, then $\textsc{LAP}(\dL, \cdot, \cdot)$ is metric on the space $\mathbb{M}(\dL)$.
    \label{prop:metric}
\end{proposition}
\begin{proof}
    Let us check all the axioms of metric from Definition~\ref{def:metricSpace}.
    For the sake of shortness, we denote $\text{LAP}(\dL, \cdot, \cdot)$ as $\mathscr{D}$.
    Consider arbitrary $\bm{x}, \bm{y}, \bm{z} \in \mathbb{L}^M$.
    \begin{enumerate}
        \item $\mathscr{D}$ is a sum of non-negative values of $\dL$ computed between some elements of $\mathbb{L}$, so $\mathscr{D}(\bm{x}, \bm{y}) \ge 0$.
        And $\bm{x} \equiv \bm{y}$ if and only if $\mathscr{D}(\bm{x}, \bm{y}) = 0$ by Definition~\ref{def:quatient}.
        \item Consider $\pi^* \in \mathbb{P}$ such that $\mathscr{D}(\bm{x}, \bm{y}) = \sum\nolimits_{i = 1}^{k}\dL\!\br{x_{i}, y_{\pi^*(i)}}$.
        Due to symmetry of $\dL$, the latter equals to $\sum\nolimits_{i = 1}^{k}\dL\!\br{y_{\pi^*(i)}, x_{i}} = \mathscr{D}(\bm{y}, \bm{x})$, therefore $\mathscr{D}(\bm{x}, \bm{y}) = \mathscr{D}(\bm{y}, \bm{x})$.
        \item There exist bijections $\pi_1, \pi_2 \in \mathbb{P}$ such that:
        \begin{align*}
        & \mathscr{D}(\bm{x},\bm{y}) + \mathscr{D}(\bm{y},\bm{z}) \\
        & = \sum_{i=1}^k \dL(x_i, y_{\pi_1(i)}) + \sum_{i=1}^k \dL(y_i, z_{\pi_2(i)}) \\
        & = \sum_{i=1}^k \dL(x_i, y_{\pi_1(i)}) + \sum_{i=1}^k \dL(y_{\pi_1(i)}, z_{\pi_2 \circ \pi_1 (i)}) \\
        & \ge \sum_{i=1}^k \dL(x_i, z_{\pi_2 \circ \pi_1 (i)}) \\
        & \ge \min_{\pi \in \mathbb{P}} \set{\sum_{i=1}^k \dL(x_i, z_{\pi(i)})} = \mathscr{D}(\bm{x}, \bm{z}).
        \end{align*}
    \end{enumerate}
    We demonstrated that all the axioms of the metric are upheld, therefore $\textsc{LAP}(\dL, \cdot, \cdot)$ is metric on the space $\mathbb{M}(\dL)$ for any metric $\dL$.
\end{proof}

\end{document}